\documentclass[5p]{elsarticle}
\usepackage{hyperref}
\usepackage{graphicx}
\usepackage{diagbox}
\usepackage[american]{babel}
\usepackage{amsmath,amssymb}
\usepackage{amsthm}
\usepackage[bold]{hhtensor}
\usepackage{multirow}
\usepackage{graphicx}
\usepackage[caption=false]{subfig}
\usepackage{mathrsfs}
\usepackage{epstopdf}
\usepackage{mathtools}
\usepackage{algorithm}
\usepackage{algorithmic}
\usepackage{array} 
\usepackage{color}

\newcommand{\squishlist}{
 \begin{list}{$\bullet$}
  { \setlength{\itemsep}{0pt}
     \setlength{\parsep}{1pt}
     \setlength{\topsep}{1pt}
     \setlength{\partopsep}{0pt}
     \setlength{\leftmargin}{1.5em}
     \setlength{\labelwidth}{1em}
     \setlength{\labelsep}{0.5em} } }
\newcommand{\squishend}{
  \end{list} }

\DeclareMathOperator*{\argmin}{argmin}
\DeclareMathOperator*{\sign}{sign}
\DeclareFixedFont{\fontn}{OT1}{ptm}{m}{n}{10pt}
\DeclareFixedFont{\fontb}{OT1}{ptm}{bx}{n}{10pt}

\let\oldsqrt\sqrt
\def\sqrt{\mathpalette\DHLhksqrt}
\def\DHLhksqrt#1#2{%
\setbox0=\hbox{$#1\oldsqrt{#2\,}$}\dimen0=\ht0
\advance\dimen0-0.2\ht0
\setbox2=\hbox{\vrule height\ht0 depth -\dimen0}%
{\box0\lower0.4pt\box2}}

\newtheorem{theorem}{Theorem}
\graphicspath{{figure/}}
\begin{document}

\begin{frontmatter}

\title{A Fast and Robust TSVM for Pattern Classification}
\author[a]{Bin-Bin~Gao}
%\cortext[cor]{Corresponding author Email:csgaobb@gmail.com}
\ead{csgaobb@gmail.com}

\author[b]{Jian-Jun~Wang}
%\cortext[cor]{Corresponding author Email:wjj@swu.edu.cn}
\ead{wjj@swu.edu.cn}

\address[a]{YouTu Lab, Tencent, Shenzhen 518000, China.}
\address[b]{School of Mathematics and Statistics, Southwest University, Chongqing 400715, China.}

\begin{abstract}
Twin support vector machine~(TSVM) is a powerful learning algorithm by solving a pair of small-sized SVM-type problems. However, there are still some specific issues such as low efficiency and weak robustness when it is faced with some applications. In this paper, we propose a Fast and Robust TSVM~(FR-TSVM) to deal with the above issues. In order to strengthen model robustness, we propose an effective fuzzy membership function and reformulate the TSVMs such that different input instances can make different contributions to the learning of the decision hyperplane. To further speed up the training procedure, we develop an efficient coordinate descent algorithm with shirking to solve the involved a pair of quadratic programming problems (QPPs). Moreover, the theoretical foundations of the proposed approach are analyzed in details. The experimental results on several artificial and benchmark datasets indicate that the FR-TSVM not only obtains a fast learning speed but also shows a robust classification performance. Code has been made available at:~\url{https://github.com/gaobb/FR-TSVM}.
\end{abstract}


\end{frontmatter}

%\linenumbers

%% main text
\section{Introduction}\label{sec:int}
Support vector machine~(SVM), invented by Vapnik~\cite{Vapnik98}, is a very popular machine learning algorithm. Due to the adoption structural risk minimization principle, SVM is capable of handling excellently with many classification and regression problems, such as machine fault-diagnosis~\cite{Widodo07}, image identification~\cite{Bernd01}, text classification~\cite{Joachims98}, biomedicine~\cite{El-Naqa02} and financial forecast~\cite{Trafalis00}, \emph{etc.}  It has become a prominent highlight of machine learning research. However, the application of SVM is occasionally restricted by some practical issues, especially computational speed and model robustness. To overcome these limitations, some solutions have also been proposed.

Mangasarian~\emph{et al.}~\cite{Mangasarian06} changed the proximal planes which are originally parallel to each other to generate a maximal spaced separating hyperplane into nonparallel ones, and proposed a generalized eigenvalue proximal SVM (GEPSVM). The GEPSVM speed-ups training with slighter accurate performance since it is solved by a pair of simple generalized eigenvalue problems. Following this concept, Jayadeva \emph{et al}.~\cite{Jayadeva07} proposed twin SVM (TSVM). Its objective is turned to be optimized by placing the non-parallel proximal planes as closely as possible to their corresponding instances' cluster and as far as possible from their adversary instances' cluster (as shown in Fig.~\ref{fig:GI-a}). Therefore, the TSVM algorithm converts GEPSVM into a pair of SVM-like convex programs and computation time is reduced satisfactorily to one-fourth of that of a conventional SVM. Recently, numerous nonparallel proximal-plane learning methods are proposed as variants of TSVM, such as TBSVM~\cite{Shao11}, Structural TSVM~\cite{qi2013structural}, Robust TSVM~\cite{qi2013robust}, Laplacian Smooth TSVM~\cite{chen2014laplacian}, Least-square TSVM(Ls-TSVM)~\cite{Kumar09}, Ls-TSVM for multi-class~\cite{nasiri2015}, Twin Parametric-margin SVM~\cite{Peng11,peng2015improvements}, Multi-label TSVM~\cite{chen2016mltsvm}, RPTWSVM~\cite{rastogi2017robust} and Pin-TSVM~\cite{xu2017novel}. In additional, some regression model based TSVM are also proposed such as TSVR~\cite{peng2010tsvr}, ~TPSVR\cite{peng2014twin}, $\epsilon$-TSVR~\cite{ye2016weighted}, TWSVR~\cite{khemchandani2016twsvr} and $v$-TWSVR~\cite{rastogi2017nu}. All of the proposed variants share the same merits of TSVM. However, one challenge is that training instances from real-world applications often carry information with significant noise. These TSVM methods are sensitive to outliers or noises.

Noisy data often can deteriorate the generalization ability of SVM or TSVM. In term of fuzzy theory technique, \emph{noisy information} can be causally converted into the \emph{fuzzy information} to meet fuzzy inference theory. The training of SVM would be too sensitive to noisy inputs if all training instances are treated equivalently at the training stage. A conceptual way to alleviate this sensitive deterioration is to contract the influence of noisy inputs. This means that a conventional SVM, which intrinsically treats every input instance in equivalence, can be improved by introducing fuzzy membership functions to soften input information. A category of fuzzy SVMs are hence developed, such as Lin and Wang~\cite{Lin02,Lin04}, Wu and Liu~\cite{Wu07}, Inoue and Abe~\cite{Inoue01}, Yang~\textit{et al.}~\cite{Yang13}, and Tang~\cite{Tang11}~\emph{etc}. The elementary concept of fuzzy SVM is to allocate a small confident membership for each noisy instance consistent with the~\emph{information fuzziness} which the instance has carried to reduce its influence on the optimization. The membership is generally assigned according to the instance's confidence intrinsically related to its native class. The introduction of fuzzy membership reduces effectively the uncertainty caused by noisy instances and leads to a robust classifier. To utilize the fuzzy concept, some variants of fuzzy TSVM have been developed in~\cite{Xu12,Khemchandani07,richhariya2018robust}. However, these fuzzy SVM methods equivalently assigned the fuzzy membership to each instance and they cannot distinguish support vectors and outliers effectively.

In term of computational efficiency, plenty of input instances from real-world applications require abundant quadratic programming computations, which slow down training efficiency. This computational cost significantly limits the applications of SVM. There are various solutions for improvement, including chunking~\cite{Kudo01}, decomposition~\cite{dong2005fast}, sequential minimal optimization (SMO)~\cite{SMO98}~\emph{etc}. However, they need to optimize the entire set of nonzero
Lagrangian multipliers and the generated kernel matrix may still
be too large to adapt to memory. Recently, Chang~\emph{et al.} proposed a primal~\emph{coordinate descent}~(CD) method to deal with large-scale linear SVM~\cite{Chang01,Chang08,Hsieh08}. 

In this paper, we propose a fast and robust twin support vector machine~(FR-TSVM) for a classification problem.  Our contributions are summarized as follows.

\squishlist
\item We construct a novel fuzzy membership function to measure the importance of each training instance. The assignment policy of fuzzy membership depends on the structural information of training instances in feature space.

\item We embed the proposed fuzzy concept into TSVM and propose a fast and robust TSVM model. In FR-TSVM, different input instances can make different contributions to the learning of decision hyperplanes. The pro- posed model can effectively alleviate the effect of noisy instances and obtain a robust performance.

\item We develop further coordinate descent algorithm with shrinking to speed-up the computations of linear and nonlinear FR-TSVM. In addition, this paper presents FR-TSVM in details including corresponding theoretical fundamentals and related properties.

\item Experiments on artificial and benchmark datasets show that the proposal brings more satisfactory performances on both classification accuracy and learning efficiency, compared with traditional SVM, FSVM and TSVM.
\squishend

The remaining parts of this paper are organized as follows. We first briefly review the basics of classical SVM and its related works, including FSVM and TSVM in Section~\ref{sec:bg}. Then, Section~\ref{sec:app} proposes FR-TSVM approach, including fuzzy membership function construction, linear and nonlinear FR-TSVM models and their optimization algorithms. After that, experimental results are reported in Section~\ref{sec:exp}. Some preliminary results
have been published in a conference presentation~\cite{gao2015coordinate}.

\begin{figure*}
 \centering
 \subfloat[TSVM]
 {\includegraphics[width= 0.45\textwidth]{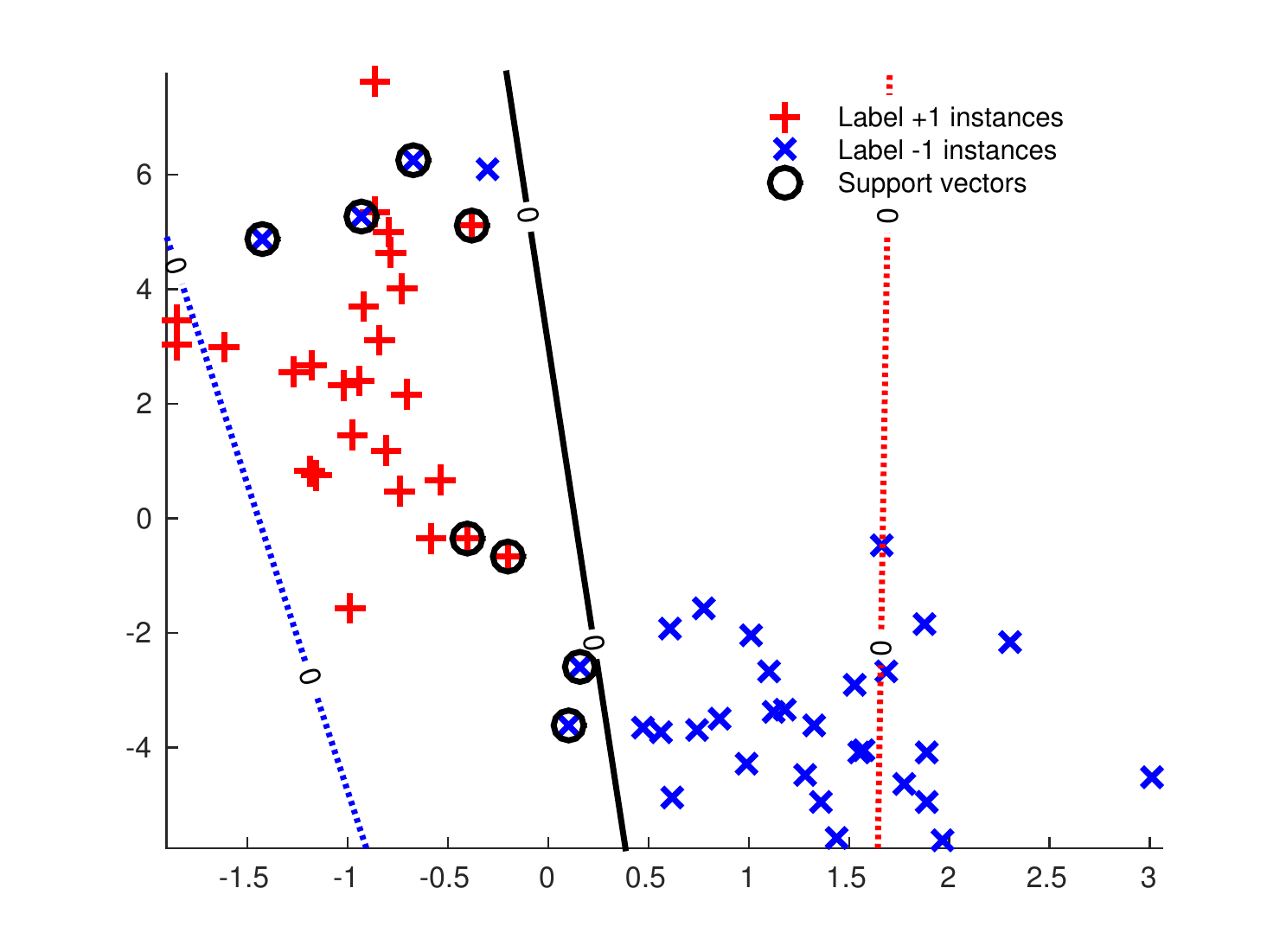}\label{fig:GI-a}}
 \subfloat[FR-TSVM]
 {\includegraphics[width= 0.45\textwidth]{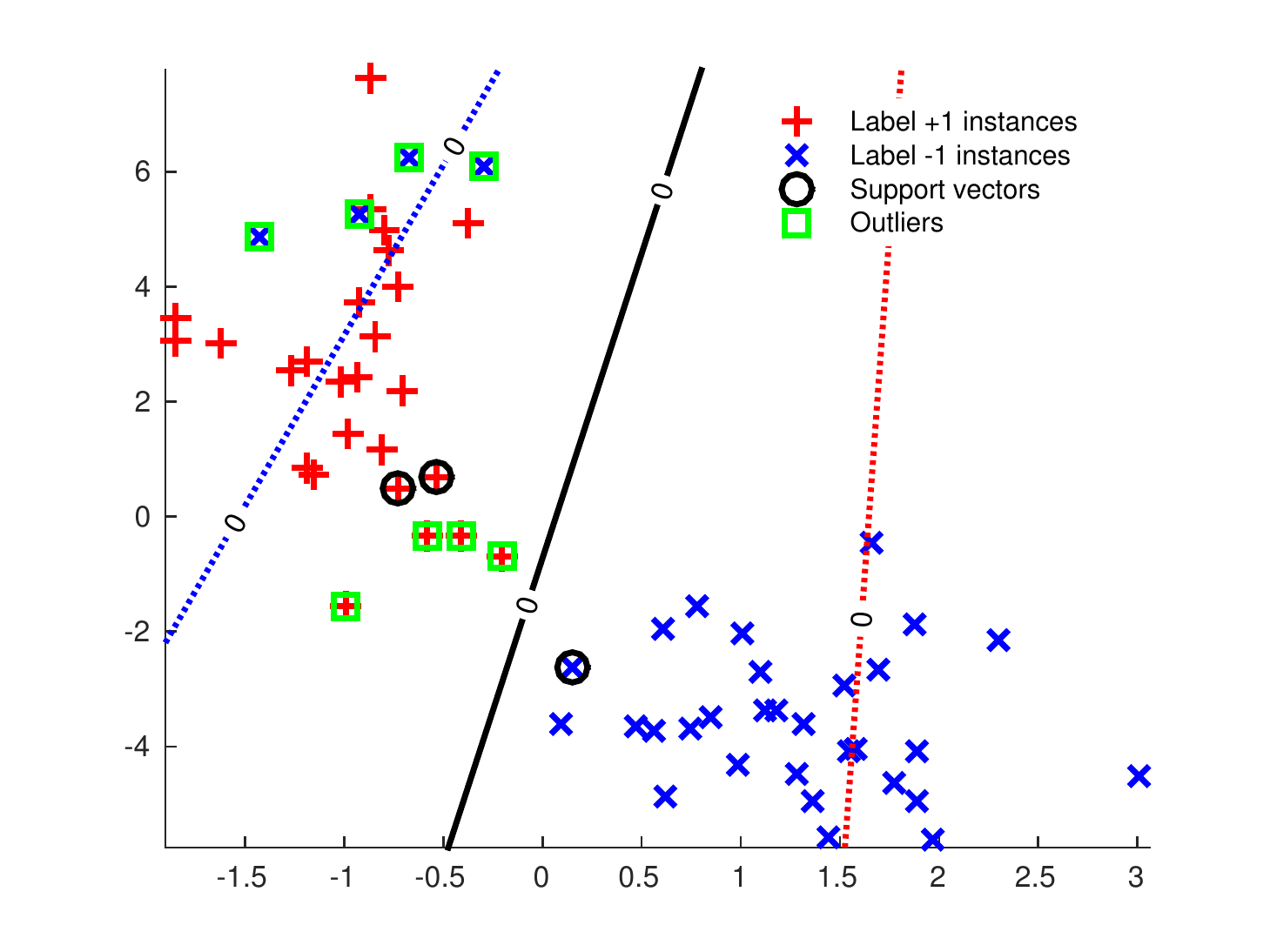}\label{fig:GI-b}}
 \caption{Geometric interpretation of linear TSVM and FR-TSVM for binary classification. Red ``{\color{red} {+}}" and blue ``{\color{blue} {x}}" represent positive and negative instances, respectively. A pair of nonparallel decision hyperplanes are shown by red and blue dashed line and the best decision plane is represented by a black line. We can see that the proposed method is easy to get a better decision hyperplane because it can effectively detect noisy points and weaken their effects.  \label{fig:GI}}
\end{figure*}

\section{Background}\label{sec:bg}
In this paper, we mainly focus on a binary classification problem in the $n$-dimensional space $\mathcal{R}^n$. We denote training data as $D=\{(\vec x_i, y_i)|i=1,2,\dotsc,l\}$, where $\vec x_i \in \mathcal{R}^n$ represents an input instance with the corresponding label $y_i \in \{+1,-1\}$. Without loss of generality, we assume that the matrix $X$ with size of $l\times n$ represents all training instances, and the matrices $X _\pm$ with sizes of $l_\pm\times n$ represent the training instances belonging to ``+1"(positive) or ``-1"(negative) class, respectively, where $l=l_++l_-$.

\subsection{Support vector machine}
As a classical machine learning algorithm, the standard linear SVM~\cite{Vapnik98} aims to construct a pair of parallel hyperplanes between two classes of instances:
\begin{equation}
\vec w^{\rm T} \vec x+b=+1~\textrm{and}~\vec w^{\rm T} \vec x+b=-1\,,
\end{equation}
where $\vec w \in \mathcal{R}^n$ and $b \in \mathcal{R}$ are the normal vector and the bias term of
hyperplanes, respectively. These separating hyperplanes are obtained by solving the following quadratic programming problem~(QPP):
\begin{align}\label{eq:p-svm}
\begin{split}
&\min_{\vec w,b,\xi_i}~\frac {1}{2}{\|\vec w\|}^2+c\sum_{i=1}^l \xi_i\\
&\textrm{s.t.}~~y_{i}(\vec w^{T}\vec x_{i}+b)+\xi_{i} \geq 1 ,\xi_{i} \geq \vec 0, i=1,2,\dotsc, l.
\end{split}
\end{align}
where $\|.\|$ stands for $\ell_2$-norm, $\xi_{i}$ is called as slack variable which denotes the misclassification error associated with the $i$-th input instance and $c>0$ is regularization factor that balances the importance between
the maximization of margin width (\emph{i.e.}, the minimization of $\frac {1}{2}{\|\vec w\|}^2$) and the minimization of the training error. The dual QPP of problem~Eq.\eqref{eq:p-svm} is:
\begin{align}\label{eq:d-svm}
\begin{split}
&\max_{\vec \alpha}~\sum_{i=1}^l \alpha_i+\frac{1}{2} \sum_{i=1}^l \sum_{j=1}^l y_i y_j \langle \vec x_i,\vec x_j \rangle \alpha_i \alpha_j\\
&\textrm{s.t.}~~\sum_{i=1}^l \alpha_i y_{i} = 0, 0 \leq \alpha_i \leq 1 , i=1,2,\dotsc, l\,,
\end{split}
\end{align}
where $\vec \alpha \in \mathcal{R}^l$ is Lagrangian multiplier. After solving
this dual QPP, a testing instance $\vec x$ is classified as ``+1" or ``-1" following decision function
\begin{equation}
f(\vec x)=\sign(\vec w^{*\rm T} \vec x+b^*)
\end{equation}
%&=\sign\Big(\sum_{i=1}^l \alpha^*_i y_i \langle \vec x_i, \vec x \rangle+\frac{1}{N_{sv}}\sum_{j=1}^{N_{sv}}\big(y_j-\sum_{i=1}^l \alpha^*_i y_i \langle \vec x_i,\vec x_j \rangle \big)\Big)\,,
where $(\vec w^*, b^*)$ and $\alpha^*_i$ are the solution of Eq.\eqref{eq:p-svm} and Eq.\eqref{eq:d-svm}, respectively.

\subsection{Twin support vector machine}
Different from the conventional SVM, TSVM is in fact constructed by two non-parallel decision planes, \emph{i.e.},
\begin{equation}\label{eq:lh-tsvm}
\vec w_+^{\rm T} \vec x+b_+=0~\textrm{and}~\vec w_-^{\rm T}\vec x+b_-=0\,.
\end{equation}
To construct such two non-parallel decision planes, a pair of primal
optimization problems are set up:
\begin{equation}\label{eq:p-tsvm1}
\begin{split}
&\min_{\vec w_+,b_+,\vec \xi_-}~~~\frac{1}{2}\parallel X_+\vec w_++\vec e_+b_+\parallel^2+c_1\vec e_-^{\rm T}\vec \xi_-\\
&\textrm{s.t.}~~-(X_-\vec w_++\vec e_-b_+)+\vec \xi_-\geq \vec e_-,\vec \xi_-\geq \vec {0}\,,\\
\end{split}
\end{equation}
and
\begin{equation}\label{eq:p-tsvm2}
\begin{split}
&\min_{\vec w_-,b_-,\vec \xi_+}~~~\frac{1}{2}\parallel X_-\vec w_-+\vec e_-b_-\parallel^2+c_2\vec e_+^{\rm T}\vec \xi_+\\
&\textrm{s.t.}~~~(X_+\vec w_-+\vec e_+b_-)+\vec \xi_+\geq \vec
e_+,\vec \xi_+\geq \vec {0}\,,
\end{split}
\end{equation}
where $c_1>0$ and $c_2>0$ are parameters, $\vec \xi_+$ and $\vec \xi_-$ denote the vectors of slack variables for positive and negative classes, respectively, and $\vec e_+,\vec e_-$ correspond to unit vectors with $l_\pm$ dimensions. By introducing the Lagrangian multipliers, the dual QPPs of Eq.\eqref{eq:p-tsvm1} and~Eq.\eqref{eq:p-tsvm2} can be represented as followings
\begin{equation}
\begin{split}\label{eq:d-tsvm1}
&\max_{\vec \alpha}\quad \vec e_-^{\rm T}\vec \alpha-\frac{1}{2}\vec \alpha^{\rm T}H_-(H_+^{\rm T}H_+)^{-1}H_-^{\rm T}\vec \alpha \\
&\textrm{s.t.}\quad \vec 0\leq\vec \alpha\leq c_1\vec e_-,
\end{split}
\end{equation}
\begin{equation}
\begin{split}\label{eq:d-tsvm2}
&\max_{\vec \beta}\quad \vec e_+^{\rm T}\vec \beta-\frac{1}{2}\vec \beta^{\rm T}H_+(H_-^{\rm T}H_-)^{-1}H_+^{\rm T}\vec \beta \\
&\textrm{s.t.}\quad \vec 0\leq\vec \beta\leq c_2\vec e_+,
\end{split}
\end{equation}
where $H_+ = [X_+, \vec e_+], H_- = [X_-, \vec e_-]$. The non-parallel hyperplanes~Eq.\eqref{eq:lh-tsvm} can be obtained from the
solutions~$\vec \alpha^*$ and~$\vec \beta^*$ of~Eq.\eqref{eq:d-tsvm1} and~Eq.\eqref{eq:d-tsvm2} by
\begin{equation}\label{eq:u_star-tsvm}
\vec u_+^* =-(H_+^{\rm T}H_+)^{-1}H_-^{\rm T}\pmb\alpha^*,\vec u_-^* =(H_-^{\rm T}H_-)^{-1}H_+^{\rm T}\pmb\beta^*,
\end{equation}
where $\vec u_\pm^*=[\vec w_\pm^{*\rm T},b_\pm^*]^{\rm T}$, $w_\pm^{*\rm T}$ and $b_\pm^*$ are the solutions of~Eq.\eqref{eq:p-tsvm1} and~Eq.\eqref{eq:p-tsvm2}. TSVM then can easily assign a label ``+1" or ``-1" to a testing instance $\vec x$ by
\begin{equation}\label{Line_TSVM_Desc}
f(\vec x)=\argmin_\pm\frac{\mid {\vec w^{*}_\pm}^{\rm T}\vec x+b^{*}_\pm\mid}{\parallel \vec w^{*}_\pm\parallel}\,,
\end{equation}
where $|\cdot|$ is a function taking its absolute value. If the matrix $H_+^{\rm T}H_+$ or $H_-^{\rm T}H_-$ is ill-conditioned, TSVM artificially introduces a term~$\lambda I (\lambda >0)$, where $I$ is an identity matrix of appropriate dimension. In the experiments, we fix the value of $\lambda$ as 0.01.
\subsection{Fuzzy support vector machine}
To reduce the effects of outliers, Lin~\emph{et al.} introduced fuzzy membership to each input instance of SVM and proposed fuzzy SVM~(FSVM)~\cite{Lin02}. In FSVM, training instance $\vec x_i$ is assigned a fuzzy membership $0\leq s_i\leq1$ besides a label $y_i \in \{+1,-1\}$. The input dataset $T$ is thus modified as $T' = \{(\vec x_i, y_i, s_i)|i=1,2,\dotsc,l\}$. These fuzzy memberships $\{s_i|i=1,2,\dotsc,l\}$ are used to reduce the influence of noisy instances for generating final decision function, which induces fuzzy SVM as followings:
\begin{align}\label{FSVM}
\begin{split}
&\min_{\vec w,b,\xi}~\frac {1}{2}{\|\vec w\|}^2+c\sum_{i=1}^ls_i \xi_i\\
&\textrm{s.t.}~~y_{i}(\vec w^{T}\vec x_{i}+b)+\xi_{i} \geq 1 ,\xi_{i} \geq 0, i=1,2,\dotsc, l.
\end{split}
\end{align}
It is noted that a smaller $s_i$ can reduce
the effect of the parameter $\xi_i$ in problem Eq.\eqref{FSVM} so that the corresponding instance $\vec x_i$ can be treated
as less important. The classification of any testing instance $\vec x$ can be obtained by determining the sign of ${\vec w^{*}}^{\rm T}{\vec x}+ b^{*}$ where $\vec w^{*}$ and $b^{*}$ are the solution of Eq.\eqref{FSVM}.

\section{The proposed FR-TSVM approach}\label{sec:app}
In this section, we will propose the approach of FR-TSVM. To this end, we first introduce fuzzy membership construction. Then, we propose linear and nonlinear FR-TSVM models by embedding fuzzy membership into TSVM. Finally, a fast optimization method is developed for solving the dual problems of the proposed FR-TSVM.

\subsection{Fuzzy membership construction}\label{sec:fc}
\begin{figure}
 \centering
 {\includegraphics[width= 0.45\textwidth]{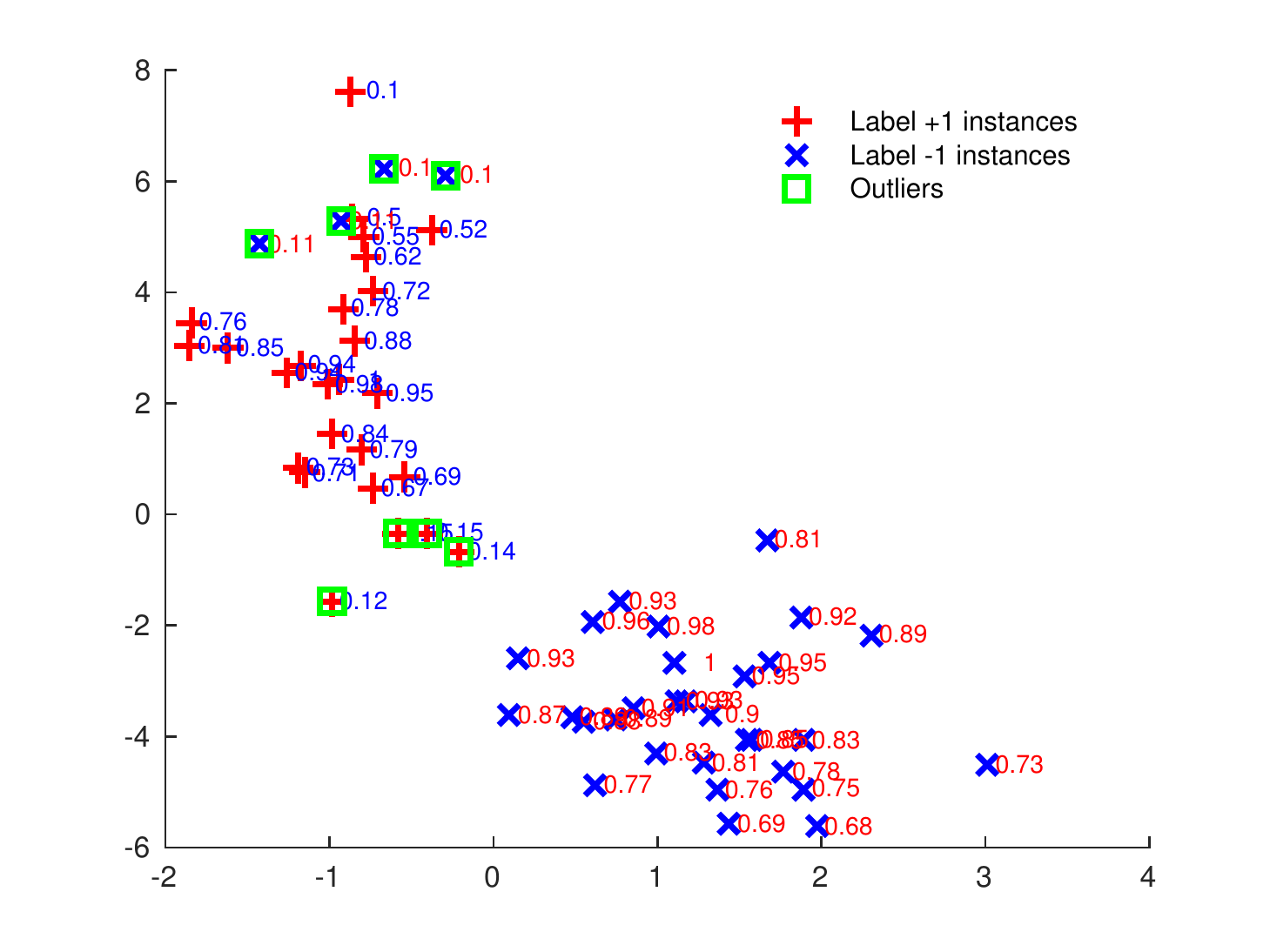}}
% \subfloat[FR-TSVM]
% {\includegraphics[width= 0.5\columnwidth]{geometry_fuzzy_mebership_value_FR-TSVM_RBF}\label{fig:GI-b}}
 \caption{Linear fuzzy membership values for binary classification. The red and blue color numbers are fuzzy membership values for  negative and positive instances, respectively. Fuzzy membership values of instances which close to class center are always larger. \label{fig:fv}}
\end{figure}

Fuzzy membership plays a key role in robust classification learning. However, there is no unified standard to construct such fuzzy membership so far. As we know, support vectors geometrically locate near the boundary area of two adjacent classes. The noisy points also reside in this area typically and unfortunately. This means that support vectors and noisy points are frequently mixed together. Inspired by Tang~\cite{Tang11}, we propose a fuzzy membership assignment for training instances. The proposed method considers not only reducing the noise carried by the outliers but also keeping the importance of support vectors.

\subsubsection{Linear case}
The construction of fuzzy membership considers firstly linear kernel case. we define directly positive and negative class centers, $\vec x_{c+}$ and $\vec x_{c-}$, as the mean points in the input space for positive and negative instances,

\begin{equation}
\vec x_{c+}=\frac{1}{l_+}\sum_{y_i = +1}\vec x_i\,,
\vec x_{c-}=\frac{1}{l_-}\sum_{y_i = -1}\vec x_i\,.
\end{equation}
Measuring the distance between instances and their class center, the hypersphere radius of positive and negative are defined and given as
\begin{equation}
r_+=\max\{\parallel \vec x_i-\vec x_{c+} \parallel|y_i=+1\}\,,
\end{equation}
\begin{equation}
r_-=\max\{\parallel \vec x_i-\vec x_{c-} \parallel|y_i=-1\}\,.
\end{equation}
With known $\vec x_{c+}, \vec x_{c-}$,
$r_+$, and $r_-$, membership $s_i$ of an instance can be assigned
according to the relationship between $\parallel \vec x_i-\vec x_{c+}
\parallel$ and $\parallel \vec x_i-\vec x_{c-}
\parallel$. In this paper, $s_{i}$ of a positive instance be given as:
\begin{equation}\label{eq:liner-fmv}
s_{i}=\left\{\begin{array}{llll}
\mu\big(1-{\parallel \vec x_i-\vec x_{c+}\parallel}/({r_++\delta)}\big),\\ \quad\textrm{if} \parallel \vec x_i-\vec x_{c+} \parallel\geq \parallel \vec x_i-\vec x_{c-} \parallel \& ~y_i =+1\\
\big(1-\mu\big)\big(1-{\parallel \vec x_i-\vec x_{c+}\parallel}/({r_++\delta)}\big),\\ \quad\textrm{if} \parallel \vec x_i-\vec x_{c+} \parallel <
\parallel \vec x_i-\vec x_{c-} \parallel \& ~y_i =+1
\end{array}
\right.
\end{equation}
where $\mu \in [0,1]$ is used to balance the effect of normal and noisy instances, and $\delta>0$ to avoid $s_i=0$. The relationship reveals that an instance is generally assigned by a proportionally decreasing $s_i$ value when it drifts farther from its native class center to increase uncertainty. Moreover, some instances are highly suspected as outliers which dwell with a sufficient far distance from their native class center(\textit{i.e.} $\parallel \vec x_i-\vec x_{c+} \parallel\geq \parallel \vec x_i-\vec x_{c-} \parallel$). In order to decrease outliers' effect toward the hyperplane, we assign a small positive real number for $\mu$. In practice, we set $u=0.1$ for FR-TSVM. A similar rule can be applied to the fuzzy membership of those negative instances. As shown in Fig.~\ref{fig:fv}, we intuitively show linear fuzzy membership values for training instances.

\subsubsection{Nonlinear case}
For nonlinear case, fuzzy membership function must be
consistently reconstructed in the feature space $\mathcal{H}$.
Similar to linear case, positive and negative class
centers $\vec \varphi_{c+}$ and $\vec \varphi_{c-} $ in the
feature space are defined as:
\begin{equation}
\vec \varphi_{c+}=\frac{1}{l_+}\sum_{\mathclap{y_i = +1}}\varphi(\vec x_i)\,,
\vec \varphi_{c-}=\frac{1}{l_-}\sum_{\mathclap{y_i = -1}}\varphi(\vec x_i)\,.
\end{equation}
where $\varphi(\vec x_i)\in\mathcal{H}$ denotes the
transformation of an arbitrary input instance $\vec x_i$. The
squared distance from $\varphi(\vec x_i)$ to $\vec \varphi_{c+}$
or $\vec \varphi_{c-}$ can be rearranged and expressed in terms of
the kernel function $\kappa(\cdot,\cdot)$:
\begin{align}
&{\parallel\varphi(\vec x_i)-\vec \varphi_{c+}\parallel}^2\nonumber\\
&\hspace{4mm}={\|\varphi(\vec x_i)\|^2-2\langle \varphi(\vec x_i),\vec \varphi_{c+}\rangle+\|\vec \varphi_{c+}\|^2}\nonumber\\
&\hspace{4mm}=\langle\varphi(\vec x_i),\varphi(\vec x_i)\rangle-\frac{2}{l_+}\sum_{\mathclap{y_j = +1}}\langle \varphi(\vec x_i),\varphi(\vec x_j)\rangle\nonumber\\
&\hspace{8mm}+\frac{1}{l_+^2}\sum_{{y_i = +1}}^{}\sum_{y_k = +1}\langle\varphi(\vec x_i),\varphi(\vec x_k)\rangle\nonumber\\
&\hspace{4mm}=\kappa(\vec x_i,\vec x_i)-\frac{2}{l_+}\sum_{\mathclap{y_j = +1}}\kappa(\vec x_i,\vec x_j)\nonumber\\
&\hspace{8mm}+\frac{1}{l_+^2}\sum_{y_i = +1}\sum_{y_k = +1}\kappa(\vec x_i,\vec x_k),
\end{align}
and
\begin{align}
&{\parallel\varphi(\vec x_i)-\vec \varphi_{c-}\parallel}^2\nonumber\\
&\hspace{4mm}=\kappa(\vec x_i,\vec x_i)-\frac{2}{l_-}\sum_{\mathclap{y_j = -1}}\kappa(\vec x_i,\vec x_j)\nonumber\\
&\hspace{8mm}+\frac{1}{l_-^2}\sum_{y_i = -1}\sum_{y_k = -1}\kappa(\vec x_i,\vec x_k),
\end{align}
where $\kappa(\vec x_i.\vec x_j)$ is the kernel function which implicitly
calculates the high-dimensional dot-product of $\varphi(\vec x_i)$
and $\varphi(\vec x_j)$. The scattering hypersphere radii in the
feature space are
\begin{equation}
r_{\varphi+}=\max\{\parallel \varphi (\vec {x}_{i})-\vec \varphi_{c+}
\parallel|y_i=+1\}\,,
\end{equation}
\begin{equation}
r_{\varphi-}=\max\{\parallel \varphi(\vec {x}_{i})-\vec \varphi_{c-}
\parallel|y_i=-1\}\,.
\end{equation}
Based on the same principle in linear case, the fuzzy membership function
for non-linear kernel can equivalently given as:
\begin{equation}\label{eq:noliner-fmv}
s_i=\left\{\begin{array}{ll}
\mu\big(1-\sqrt{{\parallel\varphi(\vec x_i)-\vec \varphi_{c+}\parallel}^2/(r_{\varphi+}^2+\delta)}\big), \\
\textrm{if}~\parallel\varphi(\vec x_i)-\vec \varphi_{c+}\parallel \geq \parallel\varphi(\vec x_i)-\vec \varphi_{c-}\parallel \& ~y_i =+1\\
\big(1-\mu\big)\big(1-\sqrt{{\parallel\varphi(\vec x_i)-\vec \varphi_{c+}\parallel}^2/(r_{\varphi+}^2+\delta)}\big), \\
\textrm{if}~\parallel\varphi(\vec x_i)-\vec \varphi_{c+}\parallel <
\parallel\varphi(\vec x_i)-\vec \varphi_{c-}\parallel\& ~y_i =+1
\end{array}
\right.
\end{equation}
where $\delta$ is similarly defined as a small positive constant to
avoid the vanishing of $s_{i}$. Of course, a fuzzy membership function
$s_{i}$ for negative class instances can be similarly defined.

\subsection{Fast and Robust twin support vector machine}\label{sec:R-TSVM}
In this section, we propose an efficient learning approach termed as fast and robust twin support vector machine~(FR-TSVM). As mentioned earlier, the FR-MSVM is similar to TSVM, as it also derives a pair of nonparallel decision hyperplanes through two QPPs. What is more, FR-TSVM is more robust and fast than TSVM.

\subsubsection{Linear FR-TSVM}
For linear case, the FR-TSVM finds two nonparallel hyperplanes in $\mathcal{R}^n$ space

\begin{equation}
\vec w_{+}^{\rm T}\vec x+b_{+}=0,\quad and \quad  \vec w_{-}^{\rm T}\vec x+b_{-}=0.
\end{equation}
Considering the crucial trade-off balance between margin maximization and error minimization, a margin term similar to that in the standard SVM~\cite{Vapnik98}, should be added firstly in the model. Since TSVM has two proximal decision functions, $\vec w_{\pm}^{\rm T}\vec x+b_{\pm}=0$, two margin terms $1/\|\vec w_+\|$ and $1/\|\vec w_-\|$ are accordingly defined for proximal decision functions, respectively. Together with the introduced fuzzy membership function and the margin terms, a weight regularized model of FR-TSVM with the linear kernel is hence proposed:
\begin{align}\label{eq:p-rtsvm1}
&\min_{\vec w_+,b_+,\vec \xi_-}~~\dfrac{1}{2}c_1\parallel \vec w_+\parallel^2+ \dfrac{1}{2}\parallel X_+\vec w_++\vec e_+b_+\parallel^2+c_3\vec s_-^{\rm T}\vec \xi_-\nonumber\\
&\textrm{s.t.}~~-(X_-\vec w_++\vec e_-b_+)+\vec \xi_-\geq \vec e_-,~\vec \xi_-\geq \vec {0},
\end{align}
\begin{align}\label{eq:p-rtsvm2}
&\min_{\vec w_-,b_-,\vec \xi_+}~~\dfrac{1}{2}c_2\parallel \vec w_-\parallel^2+\dfrac{1}{2}\parallel X_-\vec w_-+\vec e_-b_-\parallel^2+c_4\vec s_+^{\rm T}\vec \xi_+\nonumber\\
&\textrm{s.t.}~~~(X_+\vec w_-+\vec e_+b_-)+\vec \xi_+\geq \vec e_+,~\vec \xi_+\geq \vec {0},
\end{align}
where $c_i>0(i=1,2,3,4)$ are trade-off parameters for weighting the regularization, $\vec \xi_+$ and $\vec \xi_-$ denote the subsets of misclassification error for positive and negative classes respectively, both $\vec s_+\in R^{l_+}$ and $\vec s_-\in R^{l_-}$  are the fuzzy-number vectors sequentially associated with positive and negative instances, and $\vec e_+, \vec e_-$ correspond to unit vectors with their dimensions exact to sample sizes in positive and negative classes. The parameters $c_i(i=1,2,3,4)$ are used to balance the effect of maximizing the margin and minimizing the adapting error which aggregates all the individual error measured from instances to its corresponding hyperplane.  An intuitive geometric interpretation for the linear FR-TSVM is shown in Fig.~\ref{fig:GI-b}.

\begin{theorem}\label{th:l-rtsvm}
The dual forms of the primal problems Eq.\eqref{eq:p-rtsvm1}-\eqref{eq:p-rtsvm2} are
\begin{equation}
\begin{split}\label{eq:d-rtsvm1}
&\max_{\vec \alpha}\quad \vec e_-^{\rm T}\vec \alpha-\frac{1}{2}\vec \alpha^{\rm T}H_-(H_+^{\rm T}H_++c_1E_1)^{-1}H_-^{\rm T}\vec \alpha \\
&\textrm{s.t.}\quad \vec 0\leq\vec \alpha\leq c_3\vec s_-,
\end{split}
\end{equation}
\begin{equation}
\begin{split}\label{eq:d-rtsvm2}
&\max_{\vec \beta}\quad \vec e_+^{\rm T}\vec \beta-\frac{1}{2}\vec \beta^{\rm T}H_+(H_-^{\rm T}H_-+c_2E_2)^{-1}H_+^{\rm T}\vec \beta \\
&\textrm{s.t.}\quad \vec 0\leq\vec \beta\leq c_4\vec s_+,
\end{split}
\end{equation}
where $H_+ = [X_+, \vec e_+], H_- = [X_-, \vec e_-]$, and $E_i
=\left(
\begin{array}{cc}
 I& ~\\
 ~& 0\\
\end{array}
\right)$ ($i = 1, 2$). Relationships of the optimal solutions
between the primal problems~Eq.\eqref{eq:p-rtsvm1}-\eqref{eq:p-rtsvm2} and their dual problems
Eq.\eqref{eq:d-rtsvm1}-\eqref{eq:d-rtsvm2} are
\begin{equation}\label{eq:u_star}
\begin{split}
&\vec u_+^* =-(H_+^{\rm T}H_++c_1E_1)^{-1}H_-^{\rm T}\pmb\alpha^*,\\
&\vec u_-^* =(H_-^{\rm T}H_-+c_2E_2)^{-1}H_+^{\rm T}\pmb\beta^*,
\end{split}
\end{equation}
where $\vec u_\pm^*=[\vec w_\pm^{*\rm T},b_\pm^*]^{\rm T}$,
$\pmb\alpha^*$ and $\pmb\beta^*$ denote the optimal values of $\vec \alpha $ and $\vec \beta$, respectively.
\end{theorem}

\begin{proof}
Taking Lagrangian of the primal problem Eq.\eqref{eq:p-rtsvm1}, the
problem becomes:
\begin{align}\label{eq:lf}
&\L(\vec w_+,b_+,\vec \xi_-)=\frac{1}{2}c_1\parallel \vec w_+\parallel^2+\frac{1}{2}\parallel X_+\vec w_++\vec e_+b_+\parallel^2\nonumber\\
&\qquad \qquad +\alpha^{\rm T}(X_-\vec w_++\vec e_-b_+-\vec \xi_-+\vec e_-)\nonumber\\
&\qquad \qquad +c_3\vec s_-^{\rm T}\vec \xi_--\vec \eta^{\rm T}\vec \xi_-,
\end{align}
where non-negative $\vec \alpha$ and $\vec \eta$ are Lagrange
multipliers. According to the KKT conditions, we have
\begin{align}
&\nabla_{\vec w_+}L=c_1\vec w_++X_+^{\rm T}(X_+^{\rm T}\vec w_++\vec e_+b_+)
+X_-^{\rm T}\vec \alpha=0,\label{eq:lw}\\
&\nabla_{b_+}L=\vec e_+^{\rm T}(X_+^{\rm T}\vec w_++\vec e_+b_+)+\vec e_-^{\rm T}\vec \alpha=0,\label{eq:lb}\\
&\nabla_{\vec \xi_-}L=c_3\vec s_--\vec \alpha-\vec \eta=0,\label{eq:san}\\
&-(X_-\vec w_++\vec e_-b_+)+\vec \xi_-\geq \vec e_-,\vec \xi_-\geq \vec 0,\\
&\vec \alpha^{\rm T}(X_-\vec w_++\vec e_-b_+-\vec \xi_-+\vec e_-)=0,\vec \eta^{\rm T}\vec \xi_-=0.
\end{align}
Summarized from presumed conditions $\vec \alpha\geq \vec 0$, $\vec \eta\geq \vec 0$, and~Eq.\eqref{eq:san}, $\vec \alpha$ is bounded by:
\begin{equation}\label{eq:d-rtsvm1-c}
\vec 0\leq\vec \alpha\leq c_3\vec s_-.
\end{equation}
Combining~Eq.\eqref{eq:lw} and~Eq.\eqref{eq:lb} yields:
\begin{equation}
([X_+,\vec e_+]^{\rm T}[X_+,\vec e_+]+c_1E_1)[\vec w_+^{\rm
T},b_+]^{\rm T}+[X_-,\vec e_-]^{\rm T}\vec \alpha=0
\end{equation}
\begin{equation}\label{eq:p-d}
\emph{i.e.}, (H_+^{\rm T}H_++c_1E_1)^{-1}\vec {u_+}+H_-^{\rm T}\vec \alpha=0.
\end{equation}
Substituting~Eq.\eqref{eq:p-d} into the Lagrange function~Eq.\eqref{eq:lf} yields:
\begin{align}\label{eq:d-rtsvm1-of}
&\L(\vec w_+,b_+,\vec \xi_-)=\frac{1}{2}c_1E_1\vec u_+^{\rm T}\vec u_++\frac{1}{2}(H_+\vec u_+)^{\rm T}(H_+\vec u_+)\nonumber\\
&\hspace{10mm}+\vec \alpha^{\rm T}H_-\vec u_++\vec e_-^{\rm T}\vec \alpha\nonumber\\
&=\vec e_-^{\rm T}\vec \alpha-\frac{1}{2}\vec \alpha^{\rm
T}H_-(H_+^{\rm T}H_++c_1E_1)^{-1}H_-^{\rm T}\vec \alpha.
\end{align}
Combine the maximization objective in~Eq.\eqref{eq:d-rtsvm1-of} and the constraints in~Eq.\eqref{eq:d-rtsvm1-c}, and we eventually obtain the Wolfe dual form of the problem as
that in~Eq.\eqref{eq:d-rtsvm1}. Similarly, the Wolfe dual form~Eq.\eqref{eq:d-rtsvm2} of the primal
problem~Eq.\eqref{eq:p-rtsvm2} can also be proved accordingly. Despite of these dual
forms, the relationships between the optimal solutions $\vec u_\pm^*$  of the primal problems and those $\vec \alpha^*$and $\vec \beta^*$ of the dual problems illustrated in~Eq.\eqref{eq:u_star} can also be
derived from~Eq.\eqref{eq:p-d} and the related expressions.
\end{proof}

By solving the dual forms of~Eq.\eqref{eq:d-rtsvm1} and~Eq.\eqref{eq:d-rtsvm2}, one can obtain the optimal
solutions $\vec \alpha^*$ and  $\vec \beta^*$ of the dual problems,
and furthermore  $\vec u_\pm^*$ of the corresponding primal
problems. The non-parallel proximal hyperplanes can thus be
subsequently obtained. For a testing instance $\vec x\in \mathcal{R}^n$, the classification decision function can be given
as:
\begin{equation}
f(\vec x)=\argmin_\pm\frac{\mid {\vec w_\pm^*}^{\rm T}\vec x+b_\pm^*\mid}{\parallel \vec w_\pm^*\parallel}\,.
\end{equation}

\subsubsection{Nonlinear FR-TSVM}
In nonlinear case, the classification problem is intuitively solved by
mapping input instance $\vec x$ from the input space
$\mathcal{R}^n$ to a high-dimensional feature space $\mathcal{H}$
through transformation $\varphi(\vec x)$. Using alternative kernel
function $\kappa(\cdot,\cdot)$, which implicitly calculates the
dot-product of a pair of transformations $\kappa(\vec x_1,\vec x_2)
= \langle\varphi(\vec x_1),\varphi(\vec x_2)\rangle$, similarity manipulation of
transformed $\vec x_1$ and $\vec x_2$ can be resolved and utilized
to deal with the nonlinear FR-TSVM. The fact, which makes the
transformation helpful for the nonlinear FR-TSVM, is that the optimal
separating hyperplane can be constructed linearly in the
high-dimensional feature space \cite{LK}. With the kernel function, the
nonlinear dual proximal hyperplanes of FR-TSVM can be stated as:
\begin{equation}
\kappa(\vec x,X^{\rm T})\vec w_++b_+=0~\textrm{and}~\kappa(\vec x,X^{\rm T})\vec w_-+b_-=0\,.
\end{equation}
To obtain the above two hyperplanes, the primal
problems of nonlinear FR-TSVM can be expressed as:
\begin{align}\label{eq:pn-rtsvm1}
&\min_{\vec w_+,b_+,\vec \xi_-}~\frac{1}{2}c_1\parallel \vec w_+\parallel^2+\frac{1}{2}\parallel \kappa(X_+,X^{\rm T})\pmb w_++\pmb e_+b_+\parallel^2 \nonumber\\
&\quad \quad \quad \quad +c_3\vec s_-^{\rm T}\vec \xi_- \nonumber\\
&\textrm{s.t.}~ -(\kappa(X_-,X^{\rm T})\vec w_++\vec e_-b_+)+\vec \xi_-\geq \vec e_-, \vec \xi_-\geq 0,
\end{align}
\begin{align}\label{eq:pn-rtsvm2}
&\min_{\vec w_-,b_-,\vec \xi_+}~\frac{1}{2}c_2\parallel \vec w_-\parallel^2+\frac{1}{2}\parallel \kappa(X_-,X^{\rm T})\pmb w_-+\pmb e_-b_-\parallel^2\nonumber\\
&\quad \quad \quad \quad +c_4\vec s_+^{\rm T}\vec \xi_+\nonumber\\
&\textrm{s.t.}\quad (\kappa(X_+,X^{\rm T})\vec w_-+\vec e_+b_-)+\vec \xi_+\geq \vec e_+, \vec \xi_+\geq 0.
\end{align}

\begin{theorem}\label{th:r-rtsvm}
The dual forms of the primal problems~Eq.\eqref{eq:pn-rtsvm1} and~Eq.\eqref{eq:pn-rtsvm2} are
\begin{equation}\label{eq:dn-rtsvm1}
\begin{split}
&\max_{\vec \alpha}~~~\vec e_-^{\rm T}\vec \alpha-\frac{1}{2}\vec \alpha^{\rm T}S_+(S_-^{\rm T}S_-+c_1E_1)^{-1}S_+^{\rm T}\vec \alpha\\
&\textrm{s.t.}~~~~~~~\vec 0\leq\vec \alpha\leq c_3\vec s_-,\\
\end{split}
\end{equation}
\begin{equation}\label{eq:dn-rtsvm2}
\begin{split}
&\max_{\vec \beta}~~~\vec e_+^{\rm T}\vec \beta-\frac{1}{2}\vec \beta^{\rm T}S_-(S_+^{\rm T}S_++c_2E_2)^{-1}S_-^{\rm T}\pmb \beta\\
&\textrm{s.t.}~~~~~~~\vec 0\leq\vec \beta\leq c_4\vec s_+,
\end{split}
\end{equation}
where $S_+=[\kappa(X_+,X^{\rm T}),\vec e_+]$ and
$S_-=[\kappa(X_-,X^{\rm T}),\vec e_-]$. By designating $\vec v_\pm^*=[\pmb w_\pm^{*\rm T}, b_\pm^*]^{\rm T}$ for solutions of the primal problems of~Eq.\eqref{eq:pn-rtsvm1} and~Eq.\eqref{eq:pn-rtsvm2}, there are parametric relationships
between the optimal $\vec v_\pm^*$ and the optimal solutions $\vec \alpha^*$ and $\pmb \beta^*$  of their corresponding dual forms~Eq.\eqref{eq:dn-rtsvm1} and~Eq.\eqref{eq:dn-rtsvm2}:
\begin{equation}\label{eq:v_star}
\begin{split}
&\vec v_+^*=-(S_+^{\rm T}S_++c_1E_1)^{-1}S_-^{\rm T}\vec \alpha^*,\\
&\vec v_-^*=(S_-^{\rm T}S_-+c_2E_2)^{-1}S_+^{\rm T}\vec \beta^*.
\end{split}
\end{equation}
\end{theorem}

\begin{proof}
Referring to the proof of Theorem~\ref{th:l-rtsvm} for linear FR-TSVM, the
proof of Theorem~\ref{th:r-rtsvm} for nonlinear FR-TSVM can be derived accordingly
following the steps in~Eq.\eqref{eq:lf}-\eqref{eq:d-rtsvm1-of}.
\end{proof}

Once solutions of the dual problems~Eq.\eqref{eq:dn-rtsvm1} and~Eq.\eqref{eq:dn-rtsvm2} are obtained,
solutions of the primal problems~Eq.\eqref{eq:pn-rtsvm1} and~Eq.\eqref{eq:pn-rtsvm2} can be obtained through~Eq.\eqref{eq:v_star}, and
the decision function for classifying a testing instance $\vec x\in
\mathcal{R}^n$ is eventually given by:
\begin{equation}
f(\vec x)=\argmin_\pm\frac{\mid \kappa(\vec x,X^{\rm T})\vec
{w}{_\pm^*}^{\rm T}+b_\pm^*\mid}{ {\sqrt{\vec {w}_\pm^{*{\rm
T}}\kappa(X,X^{\rm T})\vec {w}_\pm^*}}}.
\end{equation}

\subsection{A fast optimization method for FR-TSVM}\label{sec:fo}
\subsubsection{Solving FR-TSVM with the pure \emph{coordinate descent}}\label{sec:CD}
Based on the quadratic differentiable expressions of the FR-TSVM's
objective functions~Eq.\eqref{eq:d-rtsvm1}-\eqref{eq:d-rtsvm2} and~Eq.\eqref{eq:dn-rtsvm1}-\eqref{eq:dn-rtsvm2}, a \emph{coordinate descent}
method~\cite{Chang08} can be further employed  for  solving the FR-TSVM. There are pairwise similarities between~Eq.\eqref{eq:d-rtsvm1}-\eqref{eq:d-rtsvm2} and~Eq.\eqref{eq:dn-rtsvm1}-\eqref{eq:dn-rtsvm2}. The intuition is that if either one of these functions
can be reformulated as a quadratic expression, the
\emph{coordinate descent} method would be applied accordingly with the
algorithms proposed by~\cite{Chang08,Hsieh08,Shao12}, and be easily extended to the other
three objective functions. By the motivation, we initially show the dual FR-TSVM
with the first objective function of~Eq.\eqref{eq:d-rtsvm1} as below.

\begin{algorithm}[t]
\caption{A dual CD method for FR-TSVM}\label{algo:FR-TSVM1}
\begin{algorithmic}[1]
\STATE Compute $Q=(H_+^{\rm T}H_++c_1E_1)^{-1}H_-^{\rm T}$ and $\overline Q_{ii}=H_{-i}Q_i$\\
\STATE Initial $\pmb\alpha \leftarrow \vec 0$ and $\vec u_+\leftarrow \vec 0$\
\WHILE {$\vec \alpha$ is not optimized}
\FOR {$i=1,2,\dotsc,l_-$}
\STATE $\nabla_i f(\vec \alpha)=-H_{-i}\vec u_+-1$
\STATE Compute$\nabla_i^{p} f(\vec \alpha)$ by ~Eq.\eqref{eq:nabla_pp}
\IF {$\nabla_i^{p} f(\vec \alpha)\neq 0$}
\STATE $\overline{\alpha}_i\leftarrow \alpha_i$
\STATE $\alpha_i\leftarrow\min(\max(\alpha_i-\nabla_i^{p} f(\vec \alpha)/\overline Q_{ii},0),c_3s_{i-})$
\STATE $u_{+i}\leftarrow u_{+i}-Q_i(\alpha_i-\overline \alpha_i )$
\ENDIF
\ENDFOR
\ENDWHILE
\end{algorithmic}
\end{algorithm}

With $Q=(H_+^{\rm T}H_++c_1E_1)^{-1}H_-^{\rm T}$ and $\overline{Q}=H_-Q$, the
problem~Eq.\eqref{eq:d-rtsvm1} can be abbreviated as a quadratic
expression:
\begin{equation}
\begin{split}
&\min_{\vec \alpha}~~~f(\vec \alpha)=\frac{1}{2}\vec \alpha^{\rm T}\overline{Q}\vec \alpha-\vec e_-^{\rm T}\vec \alpha\\
&\textrm{s.t.}~~~~\vec 0\leq\vec \alpha\leq c_3\vec s_-.
\end{split}
\end{equation}
As an iterative scheme, the FR-TSVM generates subsequently a sequence
of updating vectors  $\{\vec \alpha^0,\ldots,\vec \alpha^k,\vec \alpha^{k+1}\}$ to consecutively optimize the
objective function where $\vec \alpha^k=[\vec \alpha_1^k,\vec \alpha_2^k,\ldots,\vec \alpha_{l_-}^k]^{\rm T}\in \mathcal{R}^{l_-}$. There are
two levels of iterations. An integer $k$ first is used to index
the 2nd-level of outer iterations updated from $\vec \alpha^k $ to $
\vec \alpha^{k+1}$. In every $k$-th outer iteration, the update of
$\vec \alpha^k$ is further subdivided into 1st-level of $l_-$ inner
iterations, indexed by $i$, to generate a series of intermediate
vectors 
\begin{equation}
\{\vec \alpha^{k,1},\vec \alpha^{k,2},\ldots,\vec \alpha^{k,i},\ldots,\vec \alpha^{k,l_-},\vec \alpha^{k,l_-+1}\}.
\end{equation}
The two-level updated vector $\vec \alpha^{k,i}$ is thus expressed as:
\begin{align}
&\vec \alpha^{k,i}=[\alpha_1^{k+1},\ldots,\alpha_{i-1}^{k+1},\alpha_{i}^{k},\ldots,\alpha_{l_-}^{k}]^{\rm T},\\
&~~~~~~~~\forall i=1,2,\ldots,l_-~~~~~~~\nonumber
\end{align}
and
\begin{equation}
\vec \alpha^{k,1}=\vec \alpha^k~and ~\vec \alpha^{k,l_-+1}=\vec \alpha^{k+1}.
\end{equation}
To update the intermediate $\vec \alpha^{k,i}$ to $\vec \alpha^{k,i+1}$, the following single variable subproblem should be
solved:
\begin{equation}\label{eq:sub_qp}
\min_d f(\vec \alpha^{k,i}+d\vec e_i)\quad \textrm{s.t.}\quad
0\leq \alpha_{\it i}^{\it k}+\it d\leq c_3s_{\it i-},
\end{equation}
where $e_i$ is the $i$-th orthogonal basis vector of $\mathcal{R}^{l_-}$ space. Indeed, the objective function~Eq.\eqref{eq:sub_qp}
corresponds to a quadratic function of $d$:
\begin{equation}\label{eq:sub_qpp}
f(\vec \alpha^{k,i}+d\vec e_i)=\frac{1}{2}\overline Q_{ii}
d^2+\nabla_i f(\vec \alpha^{k,i})d+c,
\end{equation}
where $\nabla_i f$  denotes the $i$-th component of gradient $\vec \nabla f$ , and $c$ is an arbitrary constant. Apparently,~Eq.\eqref{eq:sub_qpp} has
an optimum at $d=0$ if and only if
\begin{equation}\label{eq:nabla_p}
\nabla_i^{p} f(\vec \alpha^{k,i})=0,
\end{equation}
where  $\nabla_i^{p}f(\vec \alpha^{k,i})$ is a projected
gradient. To gain the possibility to refine the optimum, the project
gradient should be satisfactory with:
\begin{equation}\label{eq:nabla_pp}
\nabla_i^{p} f(\vec \alpha^{k,i})=\left\{\begin{array}{ll}
\min(0,\nabla_i f(\vec \alpha)),&\textrm{if}~\alpha_i=0\\
\nabla_i f(\vec \alpha),&\textrm{if}~0< \alpha_i< c_3s_{i-} \\
\max(0,\nabla_i f(\vec \alpha)),&\textrm{if}~\alpha_i=c_3s_{i-}
\end{array}
\right.
\end{equation}
The key for computational reduction is that we can directly move
forward to next $i+1$ iteration without updating $\alpha_i^{k,i}$ in
the $l_-$length inner-iteration updates if~Eq.\eqref{eq:nabla_p} has been fulfilled.
In other words, we only update  $\alpha_i^{k,i}$  to temporally meet
the optimal solution of~Eq.\eqref{eq:sub_qp}. By introducing Lipschitz
continuity~\cite{Searcid}, the optimum of~Eq.\eqref{eq:sub_qpp} can be reached by:
\begin{equation}\label{eq:alpha}
\alpha_i^{k,i+1}=\min(\max(\alpha_i^{k,i}-{\nabla_i f(\vec \alpha^{k,i})}/{\overline Q_{ii}},0),c_3s_{i-}).
\end{equation}
%with condition of $\overline Q_{ii}>0$.
However $\alpha_i^{k,i+1}$ is
updated or not in the $l_-$-length inner iterations; the process
would be repeated in the outer iterations once and once again until
the presumed termination condition is reached. In the update of
~Eq.\eqref{eq:alpha}, $\overline Q_{ii}$ can be pre-calculated by  $\overline
Q_{ii}=H_{-i}Q_i$, where $Q=(H_+^{\rm T}H_++c_1E_1)^{-1}H_-^{\rm
T}$, and preserved through all the iterations, and $\nabla_if(\vec \alpha^{k,i})$ can be obtained by
\begin{equation}\label{eq:nabla-alpha}
\nabla_i f(\vec \alpha)=(\overline Q \vec \alpha)_i-1=\sum_{j=1}^{l_-}\overline Q_{ij}\alpha_j-1.
\end{equation}
Here, the computation of $\nabla_if(\vec \alpha^{k,i})$  by~Eq.\eqref {eq:nabla-alpha},
which is approximated as $O(l_-\overline n)$ where $\overline n$ is
the average count of non-zero elements in $\overline Q$  per
instance, is expensive. In order to reduce the computation,
$\nabla_if(\vec \alpha^{k,i})$  can alternatively be calculated by
\cite{Hsieh08}:
\begin{equation}\label{eq:nabla-u}
\nabla_i f(\vec \alpha)=-H_{-i}\vec u_+-1,
\end{equation}
with a pre-defined $\vec u_+$
\begin{equation}\label{eq:u}
\vec u_+=- Q \vec {\alpha},
\end{equation}
where$H_{-i}$ is the $i$-th row of matrix $H_-$. With this
alternative, the time complexity of computing  $\nabla_if(\vec \alpha^{k,i})$ can be reduced as $O(\overline n)$ . In order to
employ~Eq.\eqref{eq:nabla-u} for calculating $\nabla_if(\vec \alpha^{k,i})$ , it is
required to maintain $\vec u_+$ throughout the whole coordinate
descent procedure by an update policy:
\begin{equation}\label{eq:up-u}
u_{+i}\leftarrow u_{+i}-Q_i(\alpha_i-\overline \alpha_i).
\end{equation}
Here, the time consumption by the iterative maintaining of $\vec u_+$ requires only  $O(\overline n)$ rather than that by the direct
calculation by~Eq.\eqref{eq:u}, where $\overline \alpha_i$ and $ \alpha_i$
denote the values of the primal optimizer before and after the
corresponded update iteration, respectively. With a zero initial
value for the first $\vec u_+$ due to the generally adopted $\vec \alpha^0=\vec 0$, the optimal solution of $\vec u_+$ for the primal
problem Eq.\eqref{eq:d-rtsvm1} can be eventually obtained by iterative updates of
~Eq.\eqref{eq:up-u}. Algorithm \ref{algo:FR-TSVM1} describes the entire process.

\subsubsection{Speeding-up FR-TSVM with heuristic shrinking}\label{sec:FR-TSVM}
Although the quadratic expressions of FR-TSVM inherit most
essential merits from the convex quadratic optimization, the
solutions of FR-TSVM, even the temporal solutions $\alpha^k$, are still
constrained in a specified range, for example, $0\leq \alpha_i \leq
c_3s_{i-}$  for the problem~Eq.\eqref{eq:d-rtsvm1}. With the bounded Lagrange
multipliers, $\alpha_i=0$  or $\alpha_i=c_3s_{i-}$ , further
iterative effort may vanish and remain the temporal objective
function outcome in a steady state. Since FR-TSVM produces
relatively considerable amounts of bounded Lagrange multipliers in
the iterative process, a policy to early stop the update of such
bounded multipliers of reducing the scale of optimization programming
is indeed beneficial to the subsequent computation, and speed-ups the FR-TSVM.

At the same time, for simplicity, we examine only the \emph{coordinate descent} of~Eq.\eqref{eq:d-rtsvm1}
in the twin  constrained optimization problems~Eq.\eqref{eq:d-rtsvm1}-\eqref{eq:d-rtsvm2} and
~Eq.\eqref{eq:dn-rtsvm1}-\eqref{eq:dn-rtsvm2} with the heuristic shrinking technique~\cite{Joachims98}. Once the
heuristic shrinking is applicable, the examination of~Eq.\eqref{eq:d-rtsvm2} and
~Eq.\eqref{eq:dn-rtsvm1}-\eqref{eq:dn-rtsvm2} can be equivalently conducted and a comparable result can
be assessed with the similarity to each other.

To employ the heuristic shrinking~\cite{Joachims98}, a subset removing some
elements from $\{1,2,\cdots,l_-\}$ is defined as an active set $A$,
and the complimentary subset is defined contradictorily as an
inactive set $\overline{A}=\{1,\cdots,l_-\}\backslash A$. The use of
the active set is for dynamically collecting those non-bounded
Lagrange multiplier still effective to the optimization. With the
separation of $A$ from $\overline{A}$ , the optimization problem~Eq.\eqref{eq:d-rtsvm1}
can be decomposed and reorganized as:
\begin{equation}\label{eq:sub-qp}
\begin{split}
&\min\limits_{\vec \alpha_A} ~~~\frac{1}{2}{\vec \alpha_A^{\rm
T}\overline Q_{AA}\vec \alpha_A}+
(\overline Q_{A\overline A}{\vec \alpha_{\overline A}}-\vec e_A)^{\rm T}\vec \alpha_A\\
&\textrm{s.t.} ~~~~~~~\vec 0\leq \vec \alpha_A \leq c_3\vec s_{A-},
\end{split}
\end{equation}
where  $\overline Q_{AA}$ and $\overline Q_{A\overline A}$ are
sub-matrices of $\overline Q$ , and  $\vec \alpha_A$ and $\vec \alpha_{\overline A}$  are Lagrange multiplier sub-vectors
corresponding to subsets $A$ and $\overline{A}$, respectively. As
illustrated previously,  $\vec \alpha_{\overline A}$ contains only
those inactive bounded multipliers which can not contribute
furthermore to the optimization. A divide-and-conquer strategy is
thus used to achieve the optimization more efficiently due to the
eliminated computations of the second part of~Eq.\eqref{eq:sub-qp}. As described in
the theory of FR-TSVM, the gradient~$\nabla f(\vec \alpha)$ is a
key to update the optimizer. Following the subdivision
strategy, $\nabla f(\vec \alpha)$ in~Eq.\eqref{eq:nabla-alpha} can also be decomposed as:
\begin{align}
\nabla_i f(\vec \alpha)=\overline Q_{i,A}\vec \alpha_A+\overline
Q_{i,\overline A}\vec \alpha_{\overline A}-1
\end{align}
and only the gradient of those  $i\in A$ would be paid attention to
because they would effectively update the corresponding Lagrange
multiplier. The gradients of those $i\notin A$ are never required to be
recalculated and of course the updates of $\vec u_+$ are no longer
needed.

Here, we should note that the length-variable active set, chosen to
handle constraints, is dynamically managed by the \emph{coordinate descent}
procedure. It should be kept in mind that a nonzero gradient is a necessary condition for an ongoing optimization whether the
optimization is constrained or not. The rule is still true for a
\emph{coordinate descent} optimization. In FR-TSVM, what we have are
cyclical gradients $\nabla_if(\vec \alpha^{k,i})$  for
$i=1,2,\ldots,l_{-}$.  With the subsequent cyclical $\nabla_if(\vec \alpha^{k,i})$ , Theorem~\ref{th:bs} is established for bound and shrinking of
the active set according to the original theorem proposed by Hsieh
\textit{et al.}, \cite{Hsieh08}.

\begin{theorem}\label{th:bs}
By taking the list
$\{\alpha^{k,i}\}$ in the solution space, let $*$ be the final
convergent point. The FR-TSVM sustains the following properties:

\noindent a). If $\alpha_i^*=0$, and $\nabla_i f(\vec \alpha^*)>0$,
there is a $\exists k_i$ such that
\begin{equation}
\forall k\geq k_i, \forall r, \alpha_i^{k,r}=0. \nonumber
\end{equation}
 b). If $\alpha_i^*=c_3s_{i-}$ and $\nabla_i f(\vec \alpha^*)<0$, there is a
$\exists k_i $such that
\begin{equation}
\forall k\geq k_i, \forall r, \alpha_i^{k,r}=c_3s_{i-}. \nonumber
\end{equation}
c). $\lim\limits_{k\to \infty}\max\limits_i{\nabla_i^p
f(\vec \alpha^{k,i})}= \lim\limits_{k\to \infty}\min\limits_i{\nabla_i^p f(\vec \alpha^{k,i})}=0.$
\end{theorem}

\begin{proof}
Referring to the proof in Appendix 7.3 in~\cite{Hsieh08},
imitations are taken to obtain Theorem~\ref{th:bs} .
\end{proof}

Based on measures $\max_{i}\nabla_i^{p}f(\vec \alpha^k)>0,\min_{i}\nabla_i^{p}f(\vec \alpha^k)<0$ which are used to
evaluate the solution violation level of a certain outer iteration
$k$, a pair of bounds, $M^{k-1}=\max_{i}\nabla_i^{p}f(\vec \alpha^{k-1,i})$ and $m^{k-1}=\min_{i}\nabla_i^{p} f(\vec \alpha^{k-1,i})$ , are asserted for bounding $\nabla_if(\vec \alpha^{k,i})$  at the $(k-1)$-th outer iteration. The assertion is
used to seek a more specified range for rejecting more inactive
member from current $A$, in which members are allowed to participate
in the optimization, at the $k$-th outer iteration. Basically, the
active-set shrinkage relies on the pair of floating bounds to reject
those violated participants. According to properties (a) and (b) of
Theorem~\ref{th:bs}, the corresponding $i$ is excluded from $A$ to $\overline{
A} $ at the inner iteration for updating component $\alpha_i$ from
$\alpha_i^{k,i}$ to $\alpha_i^{k,i+1}$ while next two conditions are
hold:
\begin{align}
&\alpha_i^{k,i}=0~\textrm{and}~\nabla_i f(\vec \alpha^{k,i})>\overline M^{k-1},\\
&\alpha_i^{k,i}=c_3s_{i-}~\textrm{and}~\nabla_i f(\vec \alpha^{k,i})<\overline m^{k-1},
\end{align}
where
\begin{equation}
\overline M^{k-1}=\left\{\begin{array}{ll}
M^{k-1},&\textrm{if}~M^{k-1}>0\\
\infty,&\textrm{if}~M^{k-1}\leq 0
\end{array}
\right.
\end{equation}
and
\begin{equation}
\overline m^{k-1}=\left\{\begin{array}{ll}
m^{k-1},&\textrm{if}~m^{k-1}>0\\
{-\infty},&\textrm{if}~m^{k-1}\leq 0
\end{array}
\right..
\end{equation}
We use the temporal $M$ and $m$ to catch maximal and minimal
projected gradient in every cycle of inner iterations, and keep the
maximal and minimal values in $\overline{M}$ and $\overline{m}$,
respectively, to exclude inactive members in next outer iteration.
According to property (c) of Theorem~\ref{th:bs}, bounds $\overline{M}$ and
$\overline{m }$ would become closer after iterations, and would
theoretically meet with each other finally:
\begin{equation}\label{eq:stop}
\overline M^{k}=\overline m^{k}~\textrm{if}~k\rightarrow\infty.
\end{equation}

Although~Eq.\eqref{eq:stop} shows the ideal condition for terminating the procedure, the exact meeting of $\overline{M }$ and $\overline{m}$ in the numerical iterations is difficult. An alternative allowing a sufficiently small gap $\varepsilon$ and setting the following inequality:
\begin{equation}
\overline M^{k}-\overline m^{k}<\varepsilon
\end{equation}
for termination after $k$ finite iterations is much more practical.
While the gapped termination condition is reached, the optimal
solution $\vec \alpha^k$ of the sub-problem~Eq.\eqref{eq:sub-qp} is also possessed.
Actually, this optimal $\vec \alpha^k$ is only optimized for the
sub-problem~Eq.\eqref{eq:sub-qp}, not for the full problem~Eq.\eqref{eq:d-rtsvm1}. Hence, we ignore the current active set and set it to the full set $\{1,\dotsc,l_-\}$ to get back all the $\alpha_i$ from the cached historical $\alpha_i^{k,i}$ in the final pass at the end of the procedure to ensure the recomposed $\vec \alpha^*$ to fulfill~Eq.\eqref{eq:sub-qp}.

The heuristic shrinking might raise a failure risk with the mismatch
$\overline{M }$ and $\overline{m }$ even with a tolerable gapped
mismatch, for example, $\overline{M }\leq 0$, or $\overline{m }\geq
0$. If such a condition happens, the whole FR-TSVM of~Eq.\eqref{eq:sub-qp} should
be re-optimized with a different set of initial guests of $\vec \alpha$.
Additionally, the shrinking is in general applied heuristically in a
fixed sequence of the $l_-$-dimensional gradients. A random update scheme performed a more rapid convergent rate
than a sequential update scheme~\cite{Chang08}.

\begin{algorithm}[t]
\caption{The optimization of FR-TSVM}\label{algo:FR-TSVM2}
\begin{algorithmic}[1]
\STATE Compute $Q=(H_+^{\rm T}H_++c_1E)^{-1}H_-^{\rm T}$ and $\overline Q_{ii}=H_{-i}Q_i$
\STATE Let $A\leftarrow\{1,\dotsc,l_-\}$
\STATE Given $\epsilon$ and initialized $\vec \alpha\leftarrow \vec 0$, $\vec u_+\leftarrow \vec 0$
\STATE Let $\overline M\leftarrow\infty$ and $\overline m\leftarrow -\infty$
\WHILE {}
\STATE Let $ M\leftarrow -\infty$,$m\leftarrow \infty$
\FORALL {$i\in A$ (a randomly and exclusively selected)}
\STATE $\nabla_i f(\vec \alpha)=-H_{-i}\vec u_+-1$
\STATE $\nabla_i^{p} f(\vec \alpha)\leftarrow 0$
\IF {$\alpha_i=0$}
\STATE \textbf{if} {$\nabla_i^{p} f(\vec \alpha)>\overline M$}, \textbf{then} $A=A\backslash\{i\}$ \textbf{end if}
\STATE \textbf{if} $\nabla_i^{p} f(\vec \alpha)<0$, \textbf{then} $\nabla_i^{p} f(\vec \alpha)\leftarrow \nabla_i f(\vec \alpha)$ \textbf{end if}
\ELSIF {$\alpha_i=c_3s_{i-}$}
\STATE \textbf{if}  $\nabla_i^{p} f(\vec \alpha)<\overline m$, \textbf{then} $A=A\backslash\{i\}$ \textbf{end if}
\STATE \textbf{if} $\nabla_i^{p} f(\vec \alpha)>0$, \textbf{then} $\nabla_i^{p} f(\vec \alpha)\leftarrow \nabla_i f(\vec \alpha)$ \textbf{end if}
\ELSE
\STATE $\nabla_i^{p} f(\vec \alpha)\leftarrow \nabla_i f(\vec \alpha)$
\ENDIF
\STATE $M\leftarrow\max({M,\nabla_i^{p} f(\vec \alpha)})$
\STATE $m\leftarrow\min({m,\nabla_i^{p} f(\vec \alpha)})$
\IF {$\nabla_i^{p} f(\vec \alpha)\neq 0$}
\STATE $\overline{\alpha}_i\leftarrow \alpha_i$
\STATE $\alpha_i\leftarrow\min(\max(\alpha_i-\nabla_i f(\vec \alpha)/\overline Q_{ii},0),c_3s_{i-})$
\STATE $u_{+i}\leftarrow u_{+i}-Q_i(\alpha_i-\overline \alpha_i )$
\ENDIF
\ENDFOR
\IF {$M-m<\epsilon$}
\STATE \textbf{if} $A=\{1,\dotsc,l_-\}$, \textbf{break}
\ELSE
\STATE $A\leftarrow \{1,\dotsc,l_-\},\overline M\leftarrow \infty,\overline m\leftarrow -\infty$
\ENDIF
\STATE \textbf{if} $M\leq 0$, \textbf{then} $\overline M\leftarrow \infty$. \textbf{else} $\overline M\leftarrow M$ \textbf{end if}
\STATE \textbf{if} $M\geq 0$, \textbf{then} $\overline m\leftarrow -\infty$. \textbf{else} $\overline m\leftarrow m$ \textbf{end if}
\ENDWHILE
\end{algorithmic}
\end{algorithm}

According to the separation of active from inactive set, $u_+$
defined in~Eq.\eqref{eq:u} can be re-expressed as:
\begin{equation}
\vec u_+=-(Q_A\vec \alpha_A+Q_{\overline A}\vec \alpha_{\overline
A}),
\end{equation}
which means that some elements coincided in the
$\alpha_{i\overline{A}}$ and $\overline{\alpha}_{i\overline{A}}$ would
remain the same before and after the update iteration $i$, and can
be prevented in the update of $u_{+i}\leftarrow
u_{+i}-Q_i(\alpha_i-\overline { {\alpha}_i})$. Algorithm \ref{algo:FR-TSVM2} describes the update procedure of
Algorithm~\ref{algo:FR-TSVM1} with active set shrinking in a random scheme.

\section{Experiments}\label{sec:exp}
To validate the learning efficiency and generalization ability of FR-TSVM, some experiments are implemented on artificial datasets and publicly available benchmark datasets. All experiments are conducted in MATLAB~(R2014a) on a PC with an Intel Core i7 processor (3.6GHz) with 32GB RAM. The execution time and classification accuracy are fairly compared. Since the quadratic programming of SVM, TSVM, or FSVM has similar corresponding dual form, a Matlab optimization toolbox~\cite{matlab} is equally adapted for optimization. The proposed FR-TSVM is optimized by the coordinate decent with heuristic shrinking~\ref{algo:FR-TSVM2}. The Matlab code of FR-TSVM is released at~\url{https://github.com/gaobb/FR-TSVM}. 

The fuzzy membership assignments are different by utilizing~Eq.\eqref{eq:liner-fmv} for linear kernel case and~Eq.\eqref{eq:noliner-fmv} for nonlinear kernel case. In addition, the fuzzy parameter $u=0.1$ is set in~Eq.\eqref{eq:liner-fmv} and~Eq.~\ref{eq:noliner-fmv}. Under nonlinear case, Gaussian kernel $\kappa(\vec x_1,\vec x_2)=\exp(-\|\vec x_1-\vec x_2\|^2/g^2)$ is taken as kernel function, which in general outperforms other kinds of kernel functions. The model parameters $c_i(i=1,2,3,4)$ are carefully searched in the grids $\{2^i|i=-8,-7,\dotsc,8\}$ by setting $c_1=c_2$ for TSVM, and $c_1=c_2$, $c_3=c_4$ for FR-TSVM. The grid-searching is conducted in 10-folds cross-validations, randomly selecting  30\% of the whole samples for learning with the equivalent conditions
mentioned above.

\subsection{Experiments on artificial datasets}
To intuitively validate the FR-TSVM's classification performance, we firstly implement experiments on two 2-dimensional  artificial datasets and compare the proposed method with standard SVM, FSVM and TSVM.

The first dataset is artificial-generated~\emph{Ripleys synathetic}~\cite{Ripley}. The \emph{Ripleys synathetic} comprise 250 training instances and 1000 testing instances. We visualize the distribution of fuzzy membership value for training instances under linear and nonlinear case in Fig.~\ref{fig:fmd}, respectively. As shown in Fig.~\ref{fig:fmd}, compared to those instances locating near the class center, the fuzzy membership values of the instances which are far from their class center are always smaller.
\begin{figure*}
    \centering
 \subfloat[``+1" instances]
 {\includegraphics[height= 0.17\textwidth]{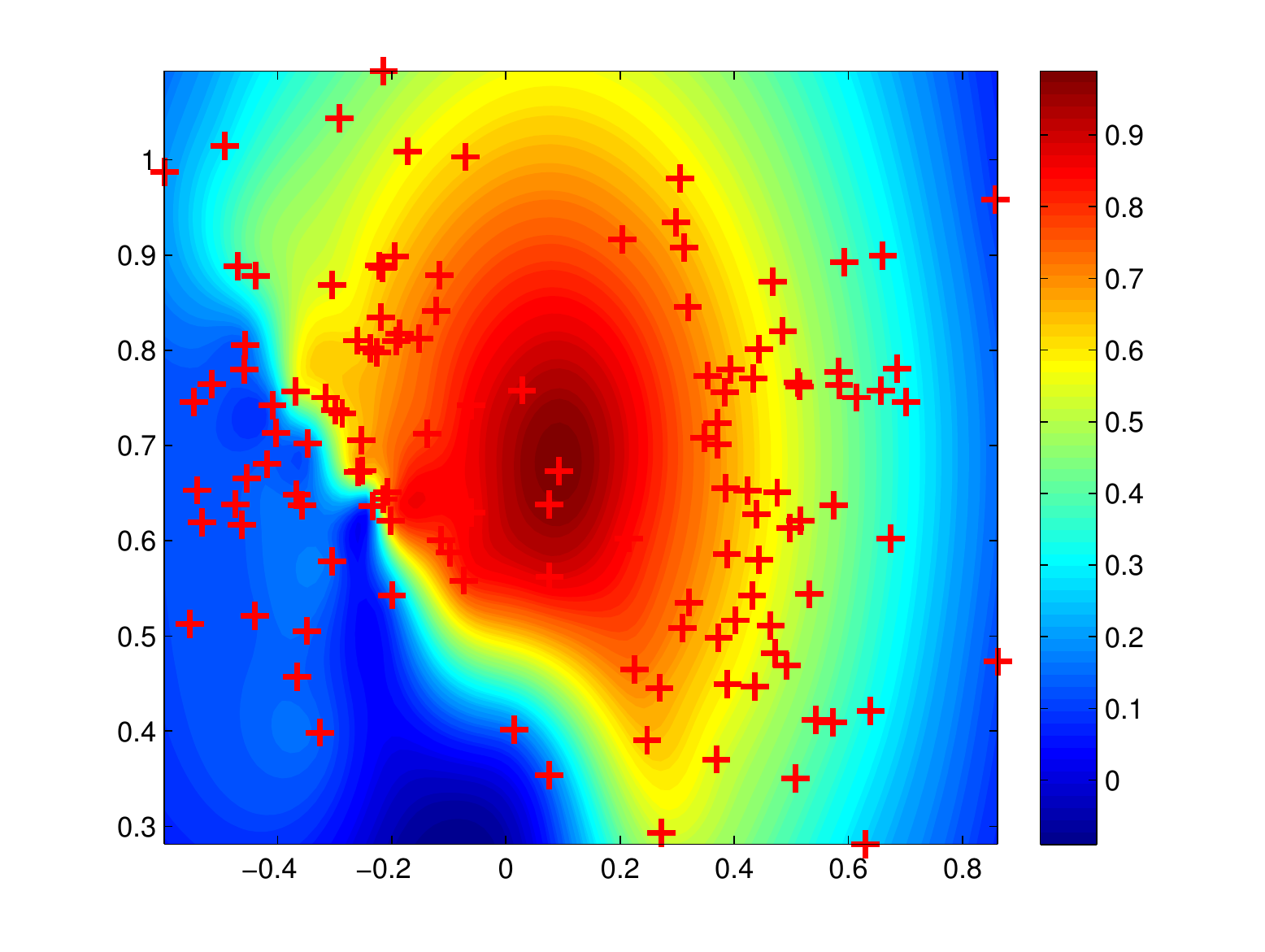}\label{fig:subfig:a}}
    \subfloat[``-1" instances]
    {\includegraphics[height= 0.17\textwidth]{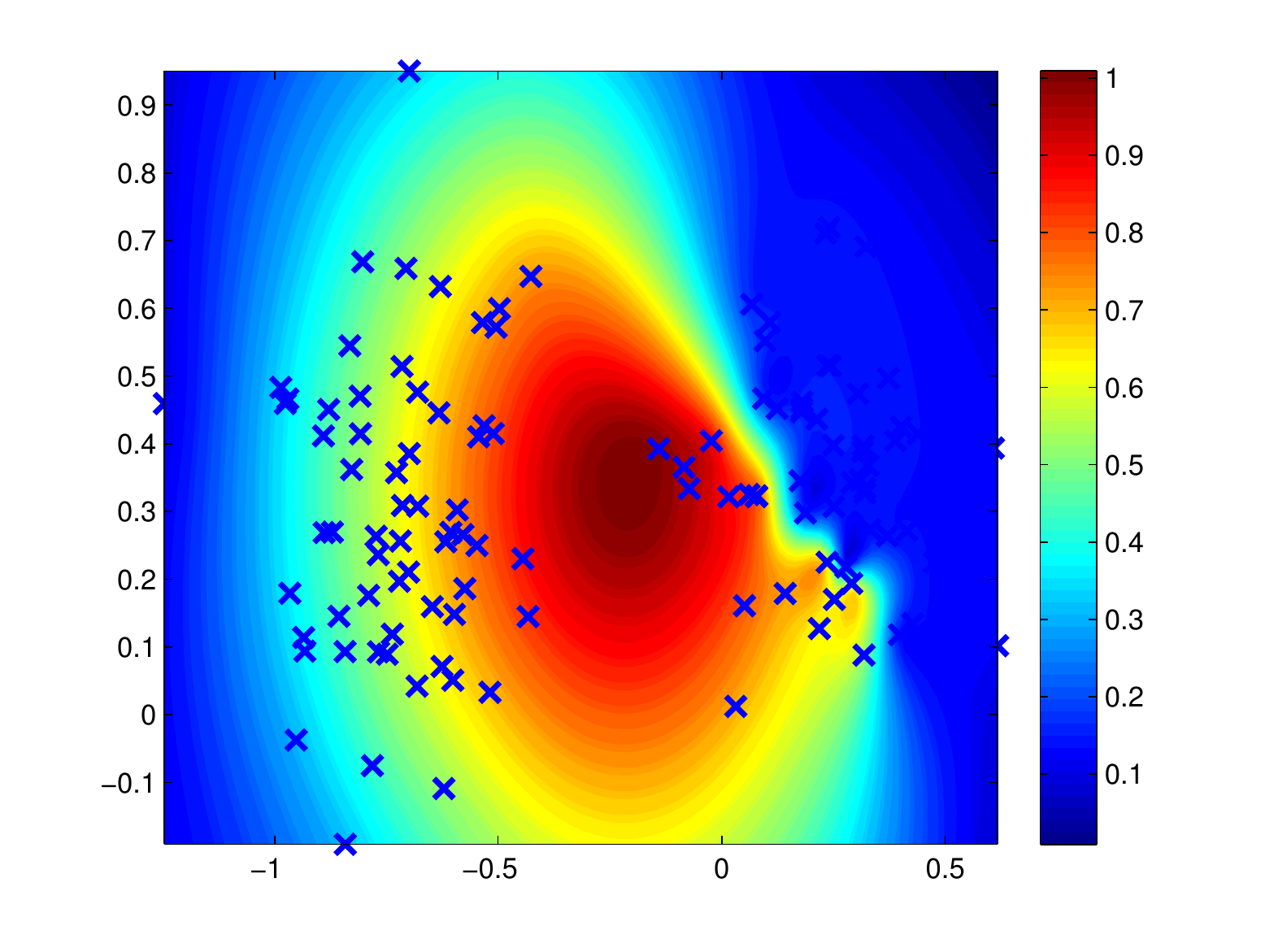}\label{fig:subfig:b}}
 \subfloat[``+1" instances]
 {\includegraphics[height= 0.17\textwidth]{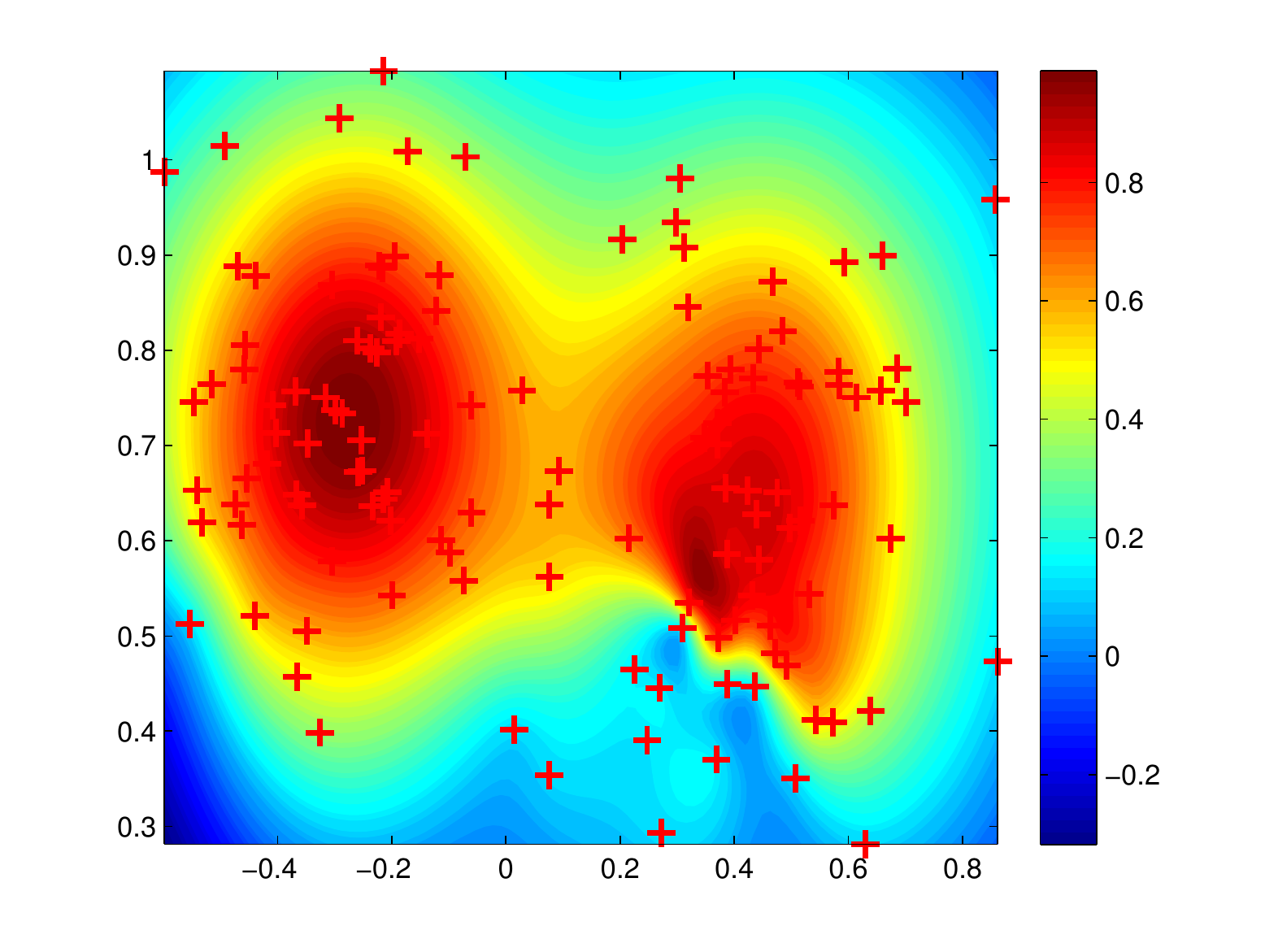} \label{fig:subfig:c}}
 \subfloat[``-1" instances]
 {\includegraphics[height= 0.17\textwidth]{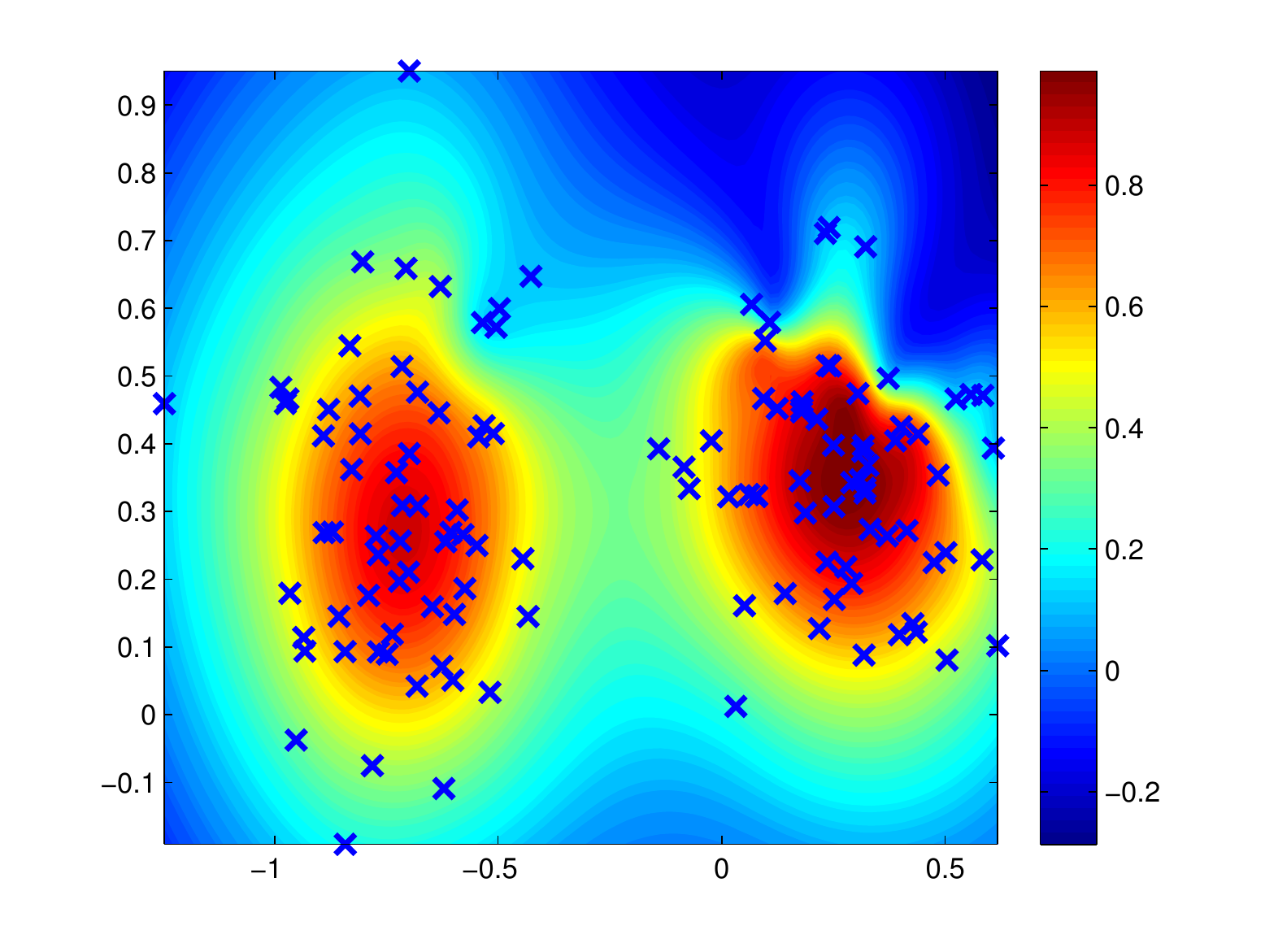} \label{fig:subfig:d}}
 \caption{Fuzzy membership distribution of training instances on Ripley dataset. (a) and (b) for linear case, (c) and (d) for nonlinear case.}\label{fig:fmd}
\end{figure*}
\begin{figure*}[t]
 \centering
    \vspace{-10pt}
    \subfloat[SVM]
    {\begin{tabular}{c}
 {\includegraphics[height= 0.17\textwidth]{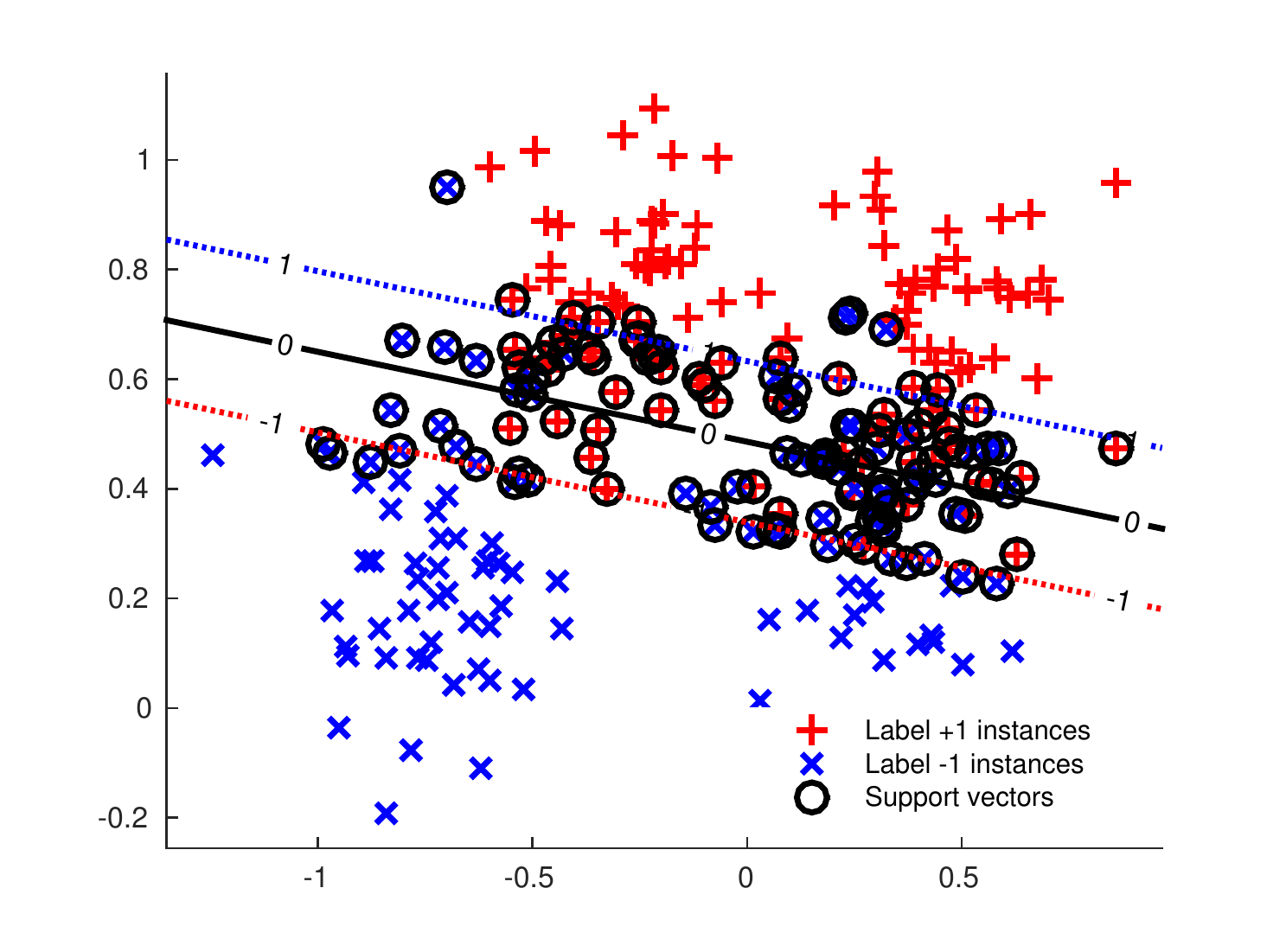}}\\
    {\includegraphics[height= 0.17\textwidth]{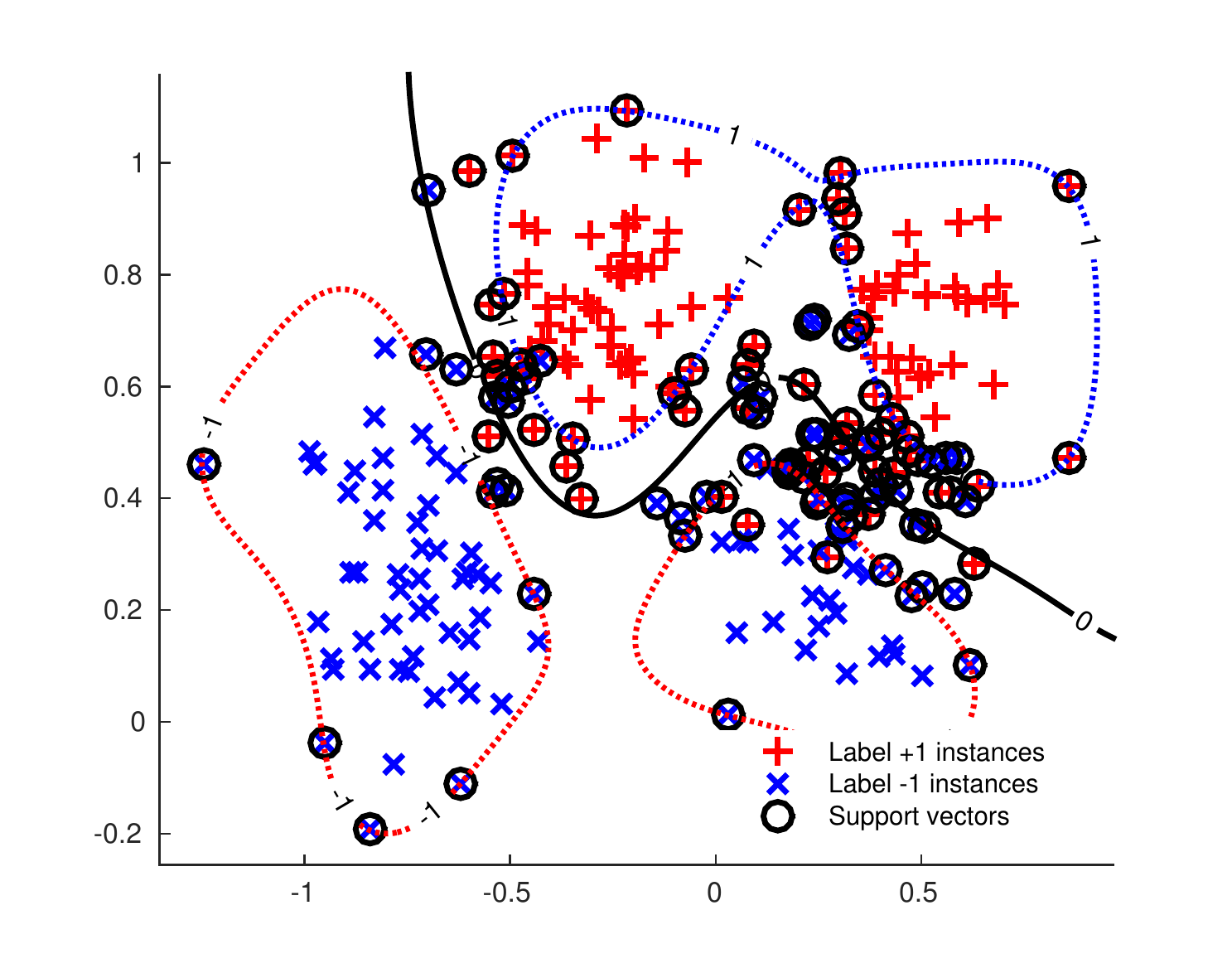}\label{fig:lra}}\end{tabular}}\hspace{-12pt}
    \subfloat[FSVM]
    {\begin{tabular}{c}
    {\includegraphics[height= 0.17\textwidth]{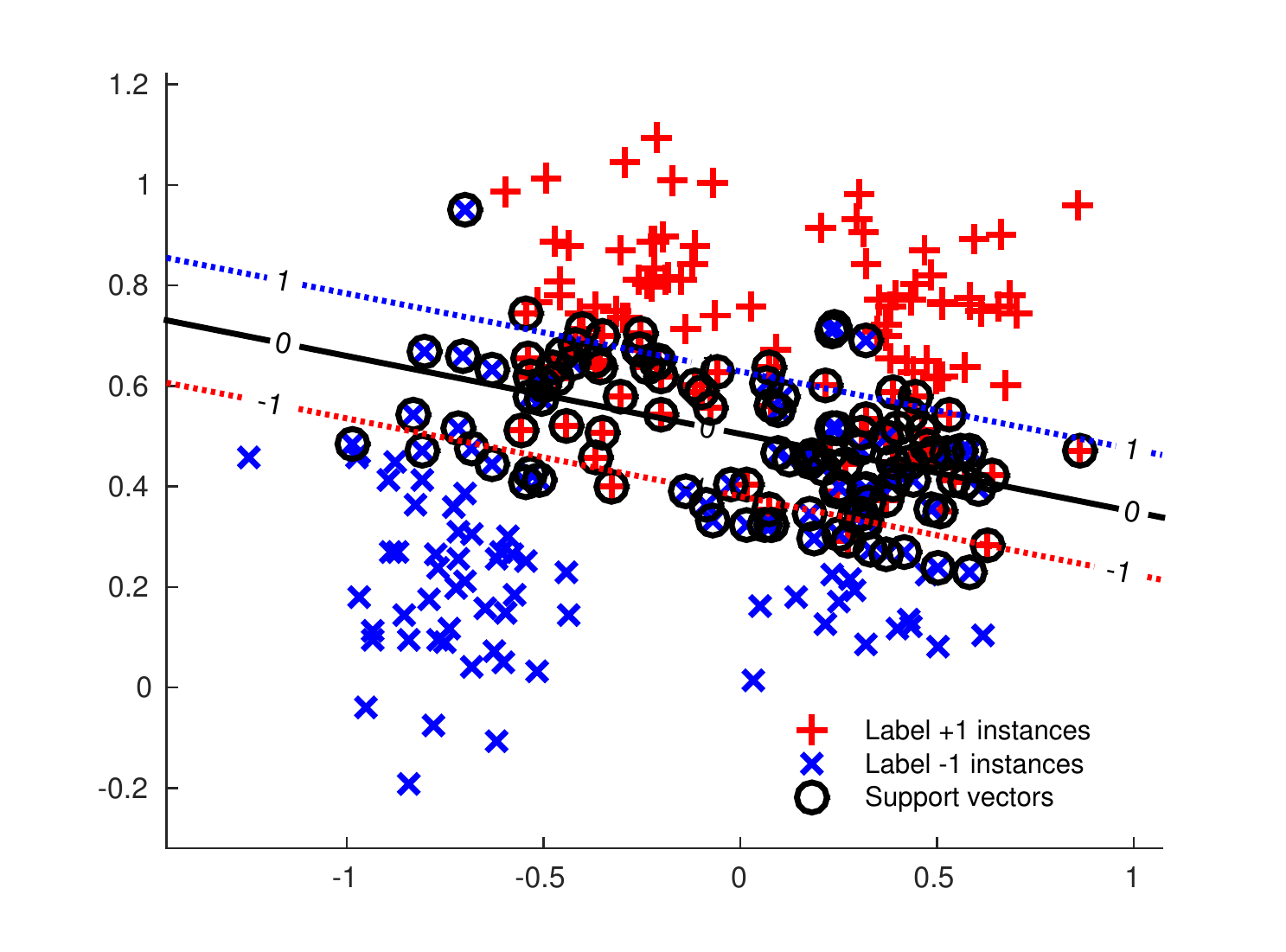}}\\
 {\includegraphics[height= 0.17\textwidth]{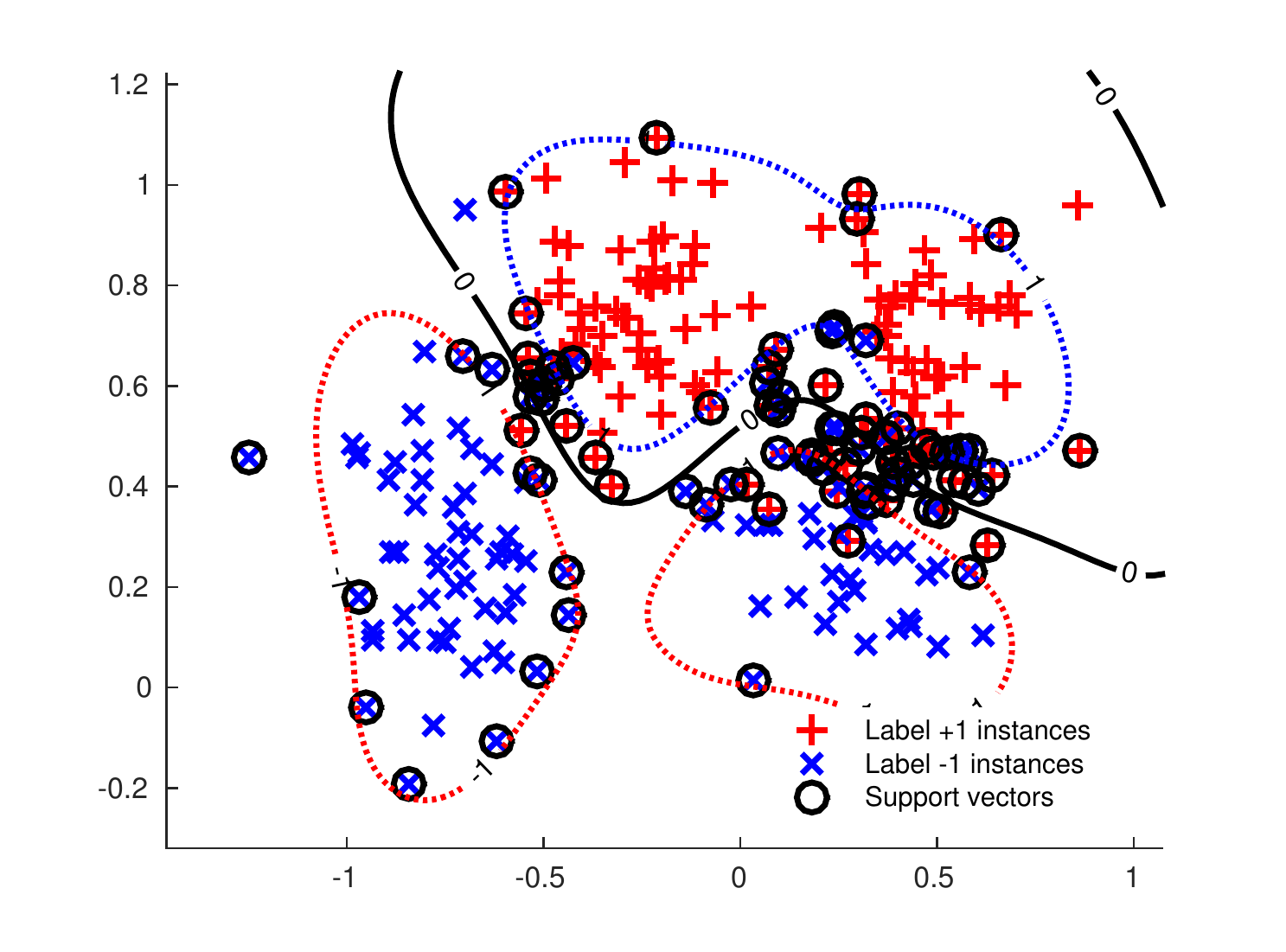}\label{fig:lrb}}\end{tabular}}\hspace{-12pt}
 \subfloat[TSVM]
    {\begin{tabular}{c}
    {\includegraphics[height= 0.17\textwidth]{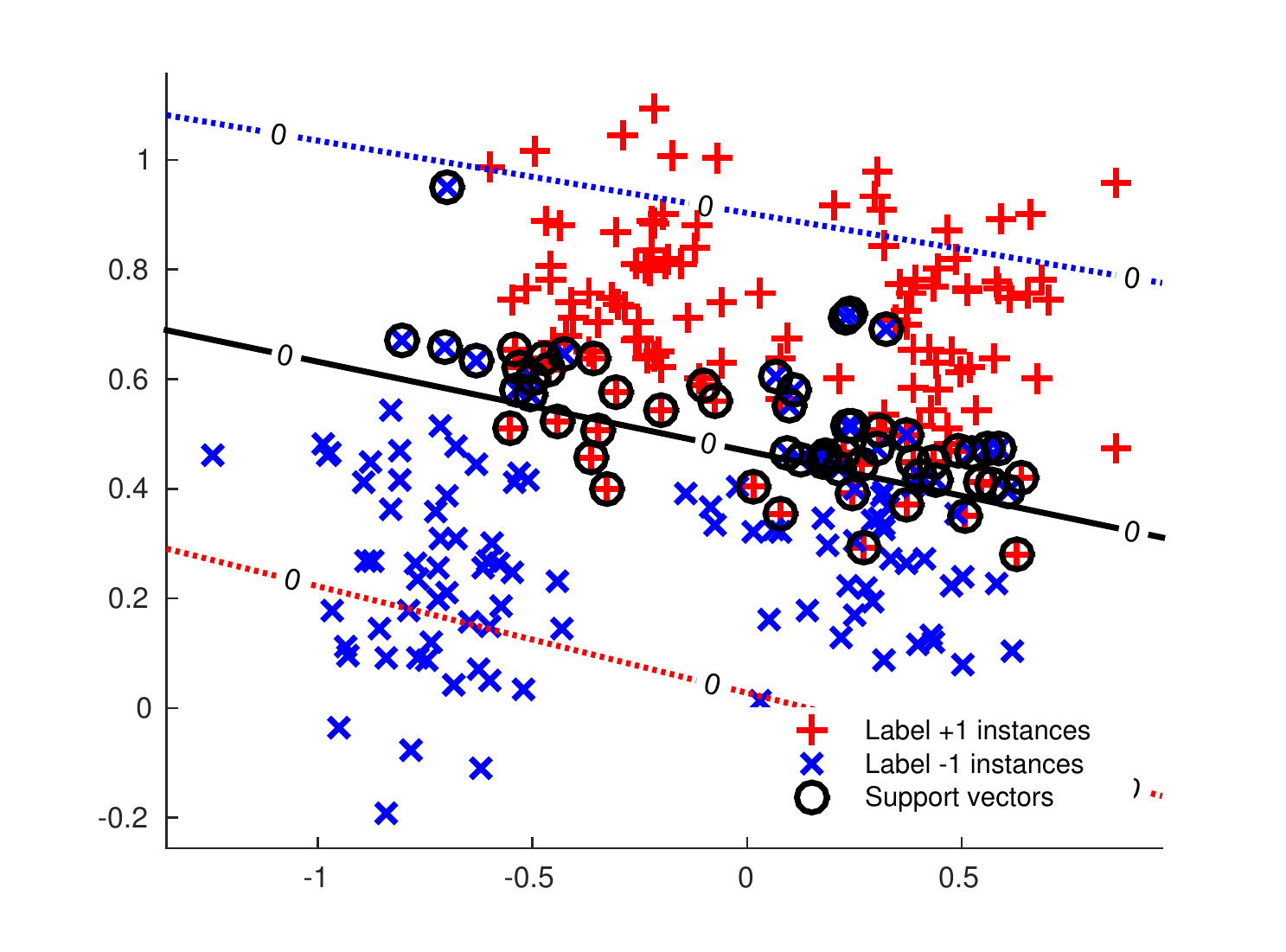}}\\
 {\includegraphics[height= 0.17\textwidth]{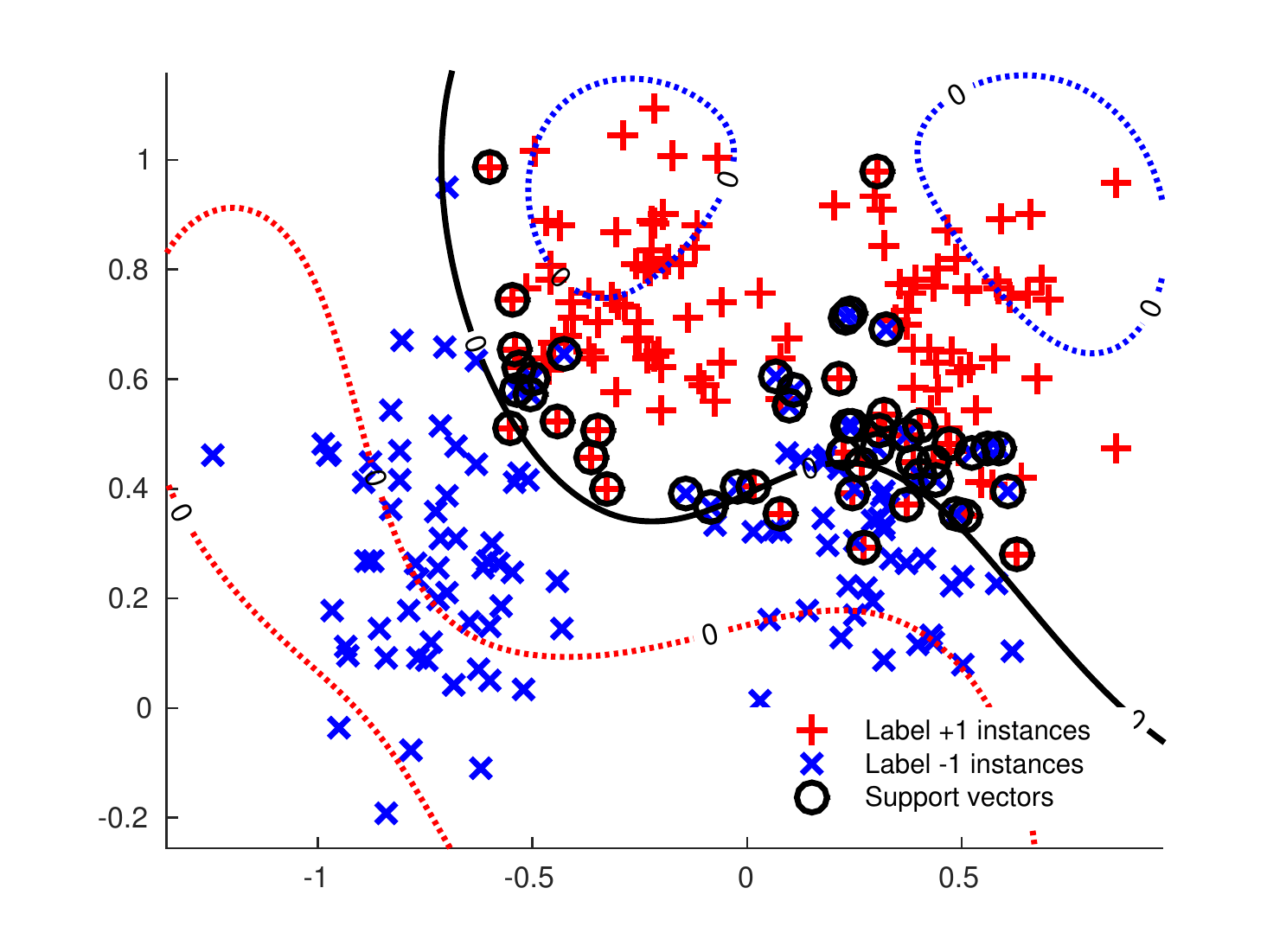}\label{fig:lrc}}\end{tabular}}\hspace{-12pt}
    \subfloat[FR-TSVM]
    {\begin{tabular}{c}
    {\includegraphics[height= 0.17\textwidth]{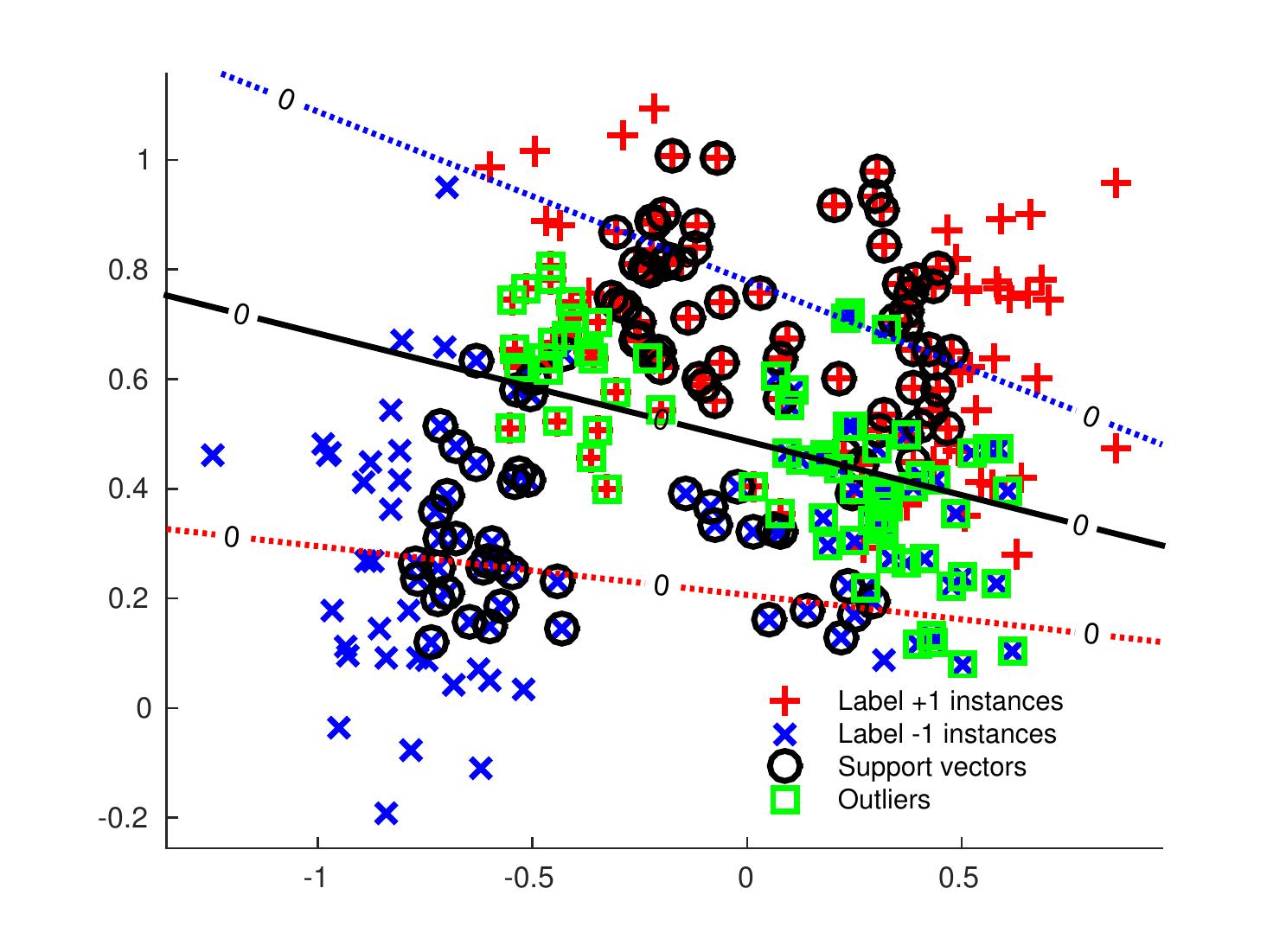}}\\
 {\includegraphics[height= 0.17\textwidth]{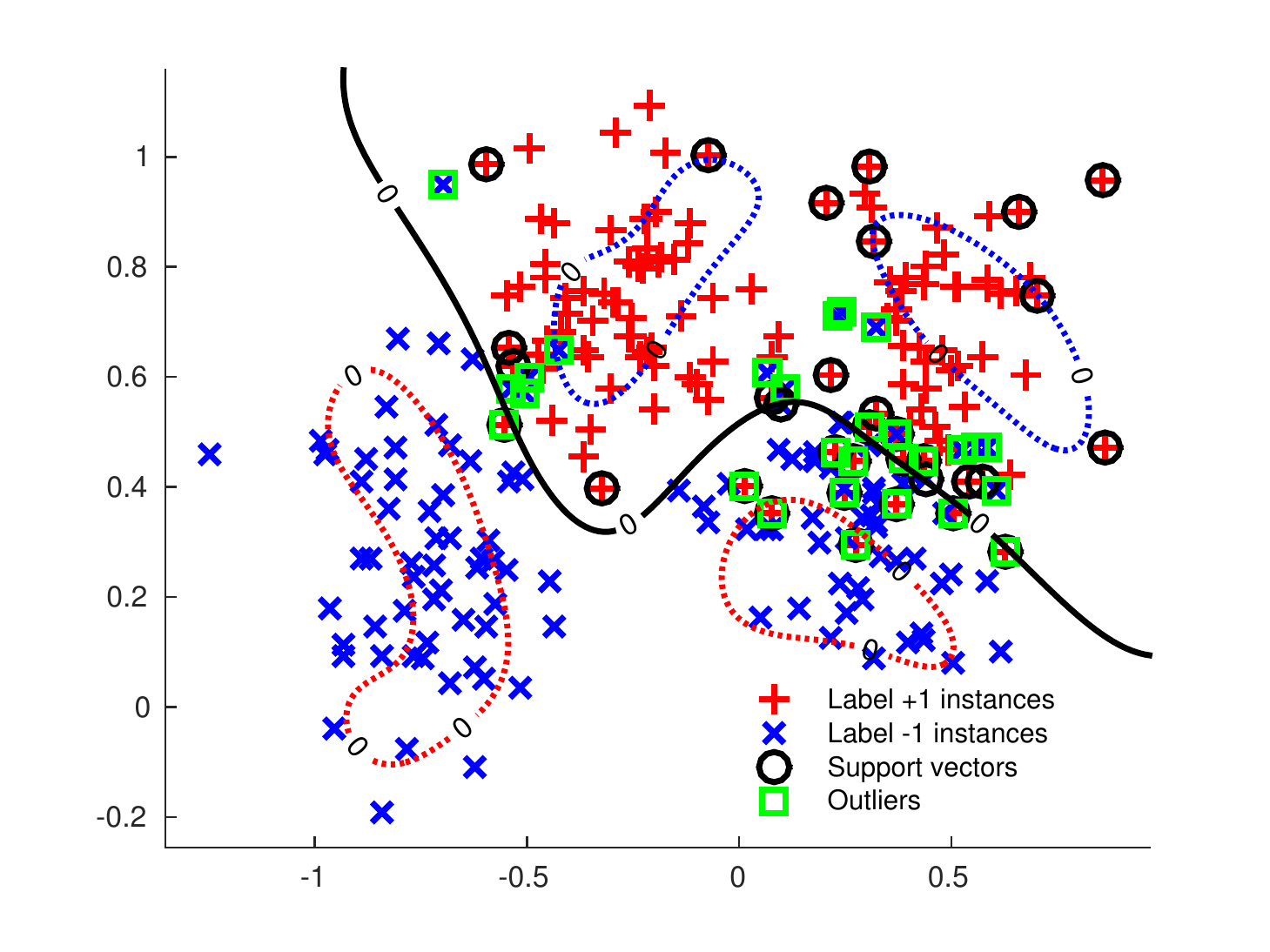}\label{fig:lrd}}\end{tabular}}
 \caption{Results of linear and nonlinear SVM, FSVM, TSVM and FR-TSVM on the first artificial dataset~(\emph{Ripleys synathetic}). The first and second  row show the results used linear kernel and Gaussian kernel, respectively. The blank straight line or curve represent the best decision hyperplane with corresponding method. \label{fig:ar-ri}} 
\end{figure*}

Table~\ref{table:linear_UCI} reports testing accuracies and learning CPU time of all methods with linear and Gaussian kernel. We can see that linear SVM and nonlinear FR-TSVM achieve the highest prediction accuracy under linear and nonlinear case, respectively. The reason for the outperformance of linear SVM is more likely from the
dataset itself than from the classifier because there may be no noise existing on the~\emph{Ripleys synathetic} dataset. The noiseless fact of this testing dataset suppresses the outstanding ability of FR-TSVM in the experiment. The ability of FR-TSVM is confirmed if we examine the accuracy in the nonlinear classification. In terms of execution time, FR-TSVM shows its excellence in computational efficiency for both linear and nonlinear learning, especially linear case. The excellence manifests the remarkable potential of employing FR-TSVM for fast classification.

\begin{table}[t]
\centering
\caption{Classification performance comparison on
artificial Ripleys dataset.}
\label{table:ar-ri}
\resizebox{0.48\textwidth}{!}{
%\begin{tabular}{|l|r|r|r|r|r|}
\begin{tabular}{@{\,}c@{\,}@{\,}l@{\,}@{\,}c@{\,}@{\,}c@{\,}@{\,}c@{\,}@{\,}c@{\,}}
\hline
Methods  &Description   &{~~SVM~~}   &{~FSVM~} &{~TSVM~~}&{FR-TSVM}\\
\hline
\multirow{2}{*}{Linear} &Acc($\%$)$\uparrow$  &$\textbf{89.70}$ &88.80 &89.20  &89.10\\
 &Time(s)$\downarrow$  &~1.46 &~2.00 &~0.28 &\textbf{~0.21}\\
\hline
\multirow{2}{*}{Nonlinear} &Acc($\%$)$\uparrow$  &90.40 &91.10 &90.50&\textbf{91.30}\\
&Time(s)$\downarrow$  &~1.56 &~1.79 &~0.60 &\textbf{~0.24}\\
\hline
\end{tabular}}
\end{table}

Fig.~\ref{fig:ar-ri} shows linear and nonlinear decision hyperplanes produced by comparative methods with equivalent settings. In these panels, while standard SVM and FSVM produce only a single hyperplane (Fig.~\ref{fig:lra} and~\ref{fig:lrb}), TSVM and FR-TSVM produce a pair of proximal hyperplanes (Fig.~\ref{fig:lrc} and \ref{fig:lrd}) for class separation. Instead of single decision hyperplane in standard SVM and FSVM, a pair of nonparallel decision hyperplanes are used in TSVM and FR-TSVM. Comparing FR-TSVM with TSVM, the final decision hyperplane of FR-TSVM is more exact than one of TSVM. The fact reveals that FR-TSVM is more capable of producing an unbiased prediction than TSVM.

The second dataset is also a 2-dimensional artificial-generated dataset. In this dataset, the positive class of instances are generated by a uniform distribution satisfying $x_1 \in [-\pi/2,2\pi]$, and $\sin(x_1)-0.25\leq x_2\leq \sin(x_1)+0.25$, while the negative class of instances consist of uniform points satisfying $x_1 \in [-\pi/2,2\pi]$, and $0.6\sin(x_1/1.05+0.5)-1.3\leq x_2\leq 0.6\sin(x_1/1.05+0.5)+0.8$, where $\vec x =[x_1,x_2]$. In our experiments, we generate 3000 instances according to the above rules. 600 instances are randomly chosen for training, and the remaining 2400 for testing. 

\begin{figure*}[t]
 \centering
    \vspace{-19pt}
    \subfloat[SVM]{
    \begin{tabular}{c}
 {\includegraphics[height= 0.17\textwidth]{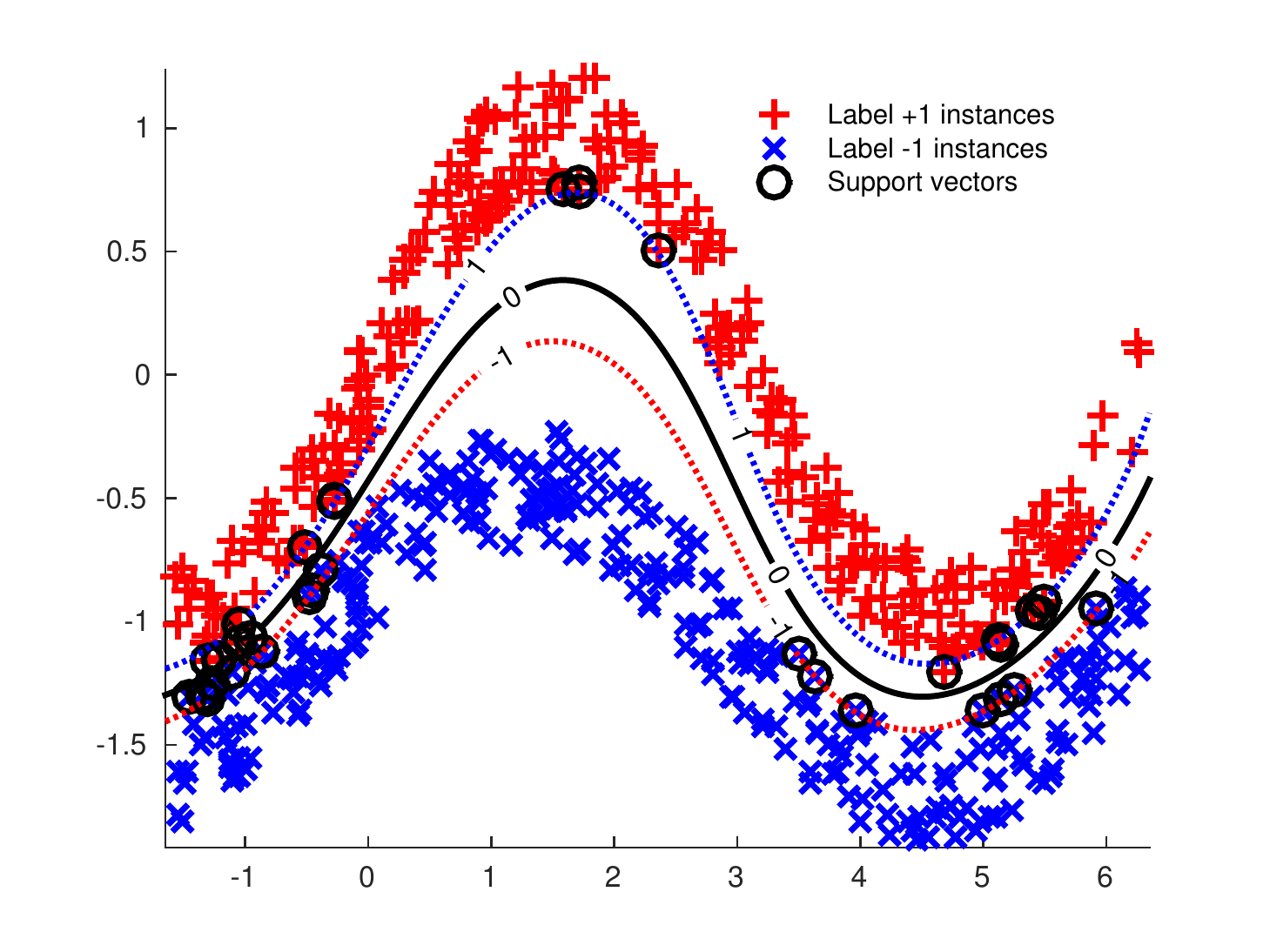}}\\
    {\includegraphics[height= 0.17\textwidth]{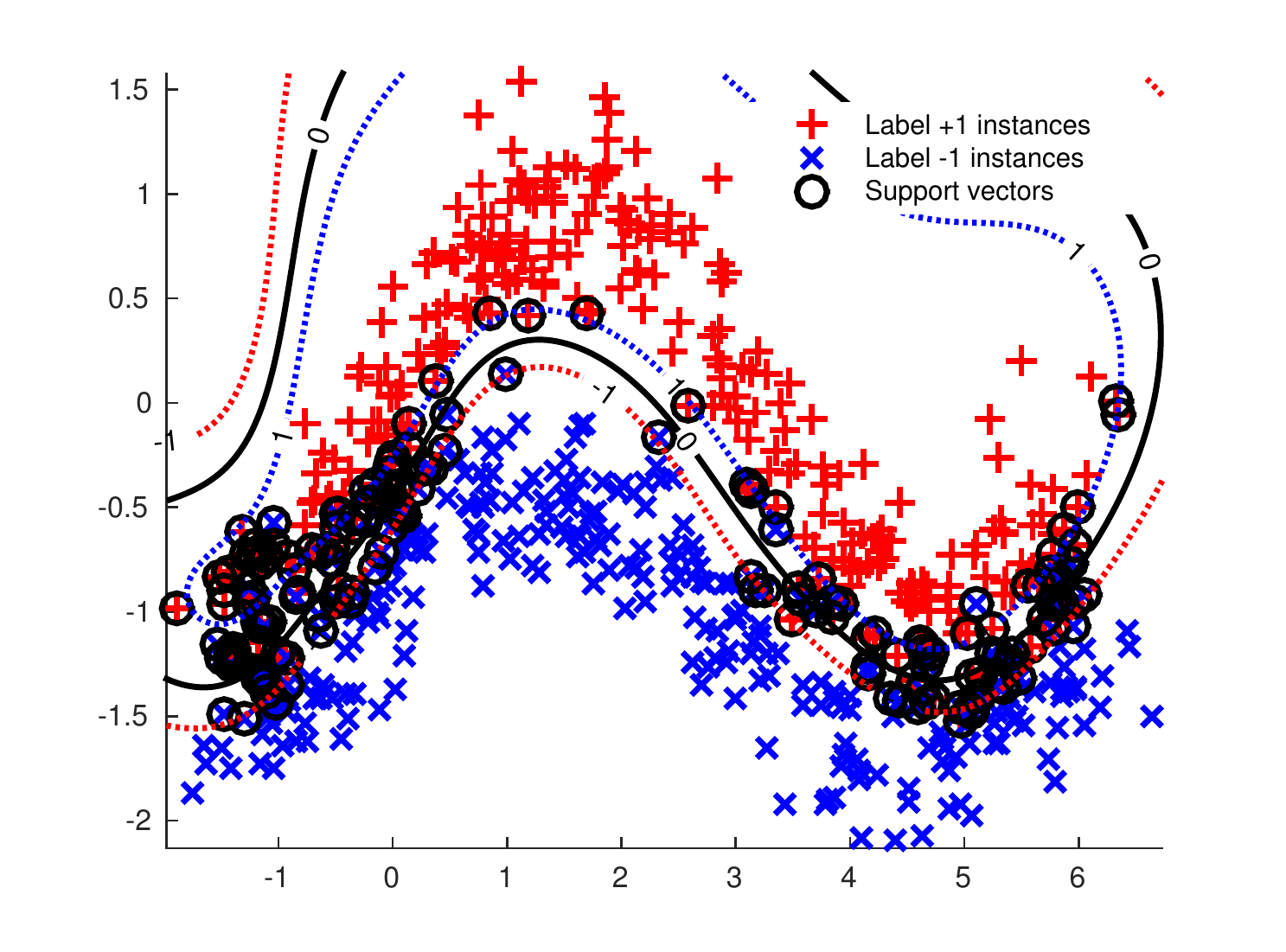}}\end{tabular}}\hspace{-12pt}
    \subfloat[FSVM]
    {\begin{tabular}{c}
    {\includegraphics[height= 0.17\textwidth]{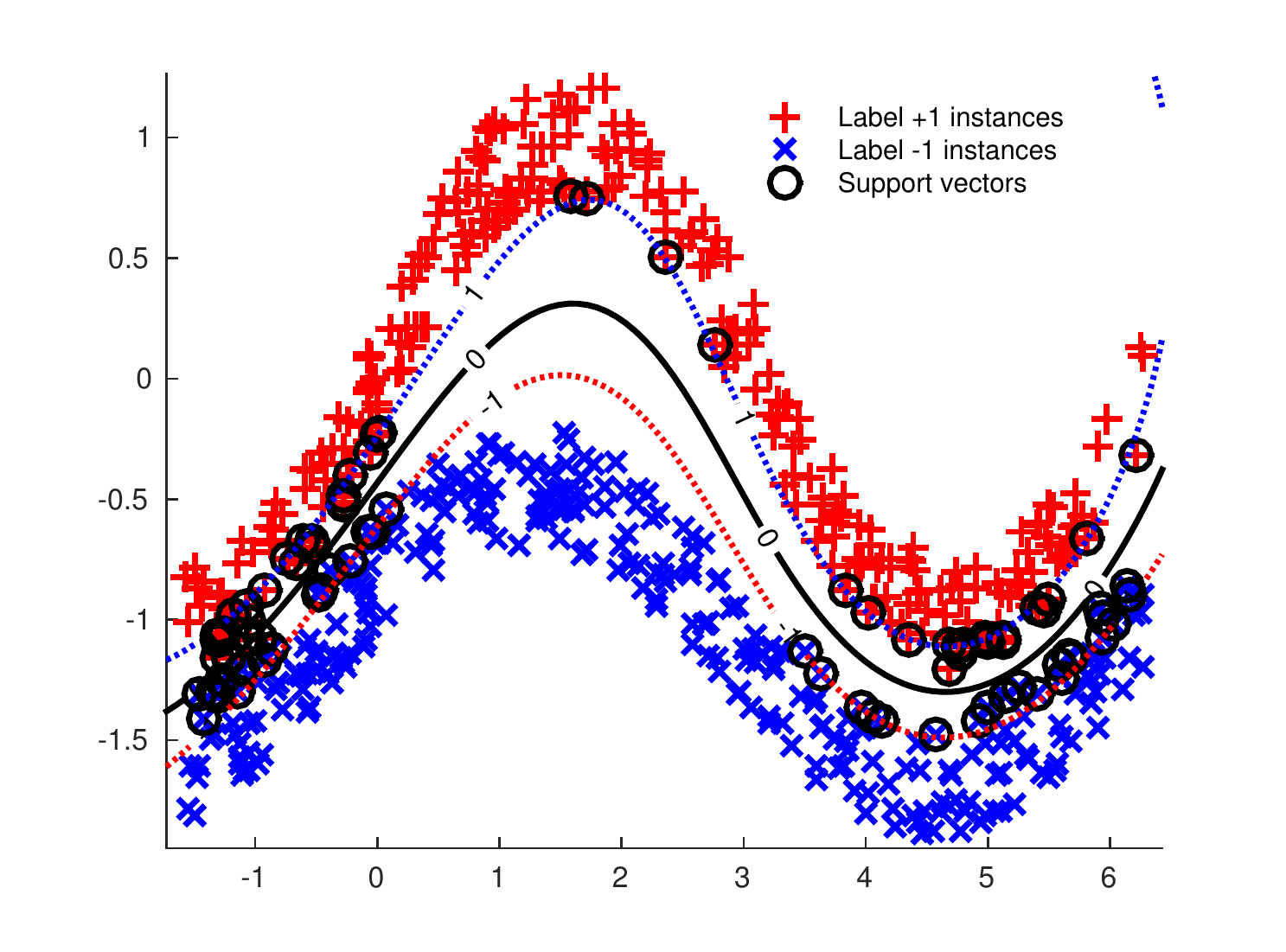}}\\
 {\includegraphics[height= 0.17\textwidth]{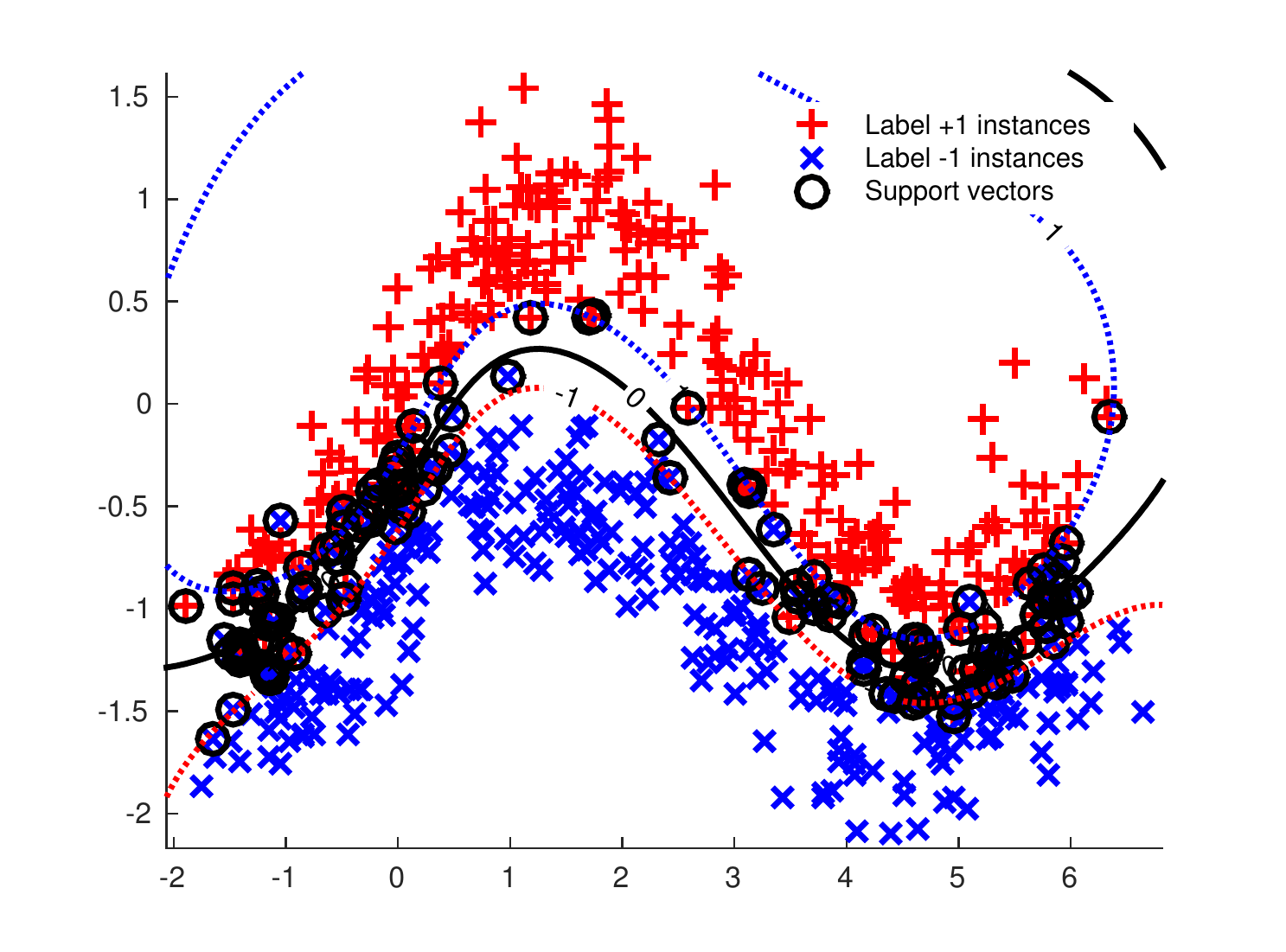}}\end{tabular}}\hspace{-12pt}
 \subfloat[TSVM]
    {\begin{tabular}{c}
    {\includegraphics[height= 0.17\textwidth]{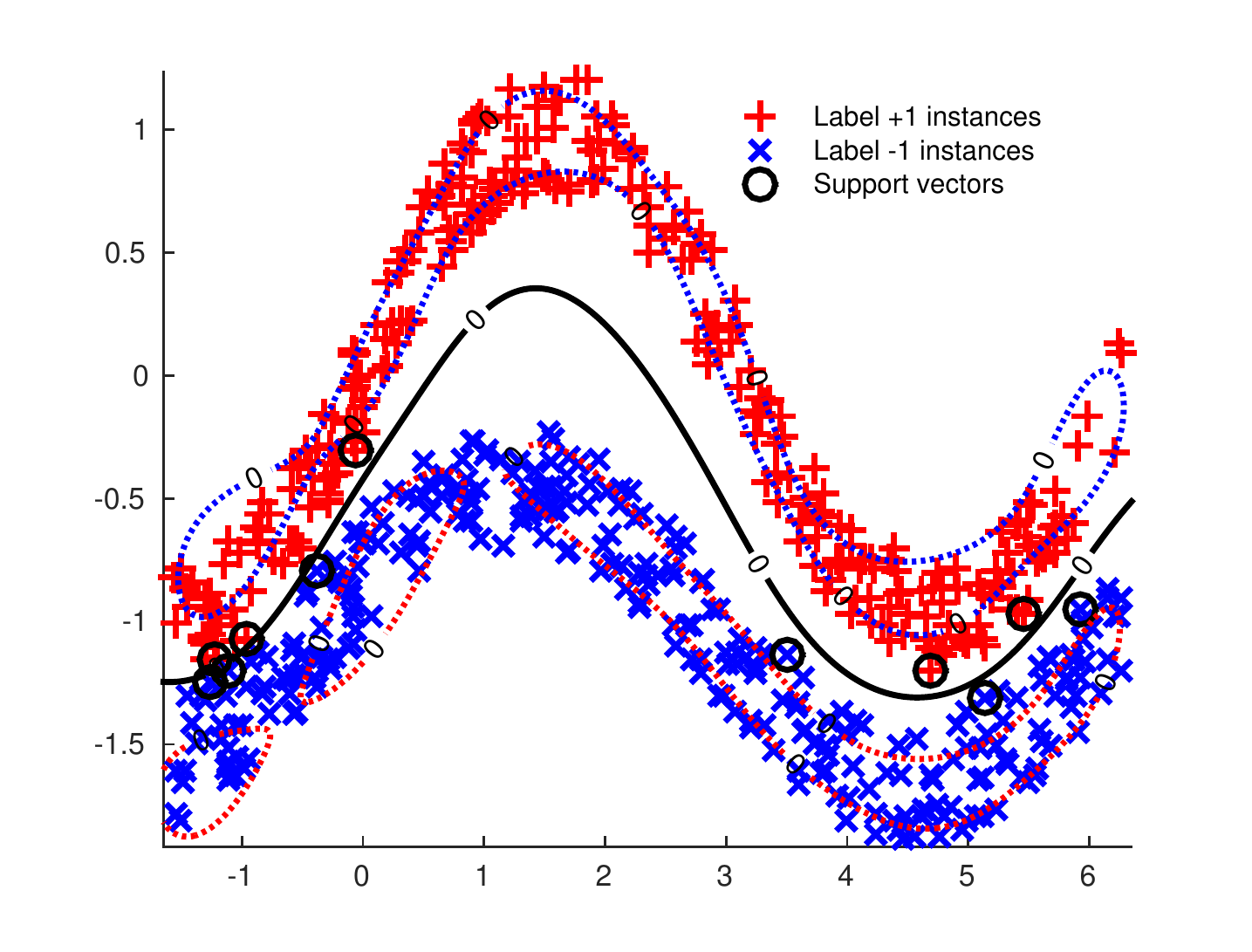}}\\
 {\includegraphics[height= 0.17\textwidth]{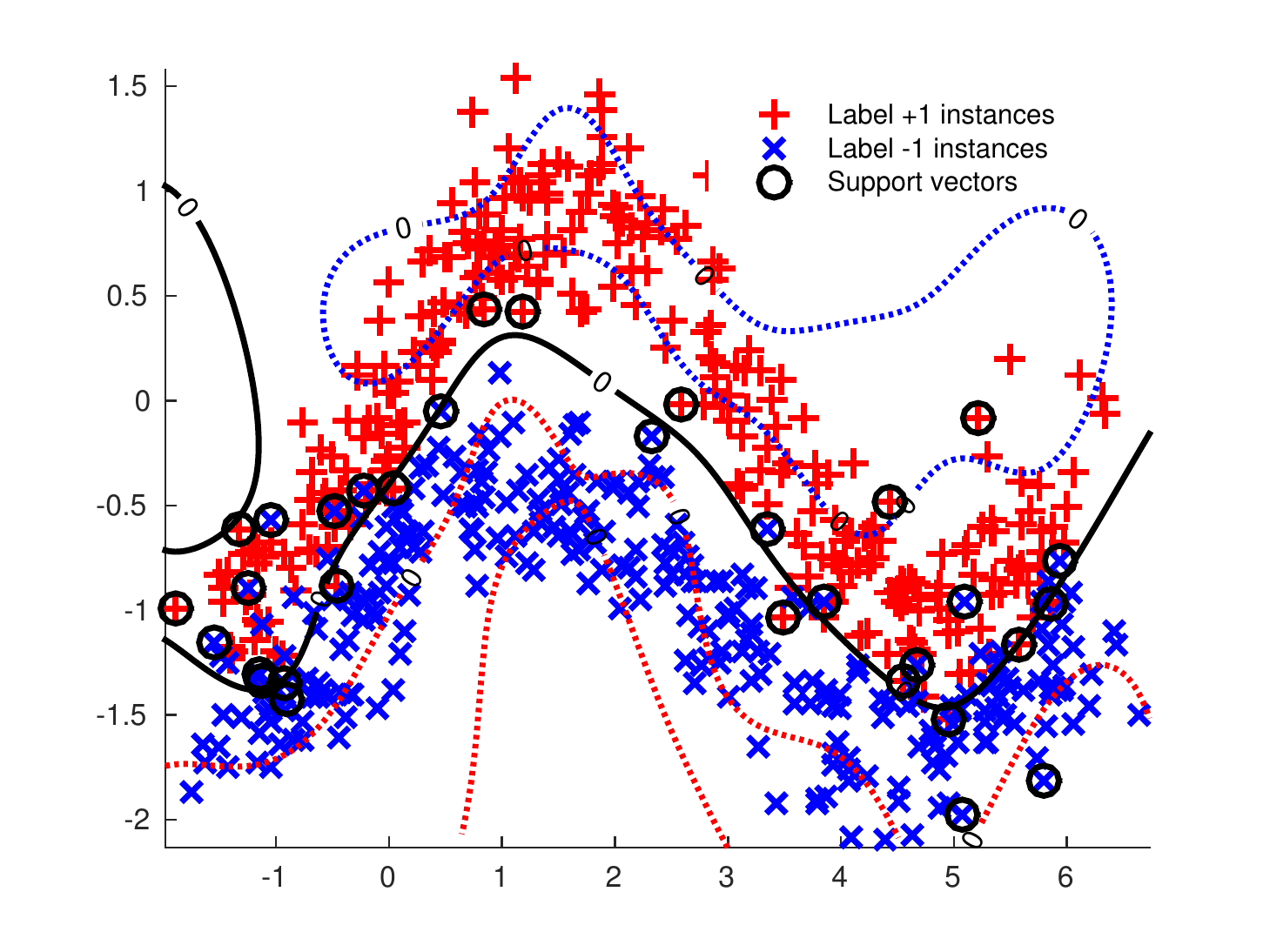}}\end{tabular}}\hspace{-12pt}
    \subfloat[FR-TSVM]
    {\begin{tabular}{c}
    {\includegraphics[height= 0.17\textwidth]{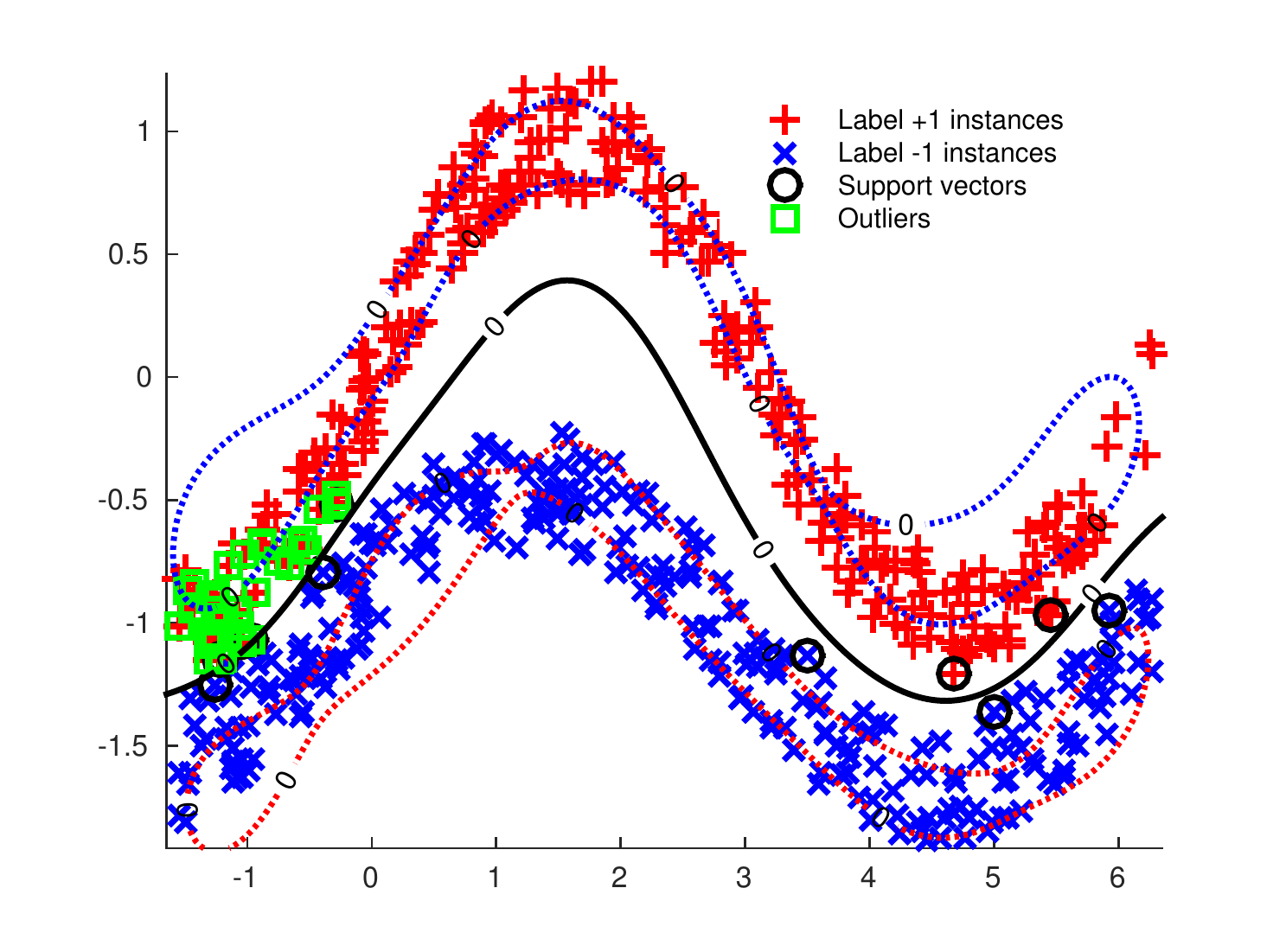}}\\
 {\includegraphics[height= 0.17\textwidth]{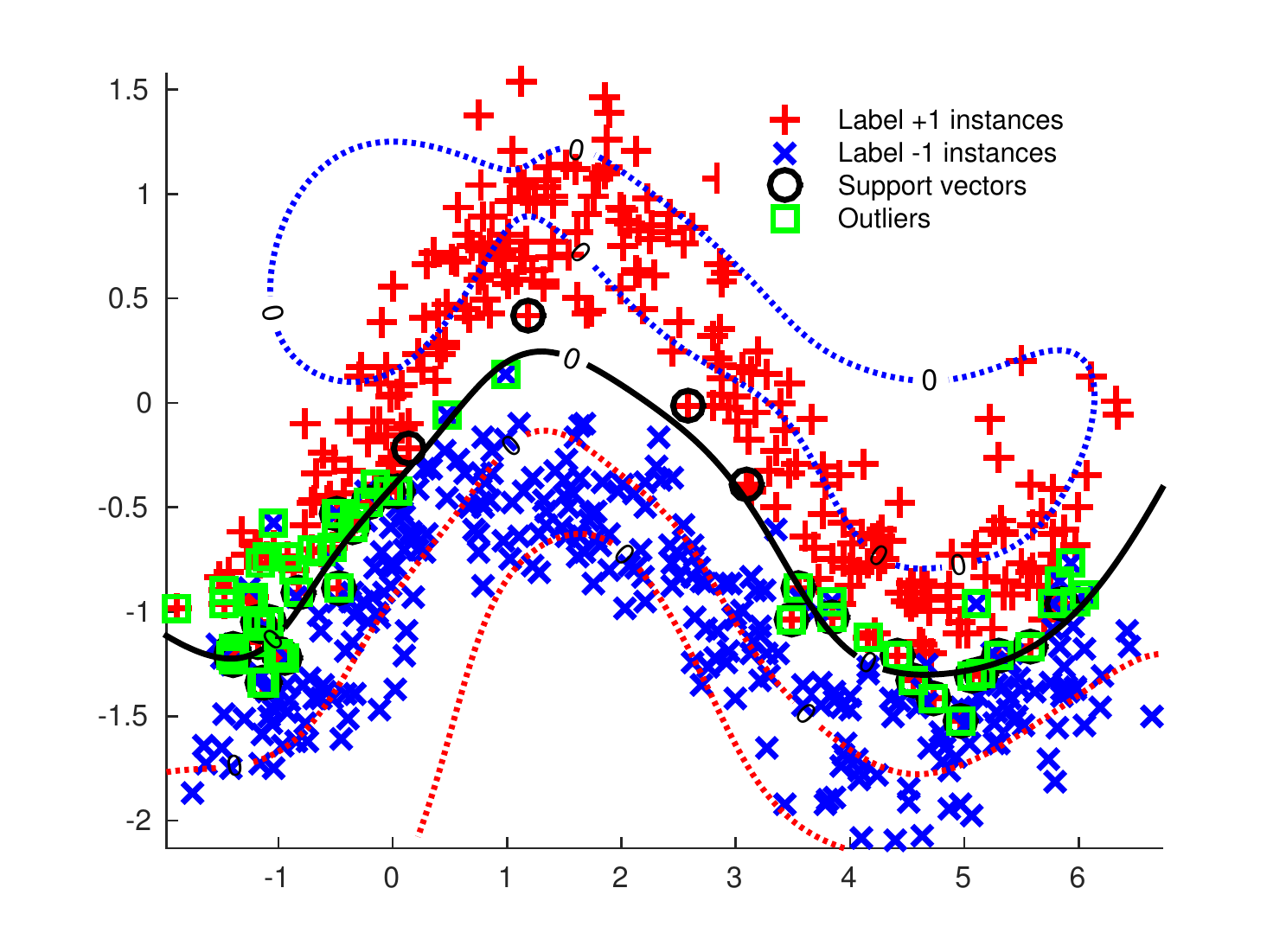}}\end{tabular}}
 \caption{Results of nonlinear SVM, FSVM, TSVM and FR-TSVM on the second artificial dataset. Comparison of decision hyperplane with/without noises. \label{fig:ar-sin}}
\end{figure*}

\begin{table}[t]
\centering
\caption{{Classification performance comparison on
the second artificial dataset.}}
\label{table:ar-sin}
\resizebox{0.48\textwidth}{!}{
%\begin{tabular}{|l|r|r|r|r|r|}
\begin{tabular}{@{\,}c@{\,}@{\,}l@{\,}@{\,}c@{\,}@{\,}c@{\,}@{\,}c@{\,}@{\,}c@{\,}}
\hline
Methods  &Description   &{~~SVM~~}   &{~FSVM~} &{~TSVM~~}&{FR-TSVM}\\
\hline
\multirow{2}{*}{nonoise} &Acc($\%$)$\uparrow$  &99.92  &99.88 &99.92 &\textbf{100.0}\\
  &Time(s)$\downarrow$ &16.77 &17.11 &~5.03 &\textbf{~0.30}\\
\hline
\multirow{2}{*}{noise}     &Acc($\%$)$\uparrow$           &99.37 &\textbf{99.63}        &97.00         &99.33~\\
&Time(s)$\downarrow$              &16.43 &16.74 &~4.93 &\textbf{~0.72}\\
\hline
\end{tabular}}
\end{table}

In Fig.~\ref{fig:ar-sin}, the first row shows the results of the nonlinear SVM, FSVM, TSVM and the proposed FR-TSVM on this dataset. We can see that the separating hyperplanes of these classifiers obtain similar results. To further validate the robust performance of the FR-TSVM, we add a noise $\varepsilon \in \mathcal{N}(0,0.2)$ for each training instance of this dataset, where $\mathcal{N}(0,0.2)$ is a normal distribution with mean 0 and variance 0.2. In Fig.~\ref{fig:ar-sin}, the second row shows the result of the nonlinear SVM, FSVM, TSVM and the proposed FR-TSVM on this dataset with noise. It can be seen that the proposed FR-TSVM can still effectively determine the separating hyperplane compared with the other three methods, especially TSVM. This indicates that the FR-TSVM can effectively alleviate the effect of noise points. 

Table~\ref{table:ar-sin} reports prediction results including classification accuracy and training time with/without noises. We can see that all the methods obtain similarity test accuracies without noises, while our FR-TSVM achieves the highest test accuracy. Under noise case, FR-TSVM and FSVM obtain better classification accuracies than TSVM and SVM, respectively. This indicates that fuzzy membership takes an important role in training model with noisy instances.  It is noteworthy that TSVM's test accuracy is reduced by almost 3\% (from 99.88\% to 97.00\%), whereas FR-TSVM has only a slight slowdown~(less than 1\%). It suggests that FR-TSVM  is more robust than TSVM when training instances consist of noisy data. As for the training time, FR-TSVM and TSVM are more efficient than SVM and FSVM, while FR-TSVM is the fastest among all methods. 

\subsection{Experiments on benchmark datasets}
To further examine the performance of FR-TSVM, 13 common datasets are gathered from the public UCI machine learning datasets~\cite{UCI2011}. To adjust values measured on different scales to a notionally common scale in the input space, all input features are normalized and scaled-down within $[0,1]$. The examination mainly focuses on binary classification. Two adversary classes are formed in every dataset. For the multi-classification datasets, we convert them into binary classification datasets by taking the majority class as the first class and gathering all the remainders together as the adversary class. Table~\ref{table:UCI} lists the statistics of datasets.

Table~\ref{table:linear_UCI} and~\ref{table:rbf_UCI}  report the learning results of these algorithms with linear and Gaussian kernels. To show the optimization cost, a particular quadratic programming time is recorded. To assess the performance, 10-folds cross-validation, as that in Section~\ref{sec:exp}, is taken. It means every classifier is repeatedly validated in the datasets with a ratio of 90\%/10\% for respective training/testing phase. Ten characteristics values are collected and averaged for the assessment. In the experiments, a standard deviation of the ten classification accuracies is provided in addition to the average to reflect the classification robustness.

\begin{table}[t]
\centering
\caption{Detailed characteristics of the benchmark datasets. $l$, $l_+$ and $l_-$ are the number of all instances, positive instances and negative instances, respectively. $n$ is the dimension of feature.\label{table:UCI}}
\resizebox{0.30\textwidth}{!}{
\begin{tabular}{lrrrr}
\hline
\multirow{2}{*}{Dataset} &\multicolumn{4}{c}{Data statistics}\\
 &$\#l$ &$\#l_+$ &$\#l_-$ &$\#n$\\
\hline
Breast     &106   &22  &84  &9\\
Ionosphere &351   &126 &225 &34\\
Iris       &150   &100 &50  &4\\
Australian &690   &307 &383 &14\\
WDBC       &569   &357 &212 &30\\
Wine       &178   &119 &59  &13\\
Hepatitis  &155   &123 &32  &19\\
WPBC       &198   &46  &148 &33\\
Bupa       &345   &200 &145 &6 \\
Sonar      &208   &111 &97  &60\\
Glass      &214   &138 &76  &10\\
Heart      &270   &120 &150 &13\\
Pima       &768   &268 &500 &8 \\
\hline
\end{tabular}}
\end{table}

\begin{table}[t]
\caption{Comparison results~(mean$\pm$std) of linear SVM, FSVM, TSVM and FR-TSVM on benchmark datasets. The best result of each row is marked in bold.}\label{table:linear_UCI}
\centering
\resizebox{0.45\textwidth}{!}{
\begin{tabular}{@{\,}l@{\,}@{\,}l@{\,}@{\,}l@{\,}@{\,}l@{\,}@{\,}l@{\,}}
%\begin{tabular}{|c||l|l|l|l|}
\firsthline
{Dataset}    &{SVM}   &{FSVM} &{TSVM}&{FR-TSVM}\\
\hline
 Breast       &96.33$\pm$4.77          &97.17$\pm$4.58          &97.17$\pm$4.58          &97.17$\pm$4.58\\
 Ionosphere   &83.53$\pm${6.48}        &85.75$\pm$4.06          &82.33$\pm$5.18          &\fontb{87.19$\pm$4.22}\\
 Iris         &100.0$\pm${0.00}        &100.0$\pm${0.00}        &100.0$\pm${0.00}        &100.0$\pm$0.00\\
 Australian   &84.92$\pm$4.53          &85.50$\pm$4.59          &85.07$\pm$4.77          &\fontb{85.93$\pm$4.39}\\
 WDBC         &95.34$\pm${5.17}        &95.87$\pm$3.21          &93.84$\pm$5.86          &\fontb{96.39$\pm$3.52}\\
 Wine       &{98.89}$\pm${2.34}      &{98.89}$\pm$2.34        &98.33$\pm$2.68          &98.89$\pm$2.34\\
 Hepatitis    &79.26$\pm$8.76          &84.94$\pm$7.29          &75.58$\pm$12.50         &\fontb{85.77$\pm$7.21}\\
 WPBC         &\fontb{79.93$\pm$9.49} &74.24$\pm$10.35         &76.88$\pm$7.01           &77.96$\pm${10.03}\\
 Bupa         &66.36$\pm$6.04          &\fontb{67.51$\pm$7.36} &61.72$\pm$5.96           &64.38$\pm${6.24}\\
 Sonar        &74.08$\pm$8.96          &77.46$\pm$7.14          &72.15$\pm$7.48          &\fontb{78.34$\pm$8.29}\\
 Glass        &70.41$\pm$10.01         &74.18$\pm$11.01         &67.64$\pm${11.02}       &\fontb{81.17$\pm$13.56}\\
 Heart        &82.22$\pm$5.18          &82.59$\pm${3.51}        &84.07$\pm$4.95          &84.07$\pm$6.06\\
 Pima         &\fontb{77.21$\pm$3.75}&75.65$\pm$4.22            &76.95$\pm$3.37          &75.13$\pm${3.78}\\
\lasthline
\end{tabular}}
\end{table}

\begin{table}[t]
\caption{Comparison results~(mean$\pm$std) of nonlinear SVM, FSVM, TSVM and FR-TSVM on benchmark datasets. The best result of each row is marked in bold.
\label{table:rbf_UCI}}
\centering
\resizebox{0.45\textwidth}{!}{
\begin{tabular}{@{\,}l@{\,}@{\,}l@{\,}@{\,}l@{\,}@{\,}l@{\,}@{\,}l@{\,}}
%\begin{tabular}{|c||l|l|l|l|}
\hline
{Dataset}    &{SVM}   &{FSVM } &{TSVM }&{FR-TSVM}\\
\hline
 Breast        &97.26$\pm$4.43~      &97.17$\pm$4.58~      &96.33$\pm$4.77~    &\fontb{98.09$\pm$4.03}~\\
 Ionosphere    &94.84$\pm${4.01}     &94.59$\pm$4.31      &92.61$\pm$6.12    &\fontb{95.41$\pm$4.93}\\
 Iris          &100.0$\pm${0.00}     &100.0$\pm${0.00}    &100.0$\pm${0.00}  &100.0$\pm$0.00\\
 Australian    &85.50$\pm$4.53       &85.50$\pm$4.59      &85.50$\pm$4.59    &\fontb{86.81$\pm$4.84}\\
 WDBC          &94.84$\pm${4.23}     &95.34$\pm$3.80      &95.34$\pm$3.80    &\fontb{96.39$\pm$3.52}\\
 Wine          &99.44$\pm${1.76}     &98.89$\pm$2.34      &{100.0}$\pm$0.00  &{100.0}$\pm${0.00}\\
 Hepatitis     &80.51$\pm$8.32       &84.28$\pm$7.42      &82.00$\pm$6.84    &\fontb{84.48$\pm${6.86}}\\
 WPBC          &81.51$\pm$7.13       &77.88$\pm$9.43      &75.30$\pm$7.93    &\fontb{82.51$\pm${8.05}}\\
 Bupa          &70.68$\pm$8.28       &\fontb{72.71$\pm$7.93} &71.86$\pm$5.71&71.84$\pm$5.67\\
 Sonar         &89.42$\pm$5.41       &88.92$\pm$6.95      &89.42$\pm$5.41    &\fontb{89.44$\pm${5.31}}\\
 Glass         &\fontb{97.68$\pm$3.95}&96.77$\pm$4.38    &94.87$\pm$5.08    &97.21$\pm$3.93\\
 Heart         &84.07$\pm$5.25       &82.59$\pm$4.29      &80.74$\pm$7.16    &\fontb{84.81$\pm$4.08}\\
 Pima          &75.65$\pm$3.80       &75.26$\pm$2.91      &\fontb{77.34$\pm$5.16}   &76.17$\pm${2.68}\\
\hline
\end{tabular}}
\end{table}

Comparing to baselines including SVM, FSVM and TSVM, the FR-TSVM produces more competitive and robust performances~(\emph{e.g}, high mean and low deviation of accuracies) on most datasets in Table~\ref{table:linear_UCI} and~\ref{table:rbf_UCI}. 
This suggests it is important to improve the classification performance introducing the fuzzy membership. Moreover, the tables show a higher improvement using nonlinear kernel than that of a linear kernel. This is due to the intrinsic difference between the kernels. As we know, a nonlinear Gaussian kernel commonly pertains a superior classification ability due to its flexibility. It meets the general concept of the classification. In addition, FR-TSVM obtains slightly lower performance than other methods on a few datasets~(\emph{e.g}, \emph{Bupa} and \emph{Pima}). A possible reason is that these datasets may not have outliers. Fig.~\ref{fig:UCI_T} shows the time cost of training on 13 UCI datasets. It is seen that the training time of TSVM and FR-TSVM are shorter than those of SVM and FSVM. This result is not surprising because TSVM and FR-TSVM are solved by two small QPPs rather than one large QPP in SVM and FSVM. Compared to the training time of TSVM, our FR-TSVM are faster. In short, the results in Table~\ref{table:linear_UCI} and~\ref{table:rbf_UCI} indicate, FR-TSVM effectively improves the classification accuracy and efficiently reduces training time compared to the traditional methods.  The excellence strongly reflects that the classifier is potential for future applications.

\begin{figure*}[t]
 \centering
    \subfloat[]{
    \begin{tabular}{c}
 {\includegraphics[height= 0.09\textwidth]{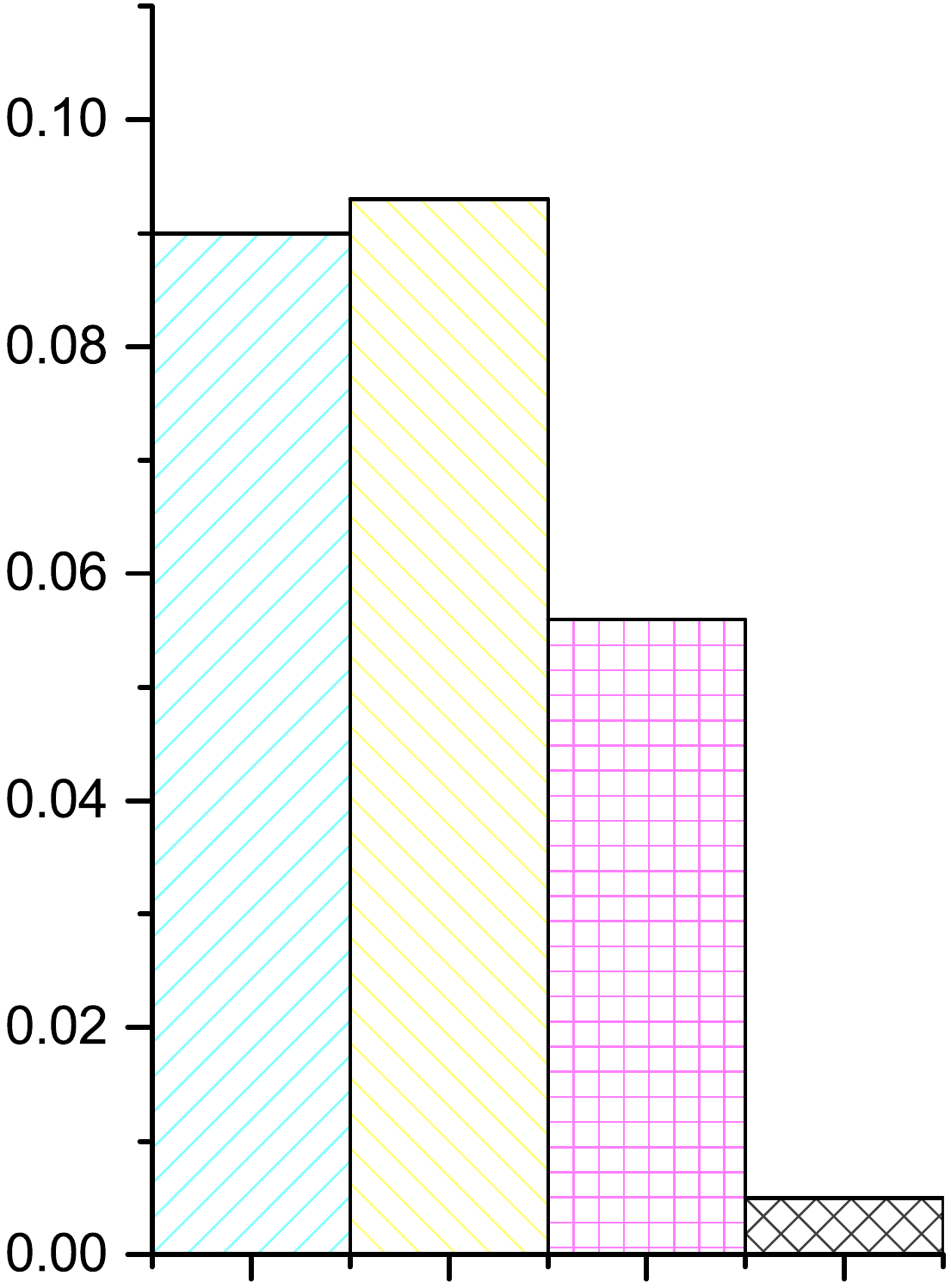}}\\
    {\includegraphics[height= 0.09\textwidth]{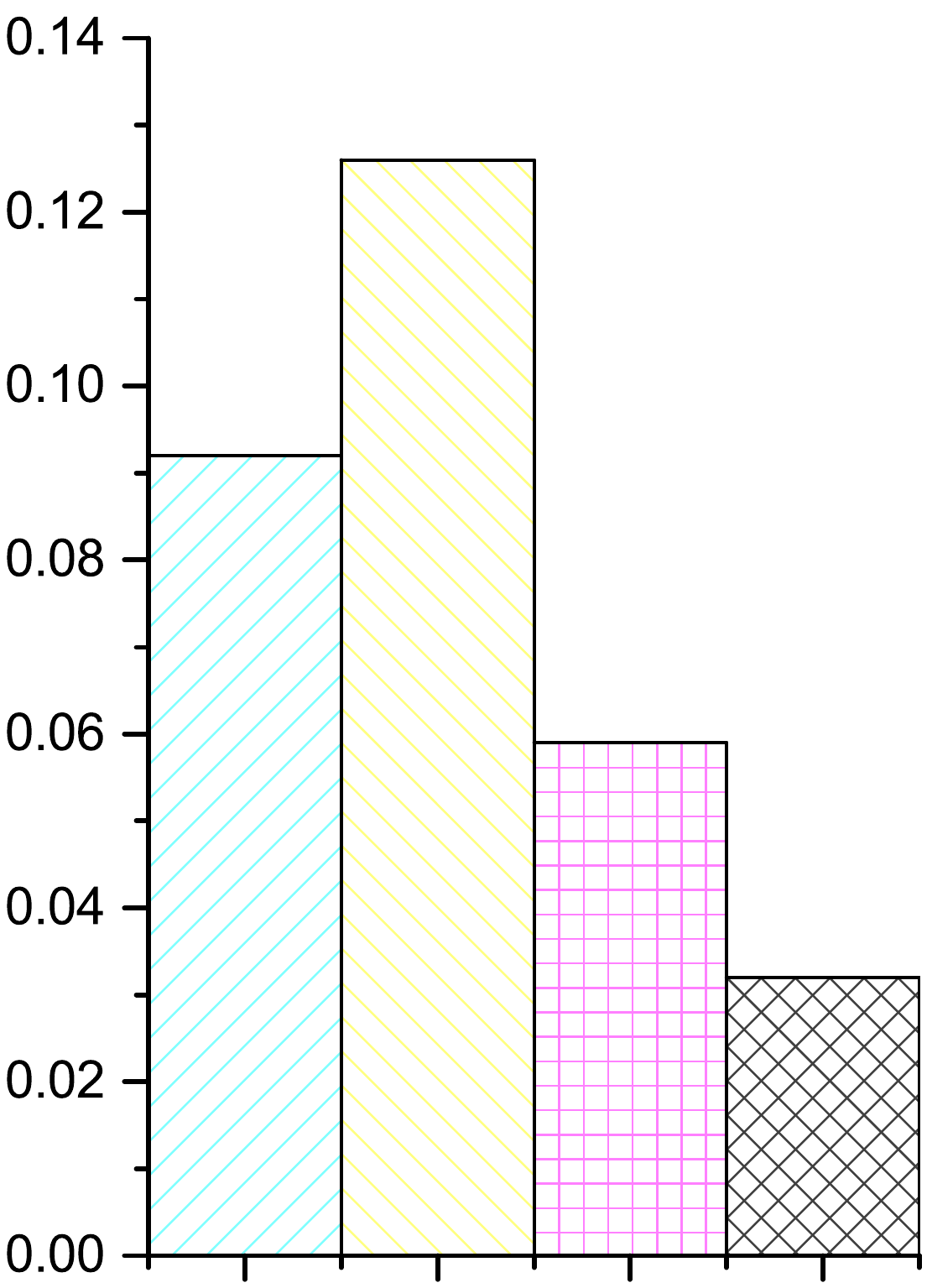}}\end{tabular}}\hspace{-12pt}
    \subfloat[]
    {\begin{tabular}{c}
    {\includegraphics[height= 0.09\textwidth]{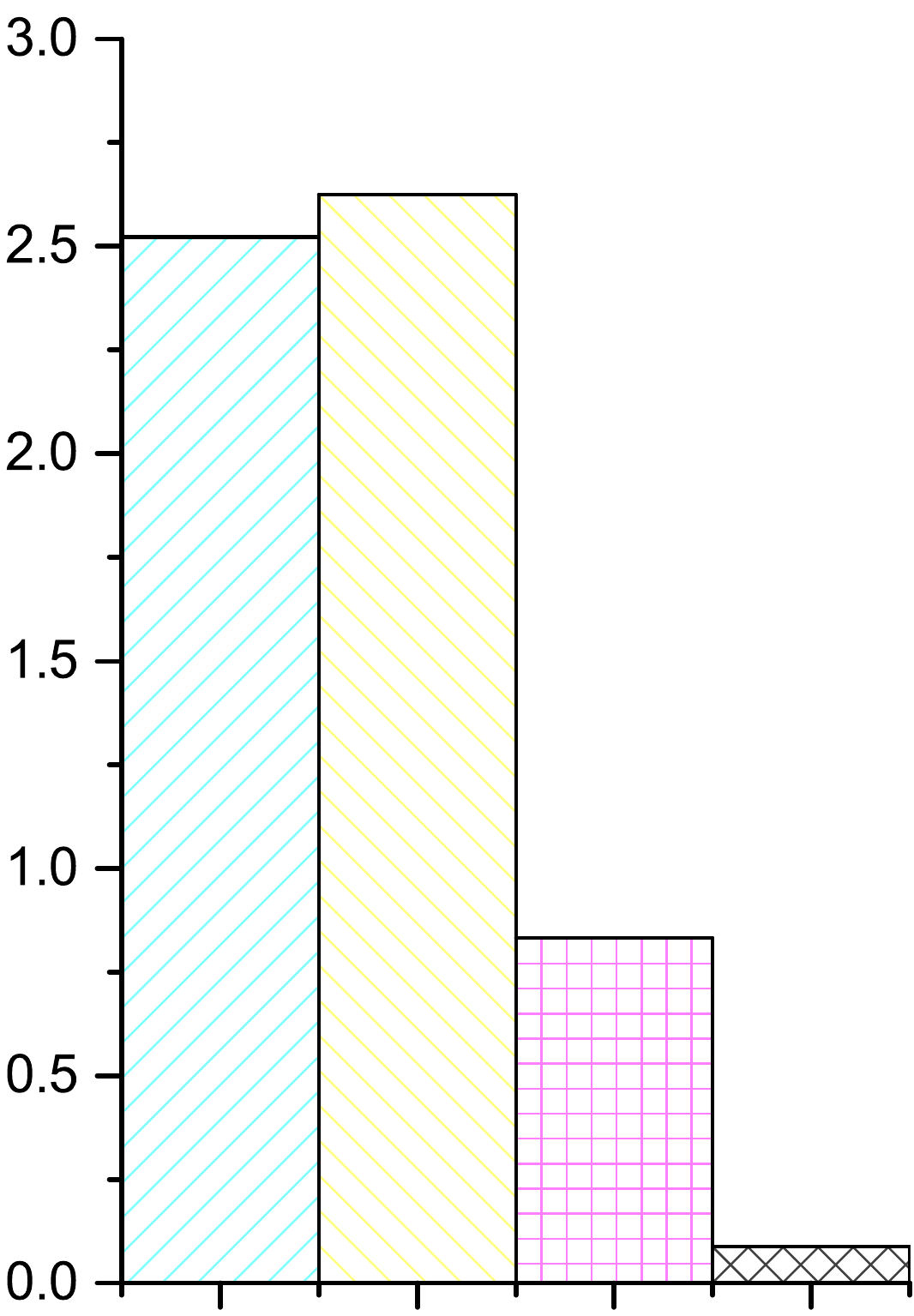}}\\
 {\includegraphics[height= 0.09\textwidth]{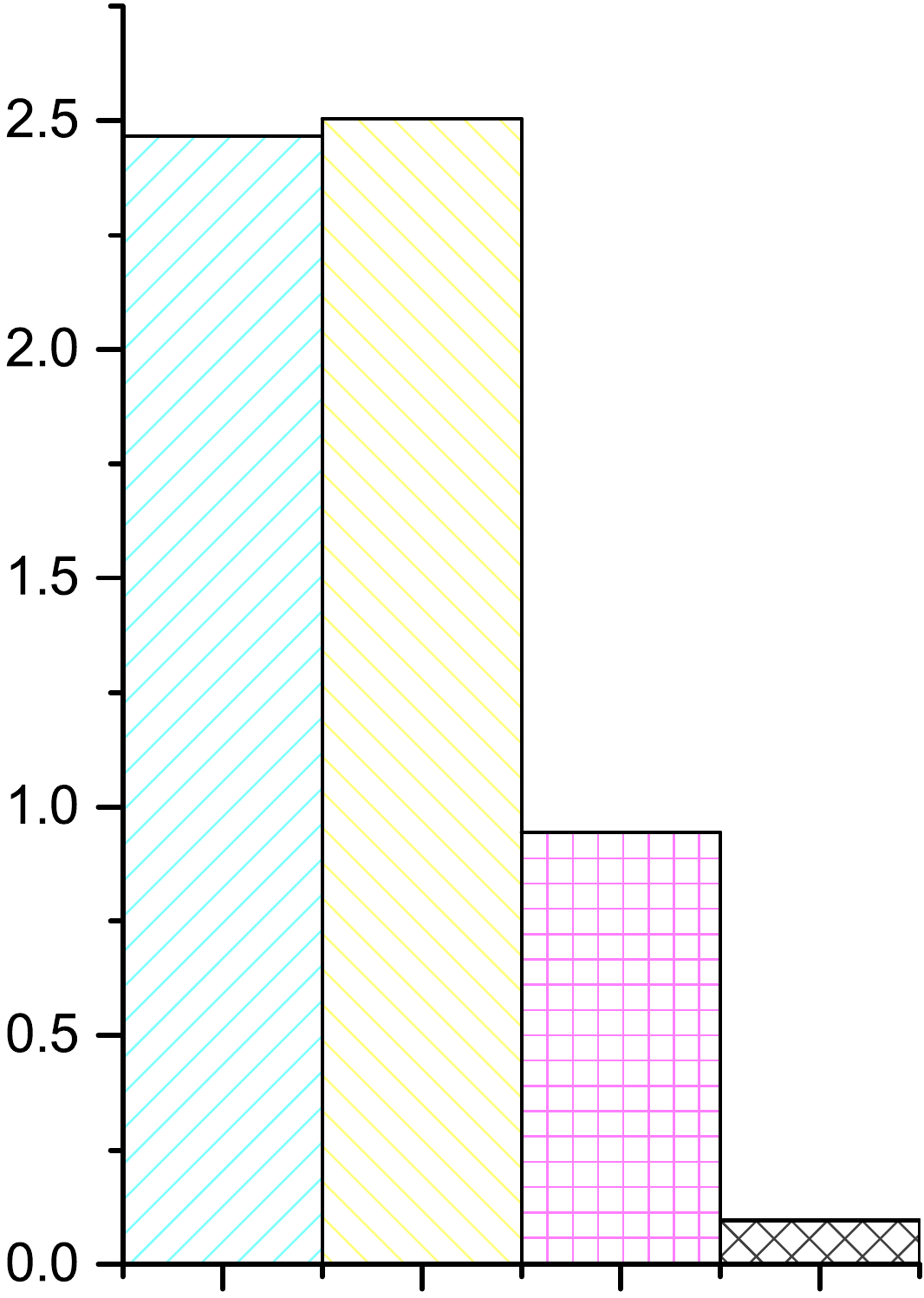}}\end{tabular}}\hspace{-12pt}
 \subfloat[]
    {\begin{tabular}{c}
    {\includegraphics[height= 0.09\textwidth]{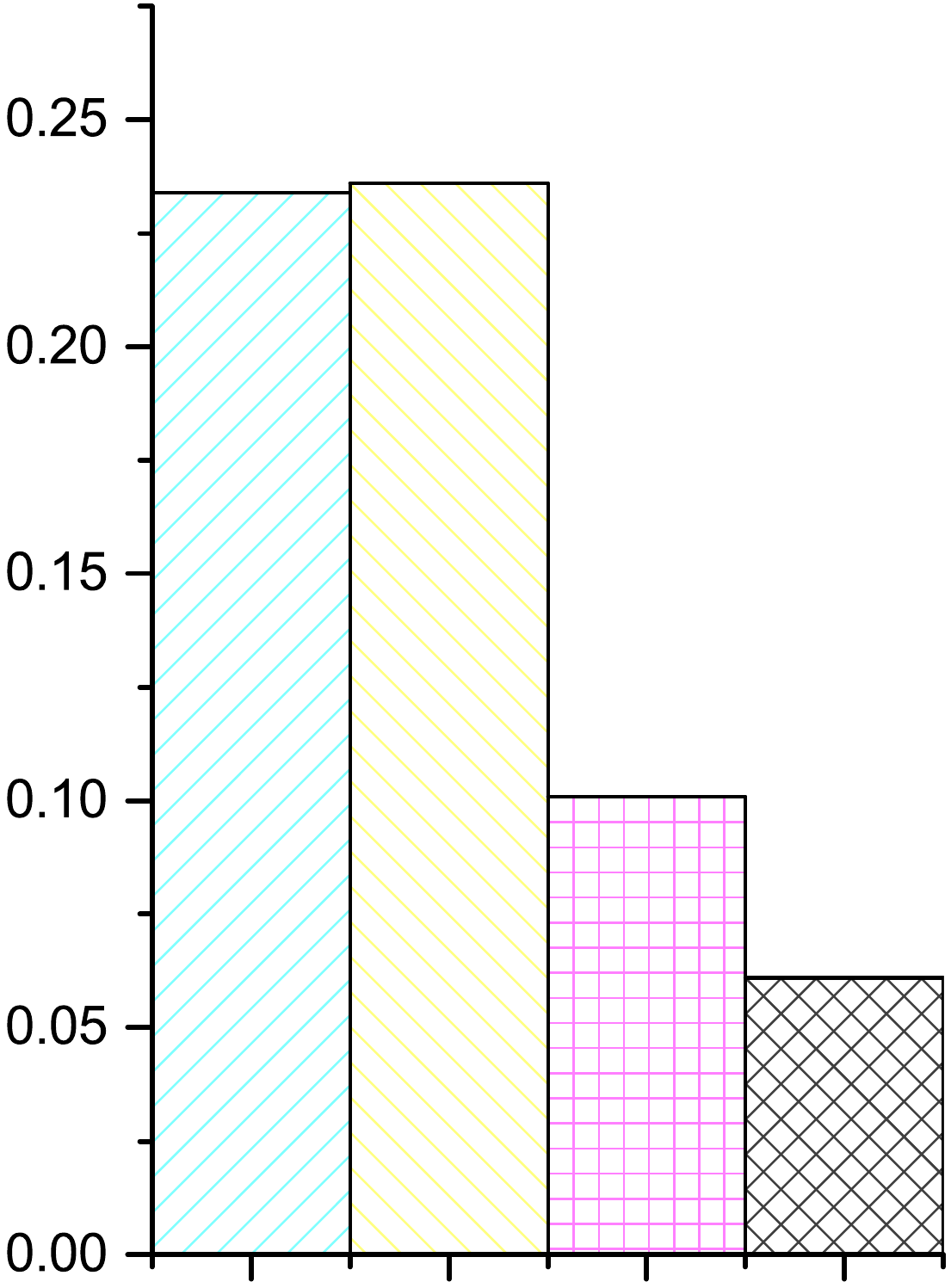}}\\
 {\includegraphics[height= 0.09\textwidth]{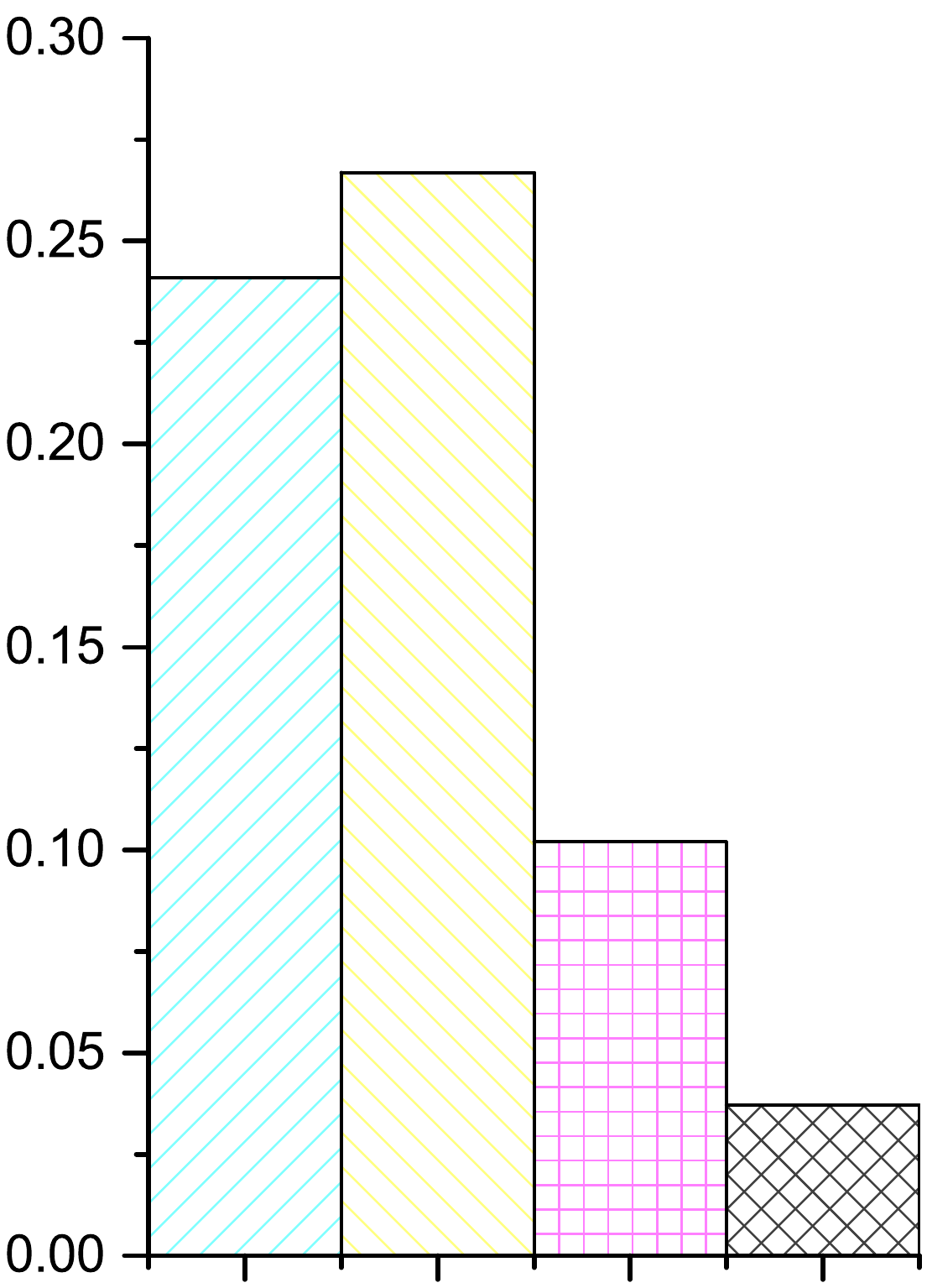}}\end{tabular}}\hspace{-12pt}
    \subfloat[]
    {\begin{tabular}{c}
    {\includegraphics[height= 0.09\textwidth]{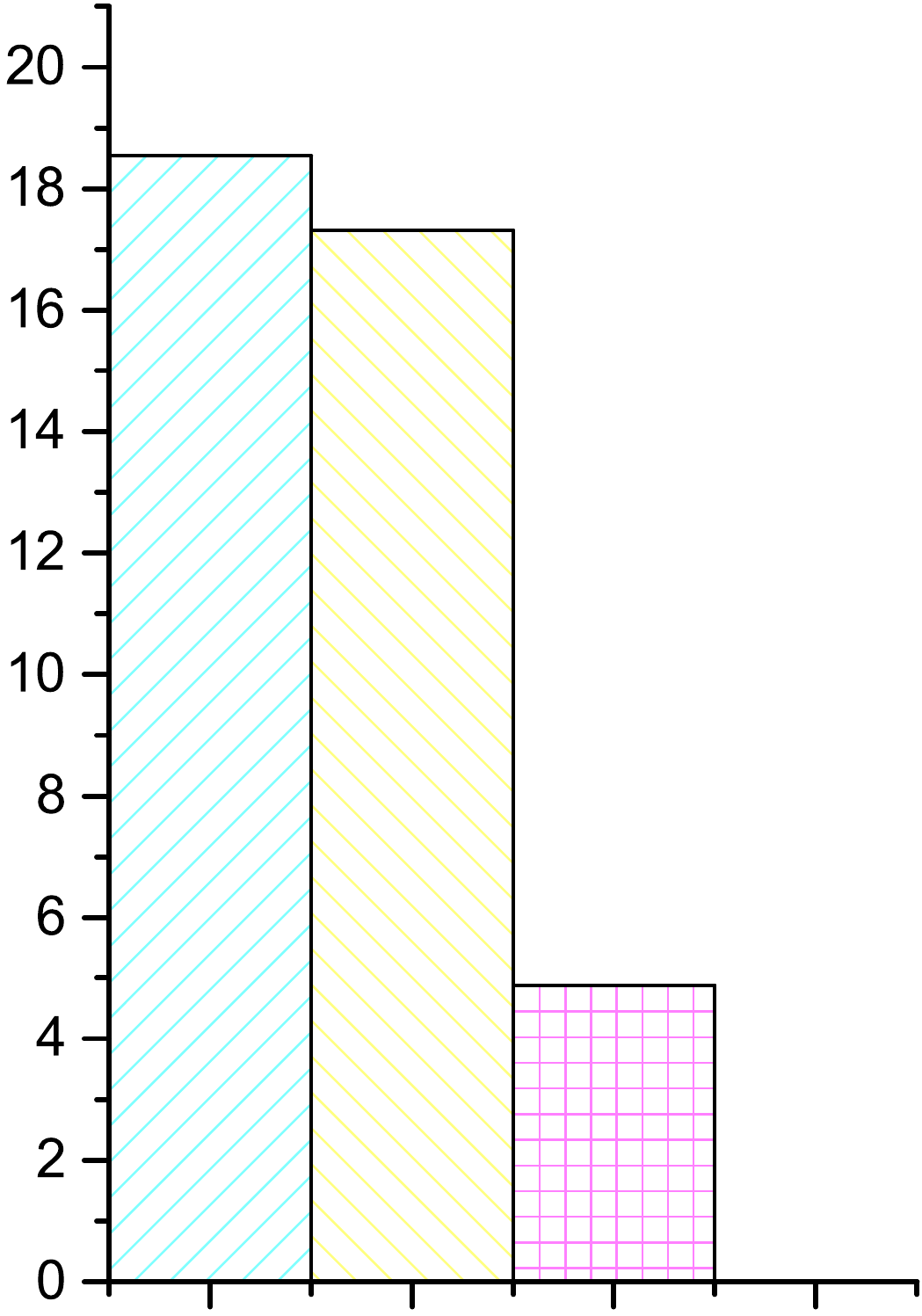}}\\
 {\includegraphics[height= 0.09\textwidth]{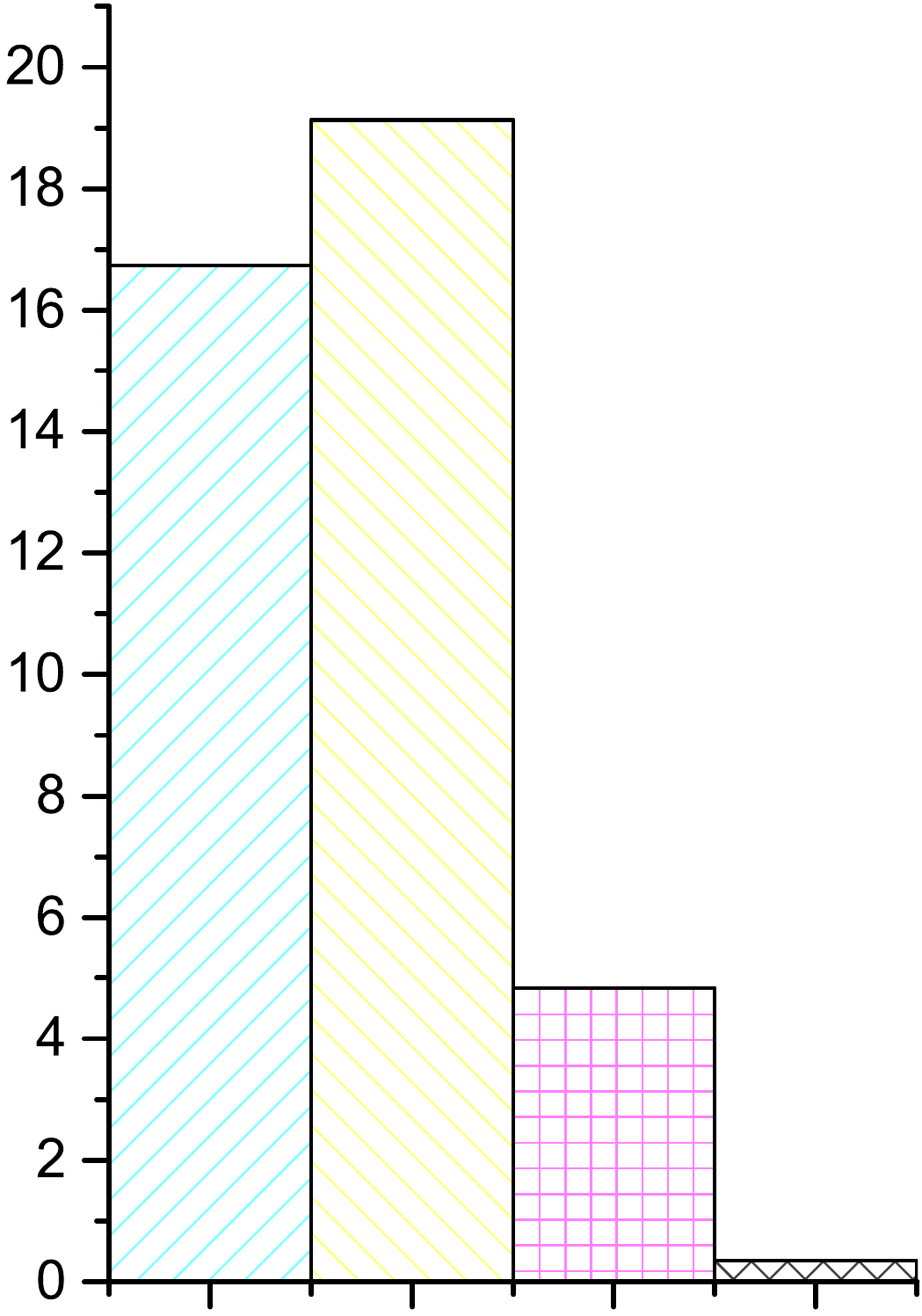}}\end{tabular}}\hspace{-12pt}
 \subfloat[]
    {\begin{tabular}{c}
    {\includegraphics[height= 0.09\textwidth]{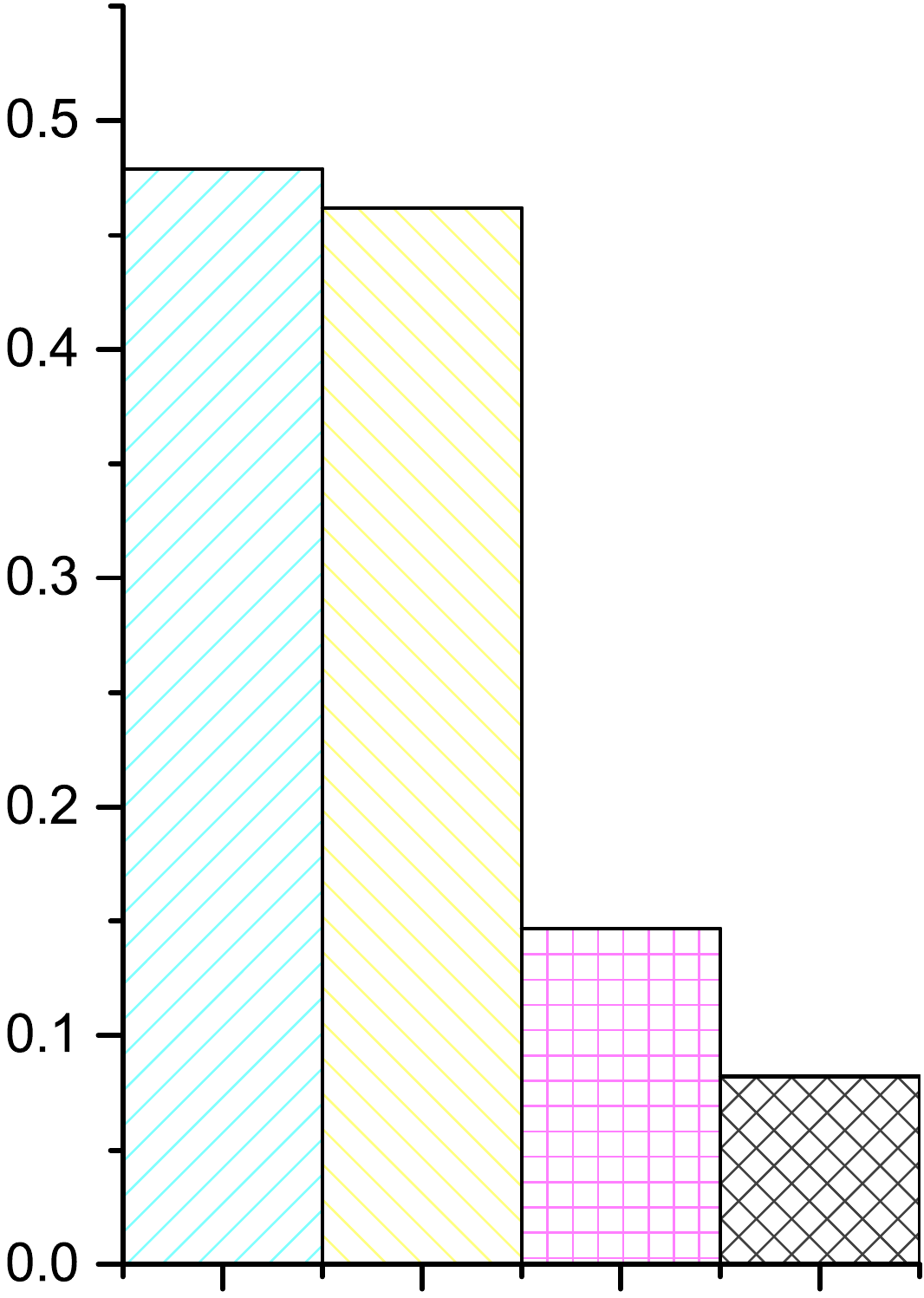}}\\
 {\includegraphics[height= 0.09\textwidth]{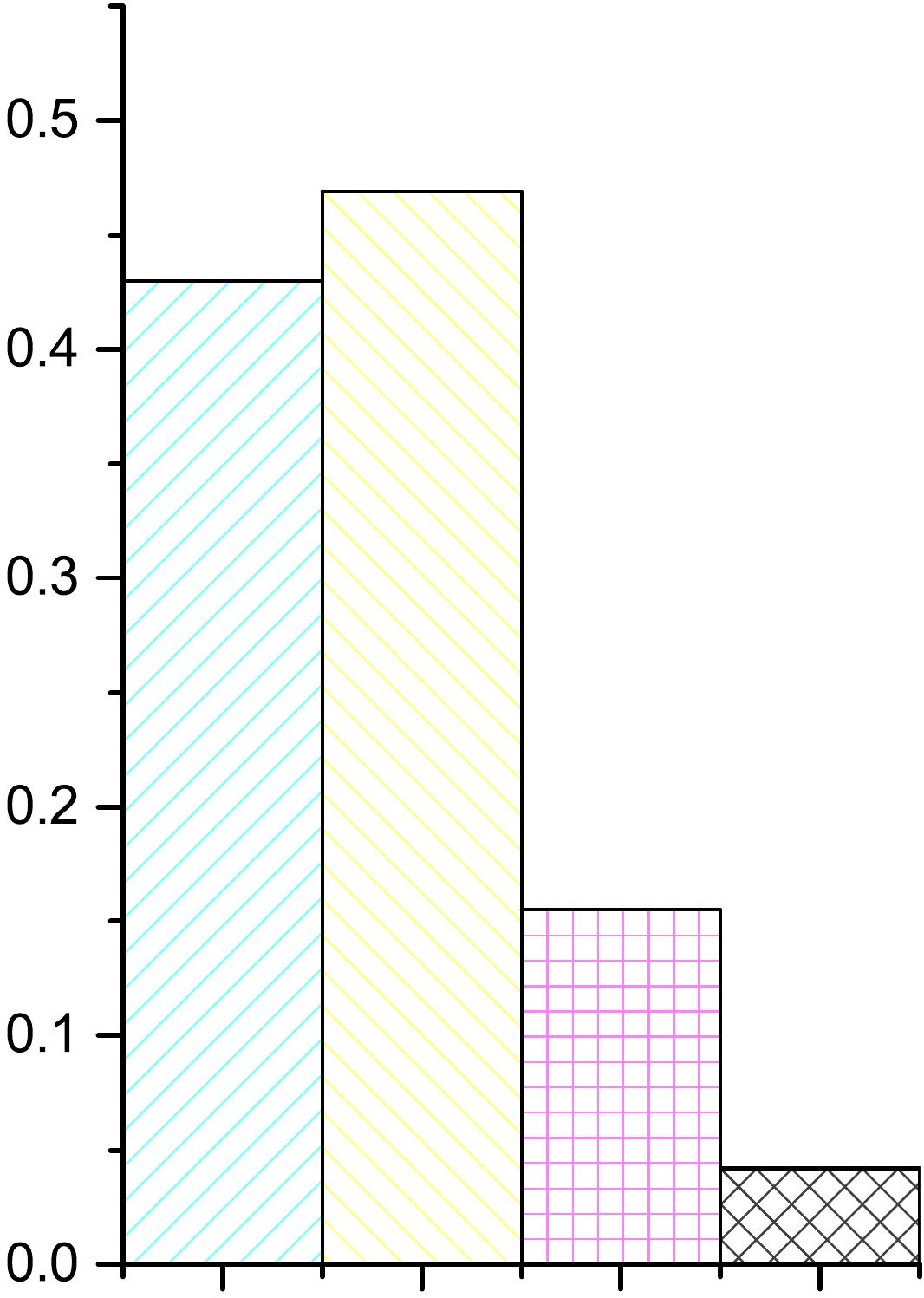}}\end{tabular}}\hspace{-12pt}
 \subfloat[]
    {\begin{tabular}{c}
    {\includegraphics[height= 0.09\textwidth]{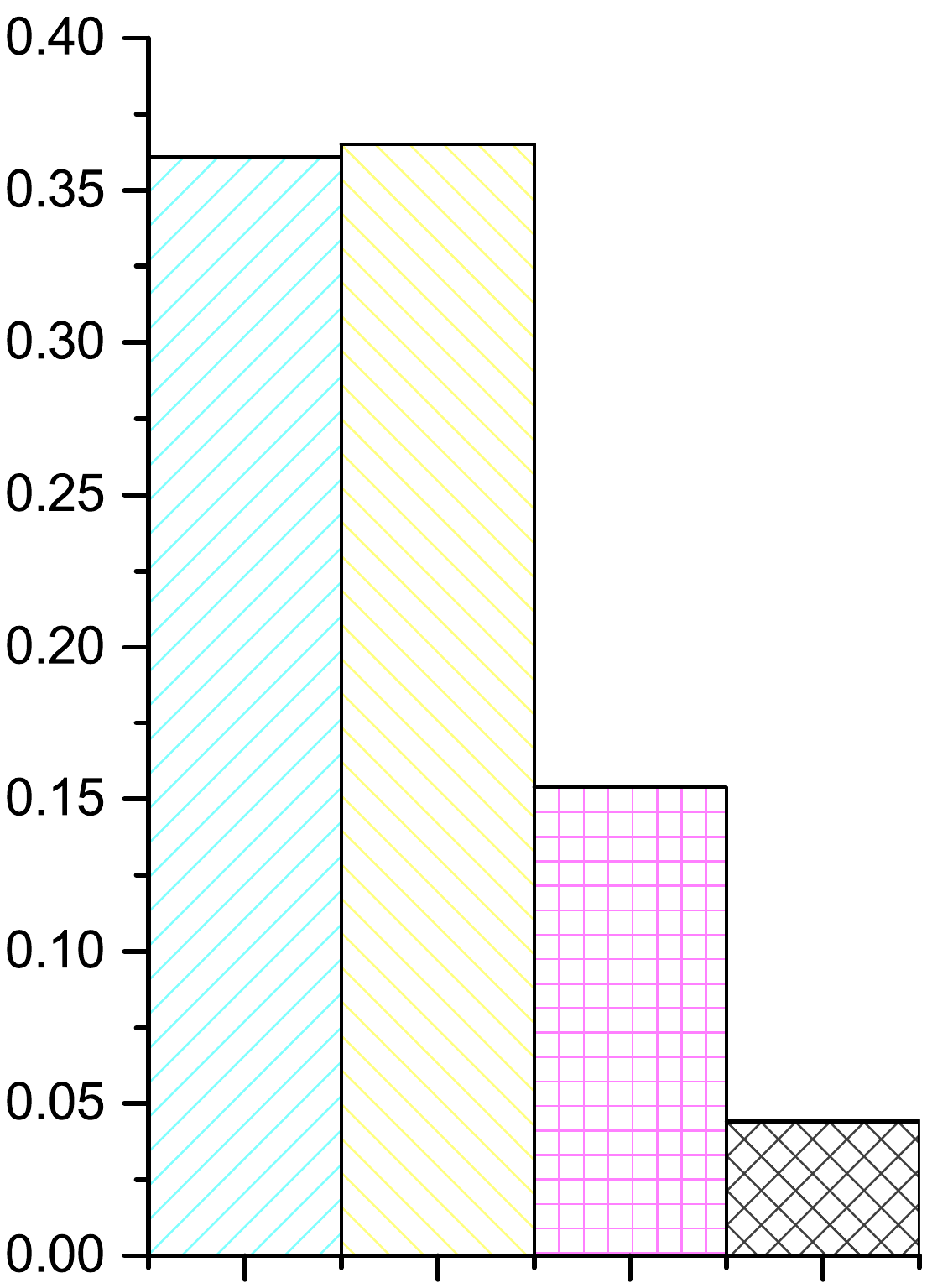}}\\
 {\includegraphics[height= 0.09\textwidth]{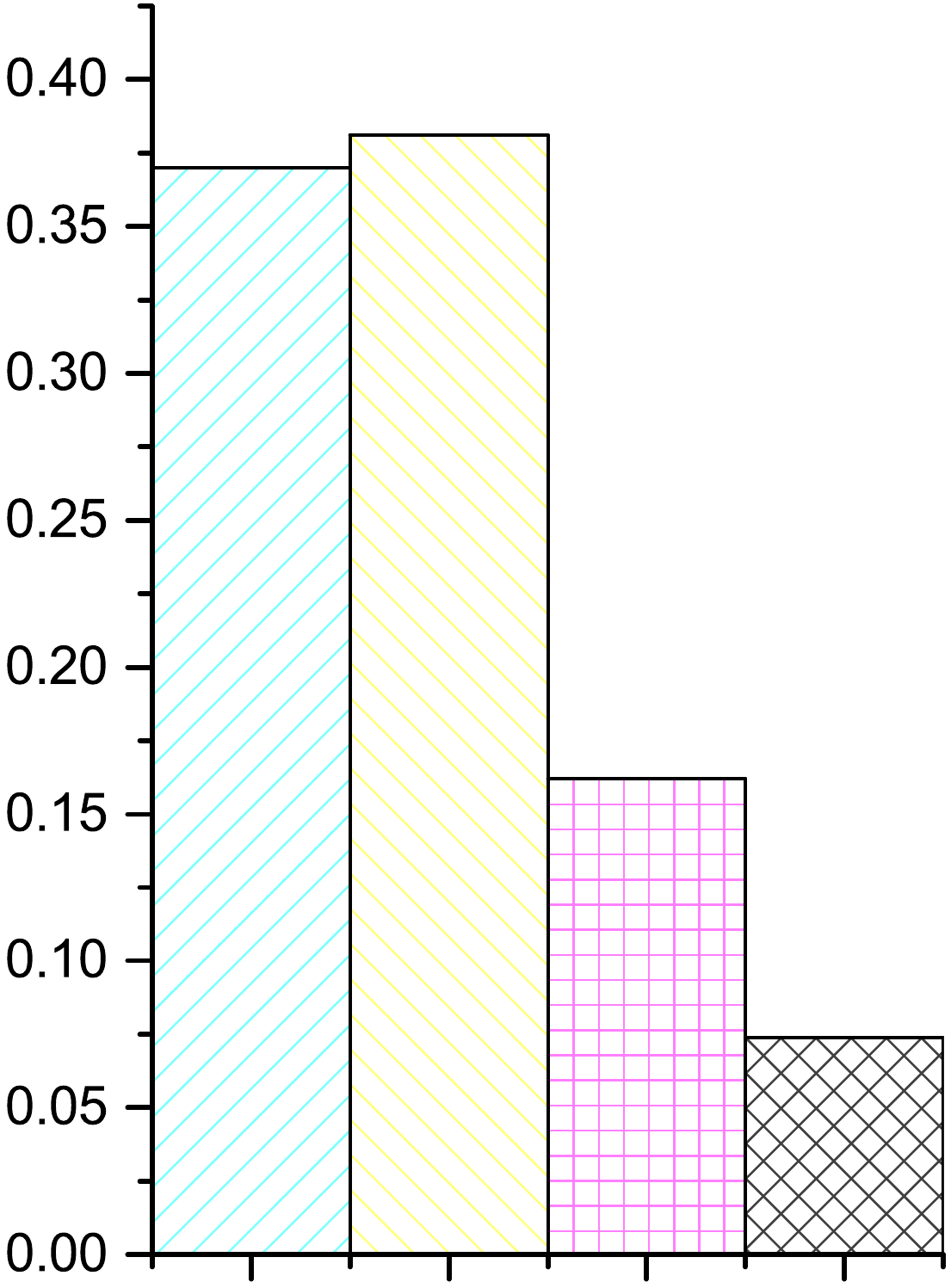}}\end{tabular}}\hspace{-12pt}
 \subfloat[]
    {\begin{tabular}{c}
    {\includegraphics[height= 0.09\textwidth]{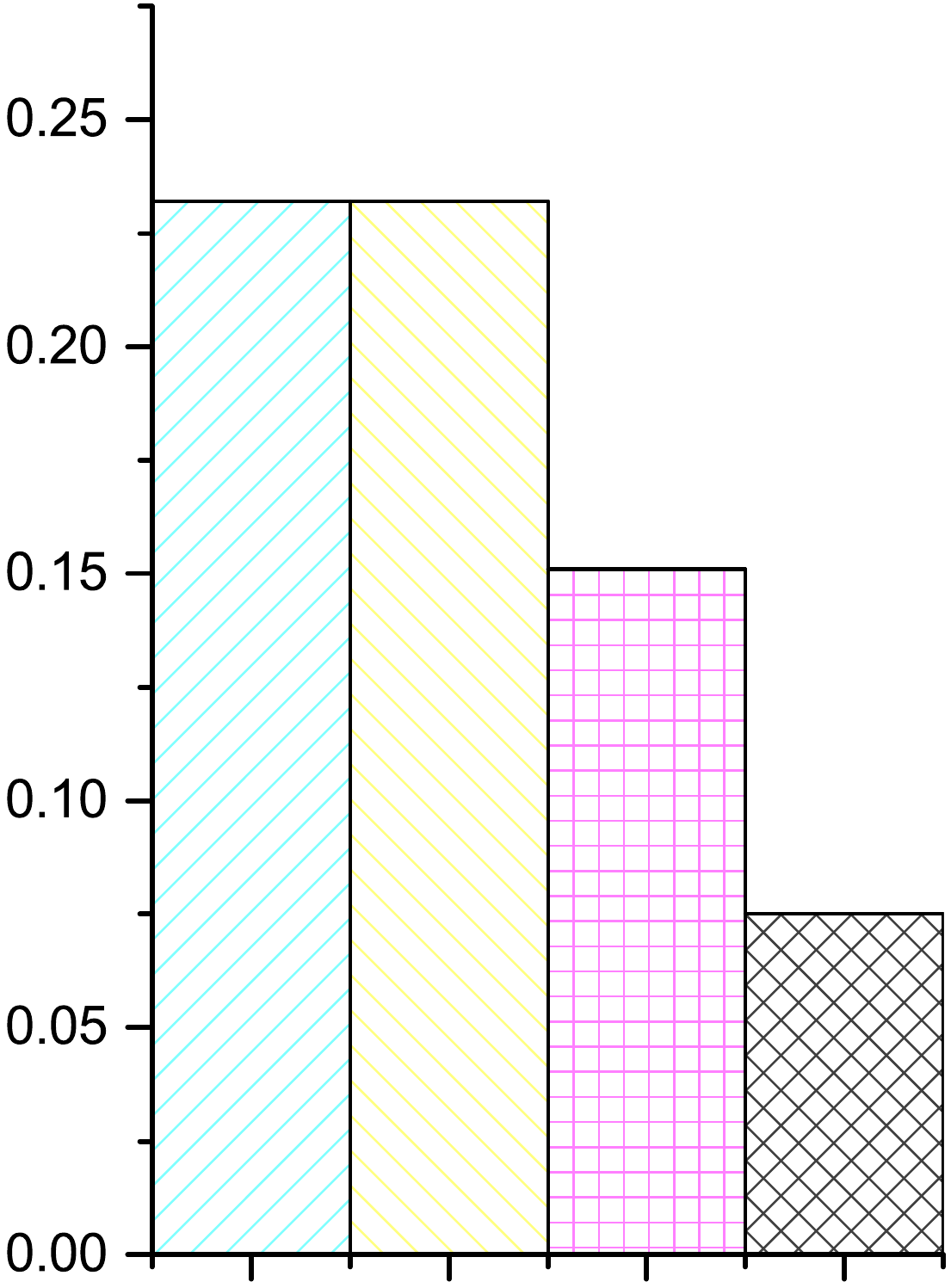}}\\
 {\includegraphics[height= 0.09\textwidth]{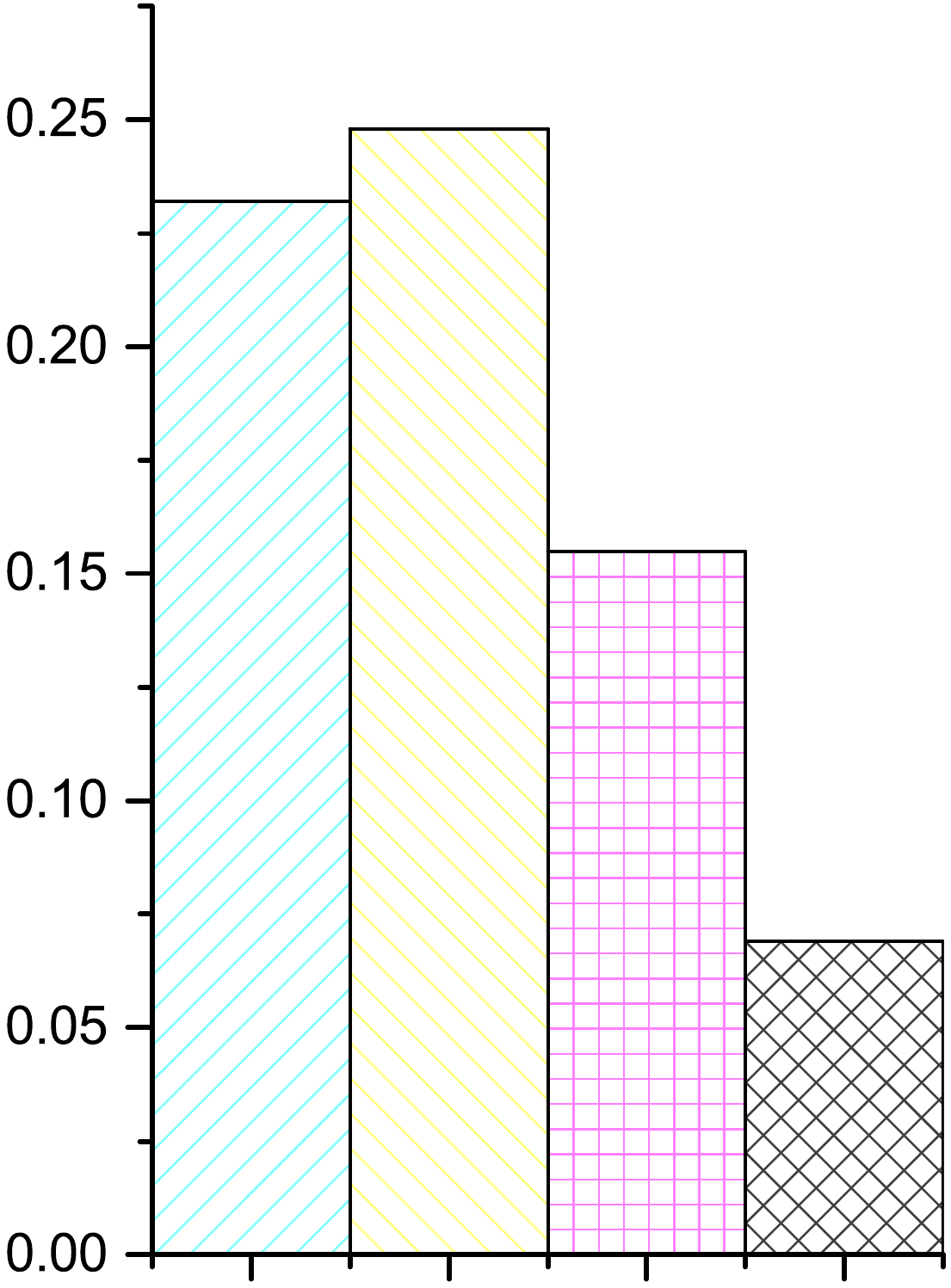}}\end{tabular}}\hspace{-12pt}
    \subfloat[]
    {\begin{tabular}{c}
    {\includegraphics[height= 0.09\textwidth]{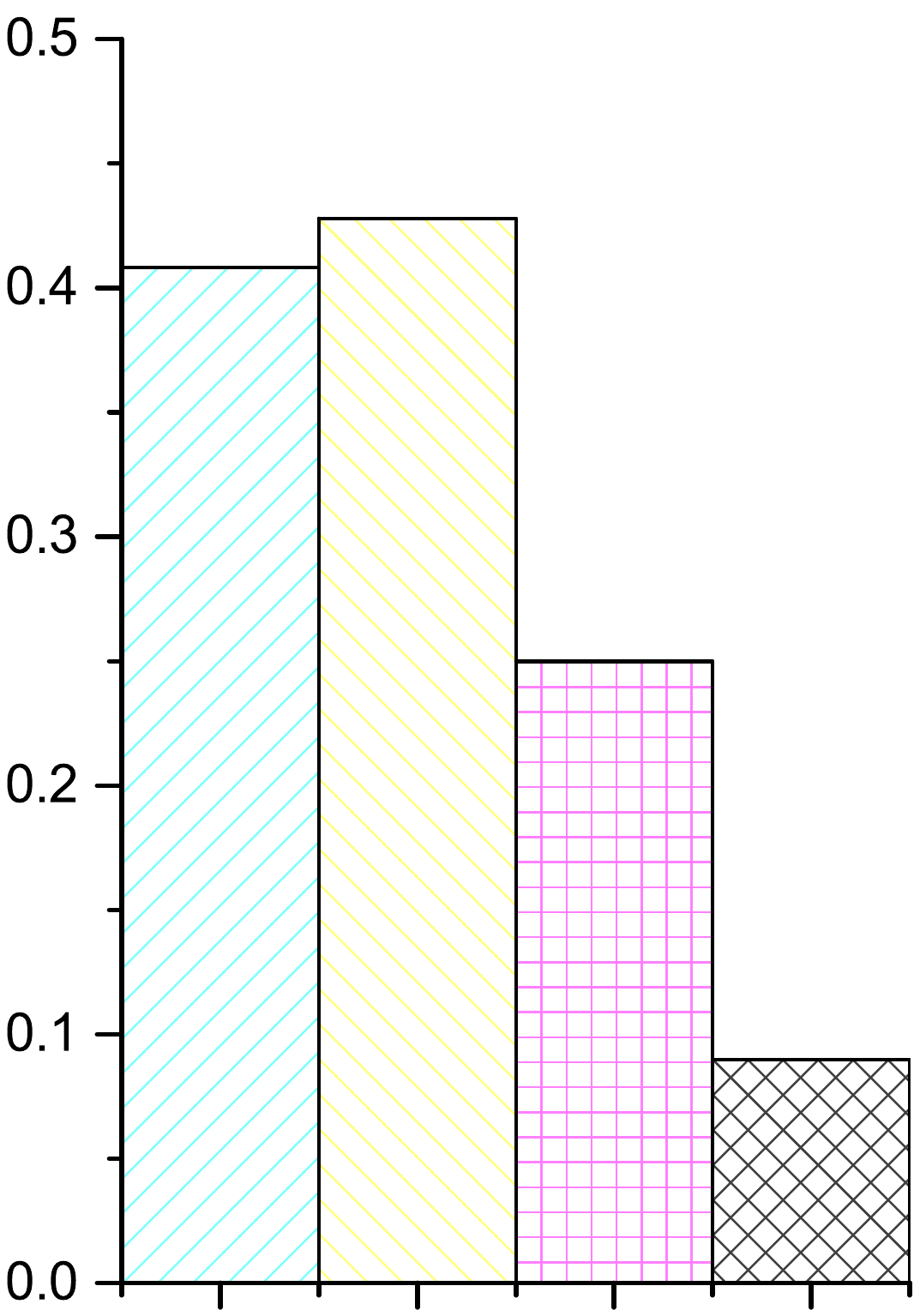}}\\
 {\includegraphics[height= 0.09\textwidth]{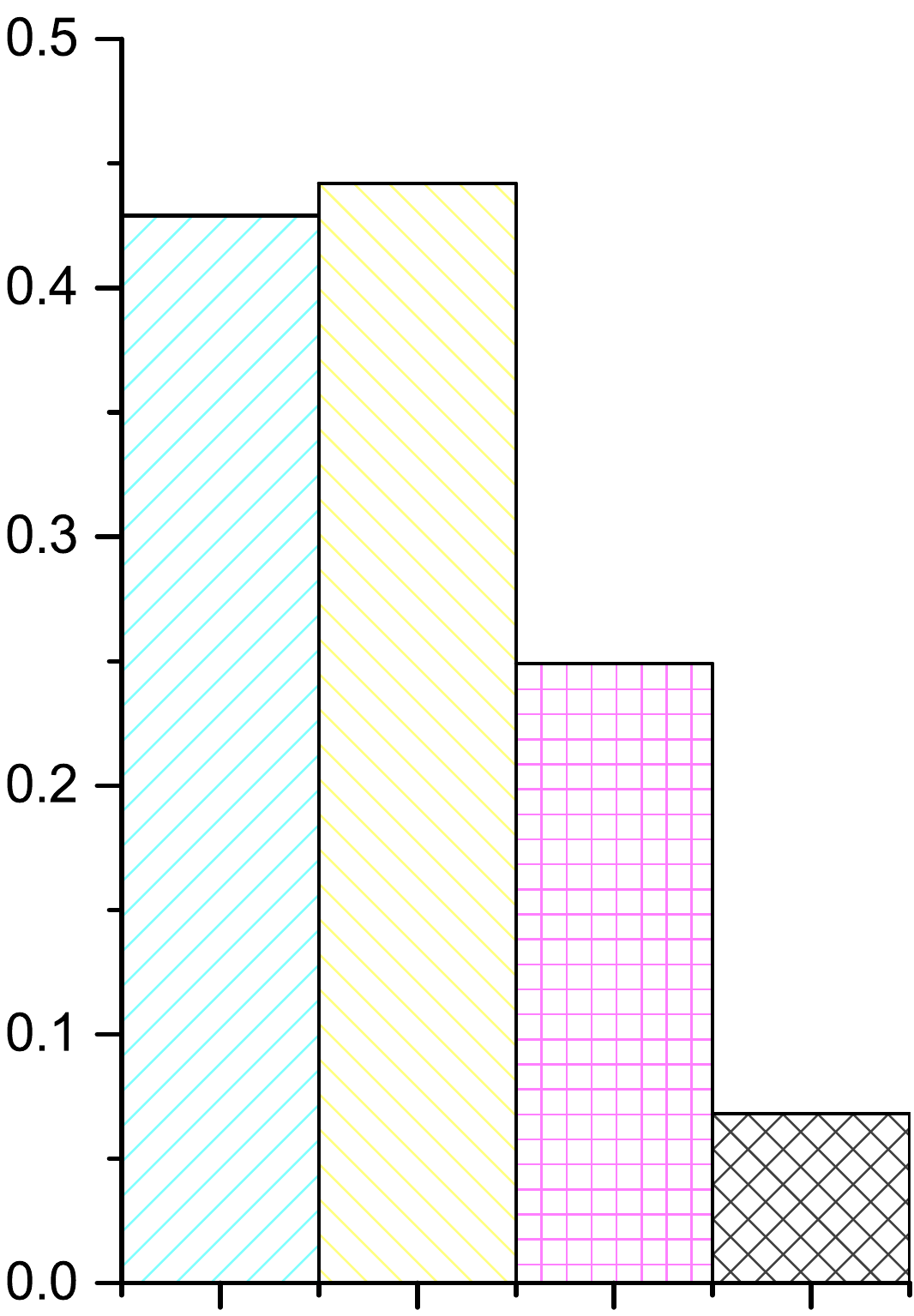}}\end{tabular}}\hspace{-12pt}
 \subfloat[]
    {\begin{tabular}{c}
    {\includegraphics[height= 0.09\textwidth]{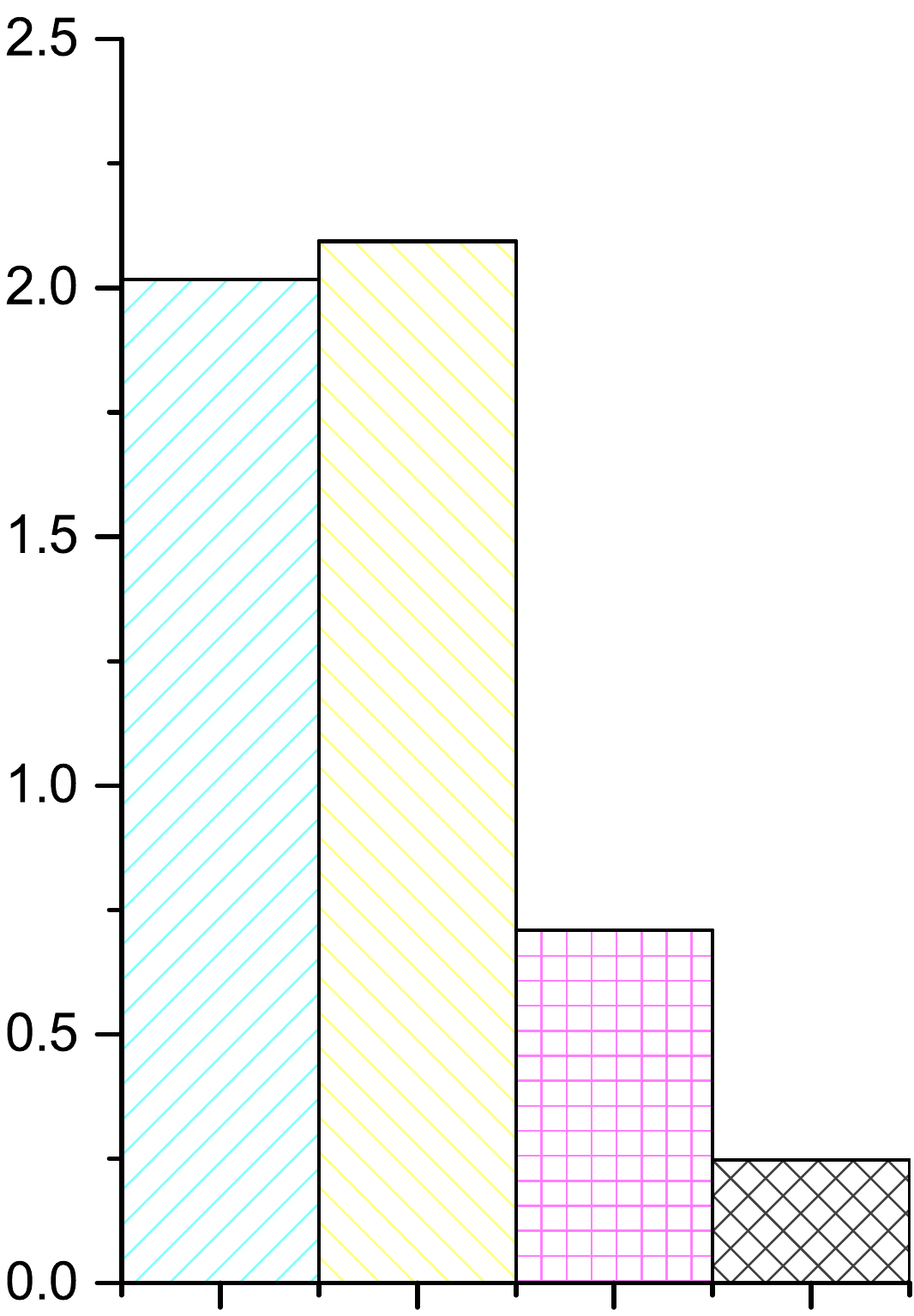}}\\
 {\includegraphics[height= 0.09\textwidth]{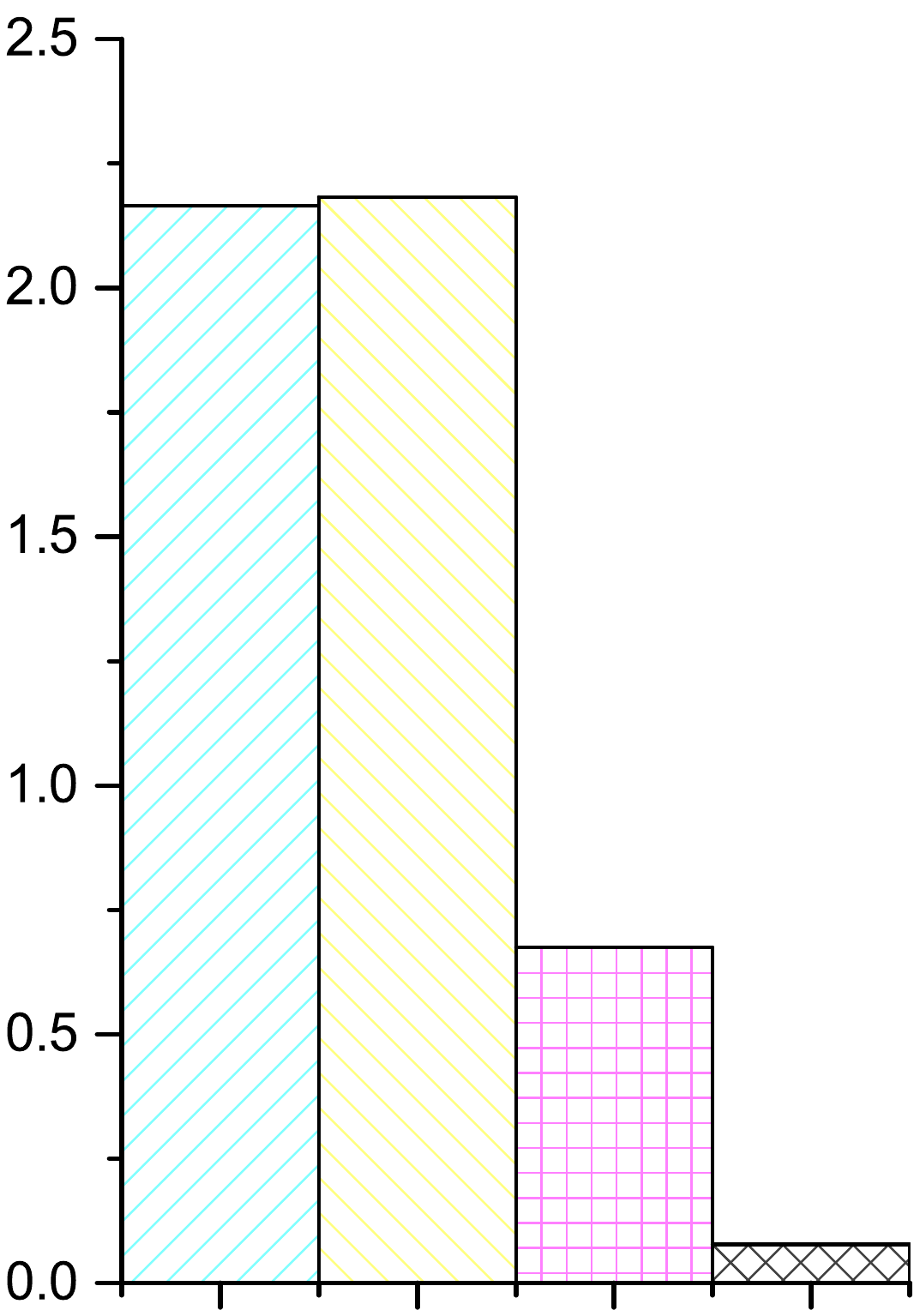}}\end{tabular}}\hspace{-12pt}
 \subfloat[]
    {\begin{tabular}{c}
    {\includegraphics[height= 0.09\textwidth]{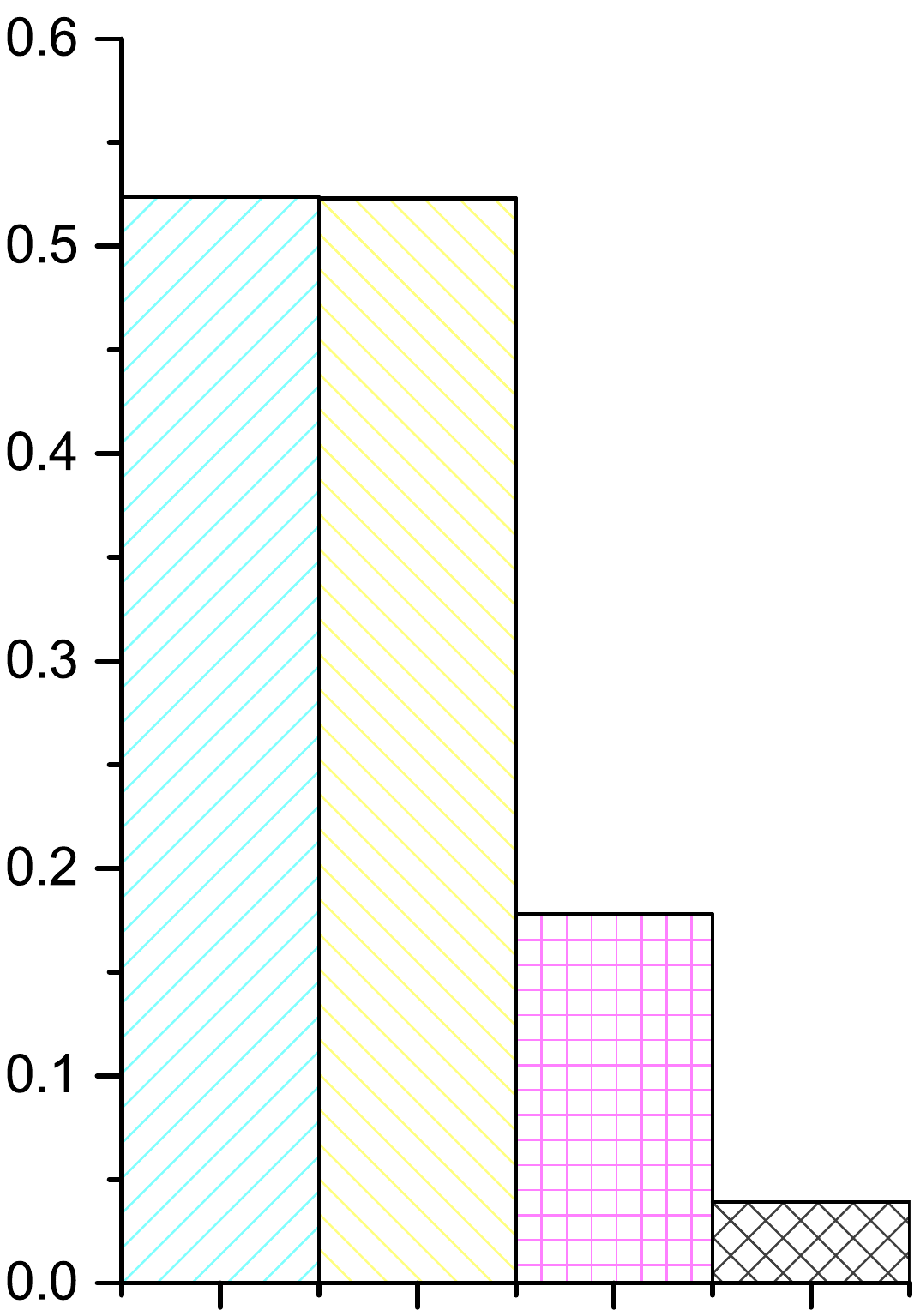}}\\
 {\includegraphics[height= 0.09\textwidth]{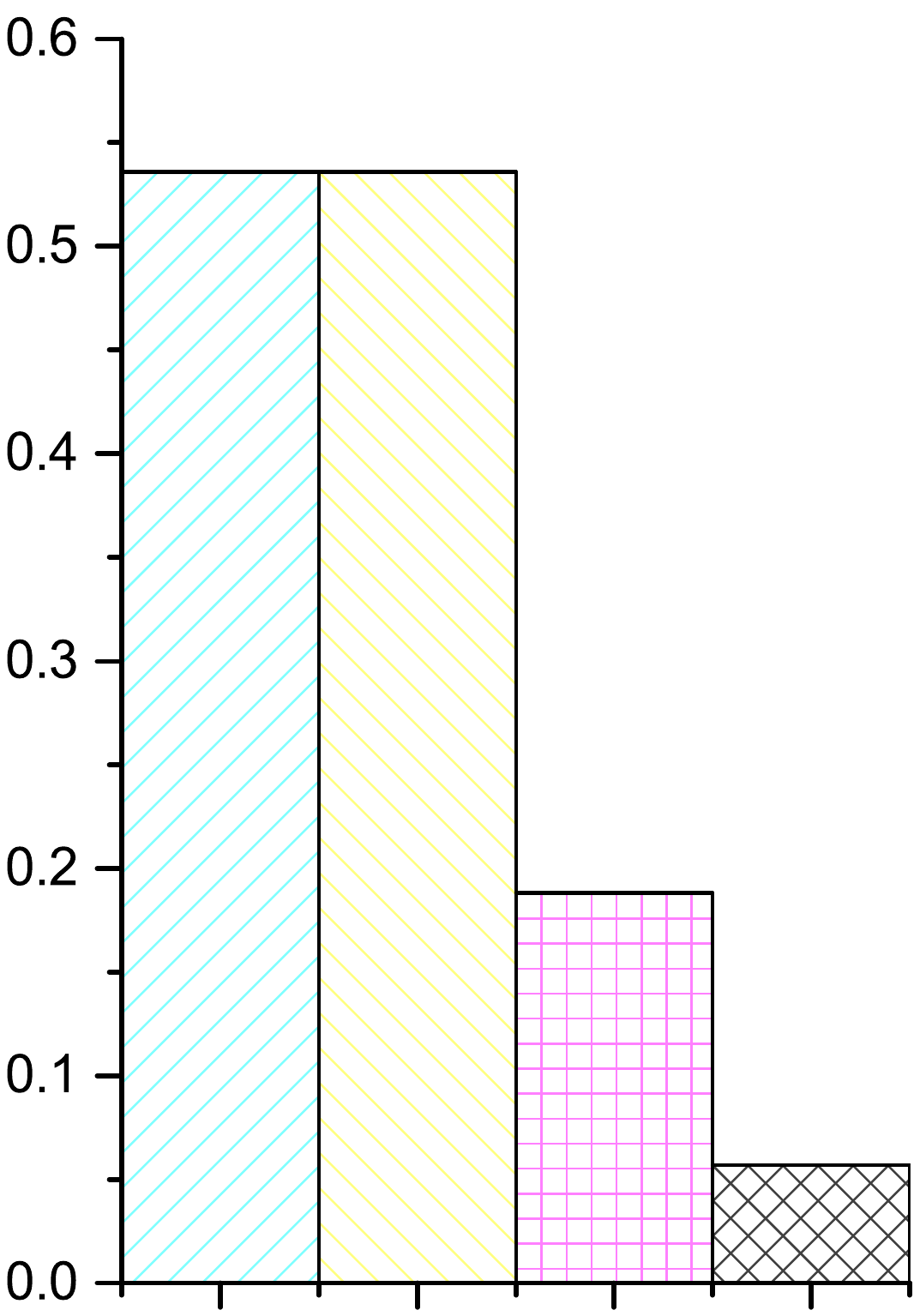}}\end{tabular}}\hspace{-12pt}
 \subfloat[]
    {\begin{tabular}{c}
    {\includegraphics[height= 0.09\textwidth]{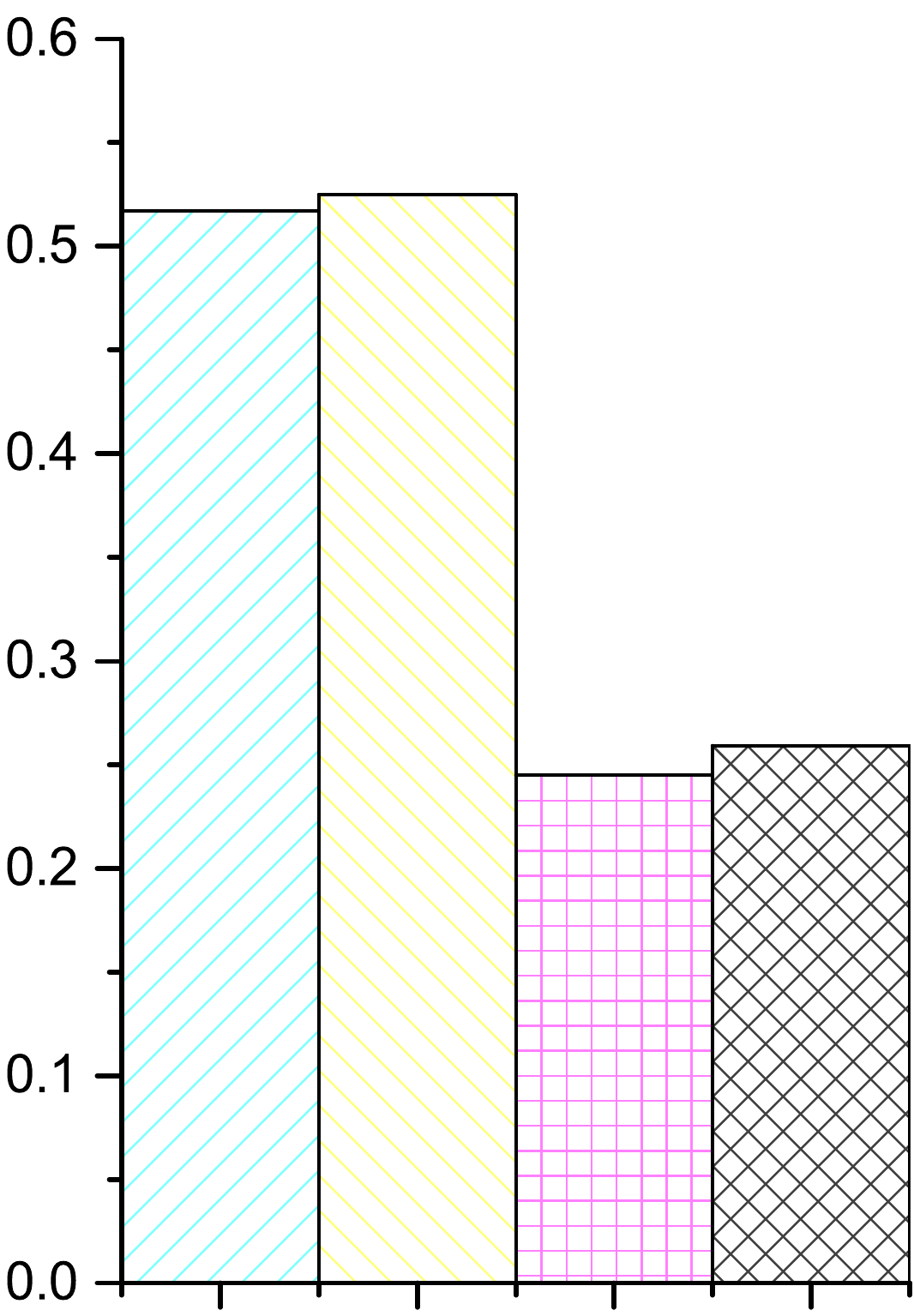}}\\
 {\includegraphics[height= 0.09\textwidth]{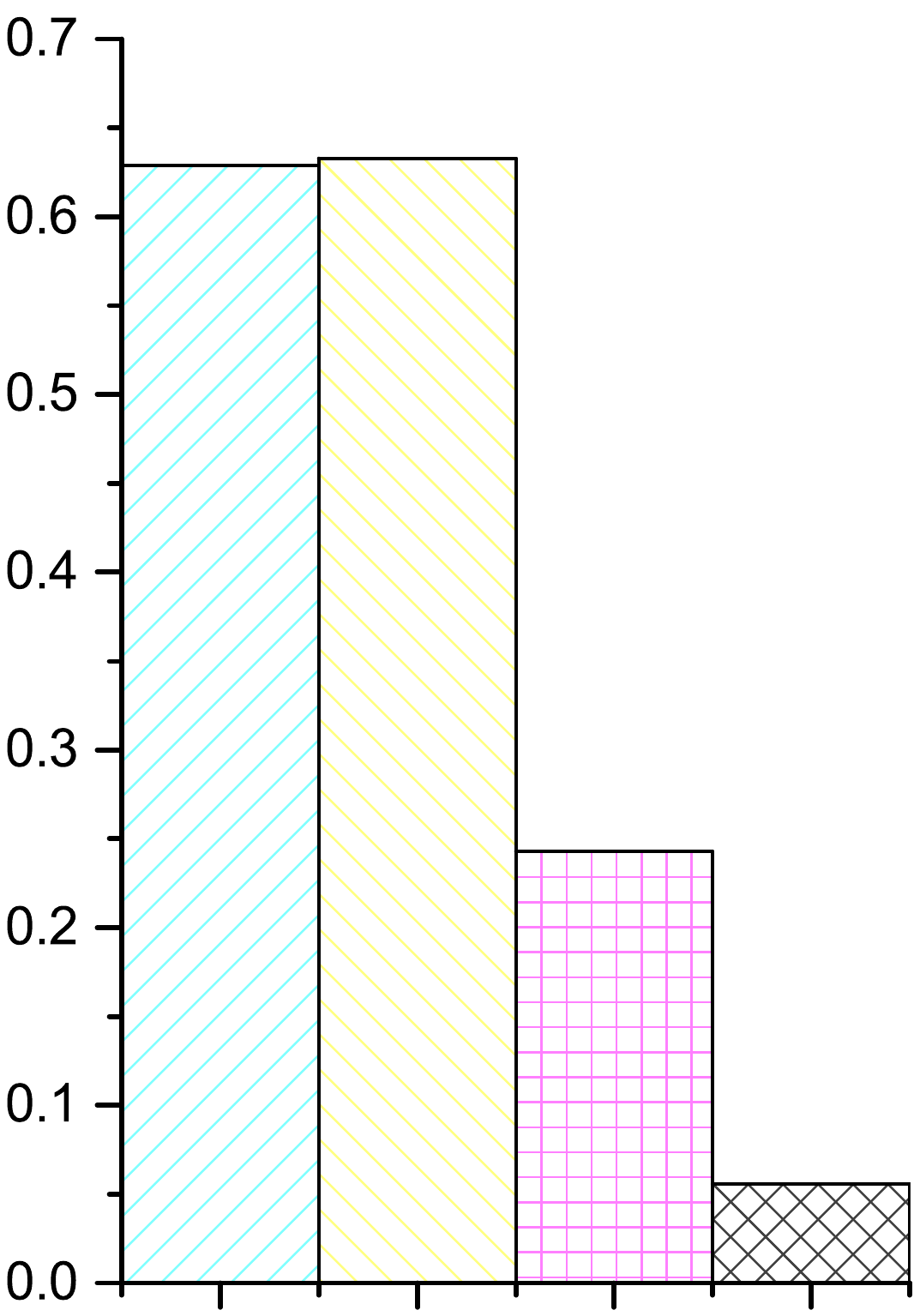}}\end{tabular}}\hspace{-12pt}
    \subfloat[]
    {\begin{tabular}{c}
    {\includegraphics[height= 0.09\textwidth]{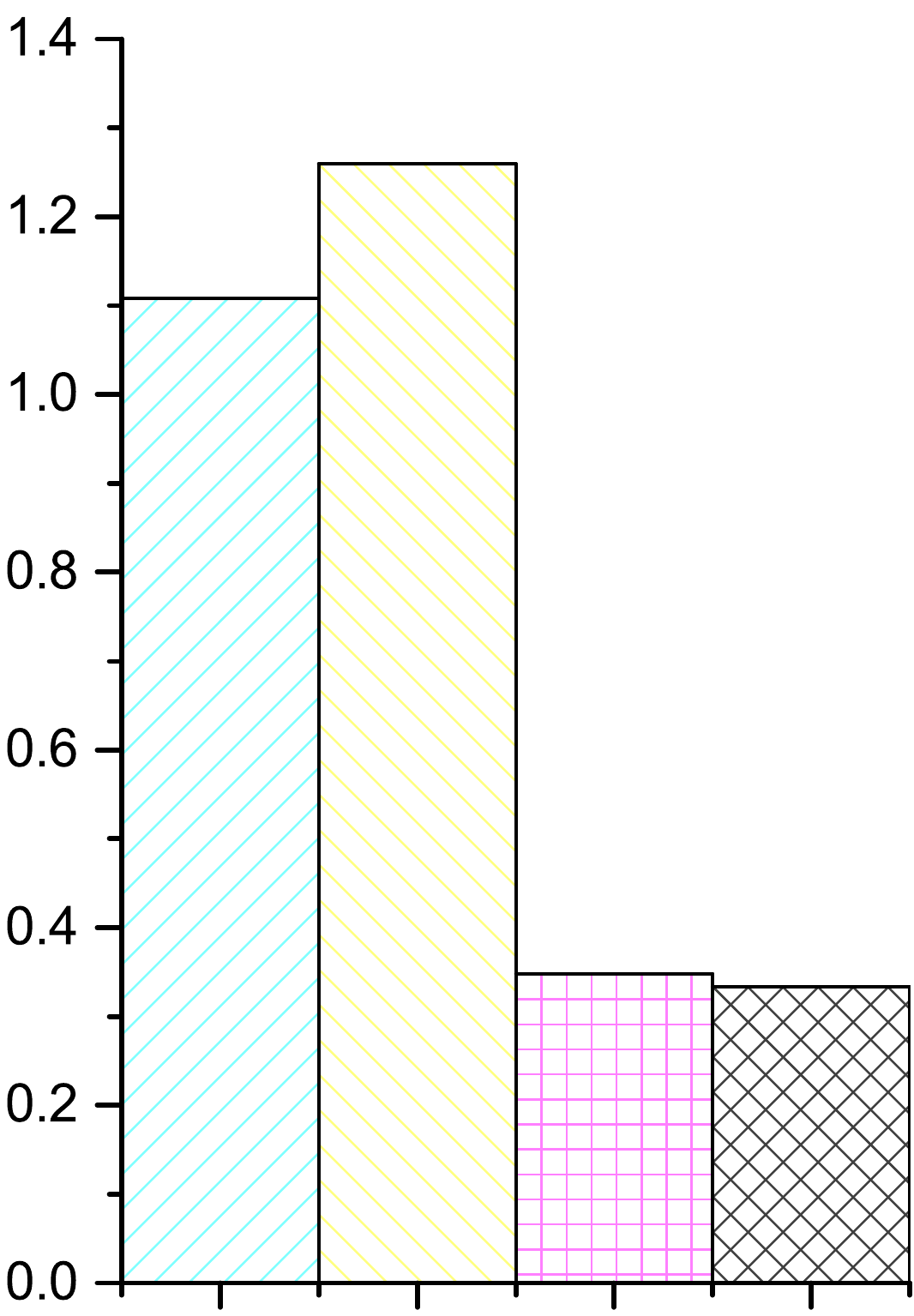}}\\
 {\includegraphics[height= 0.09\textwidth]{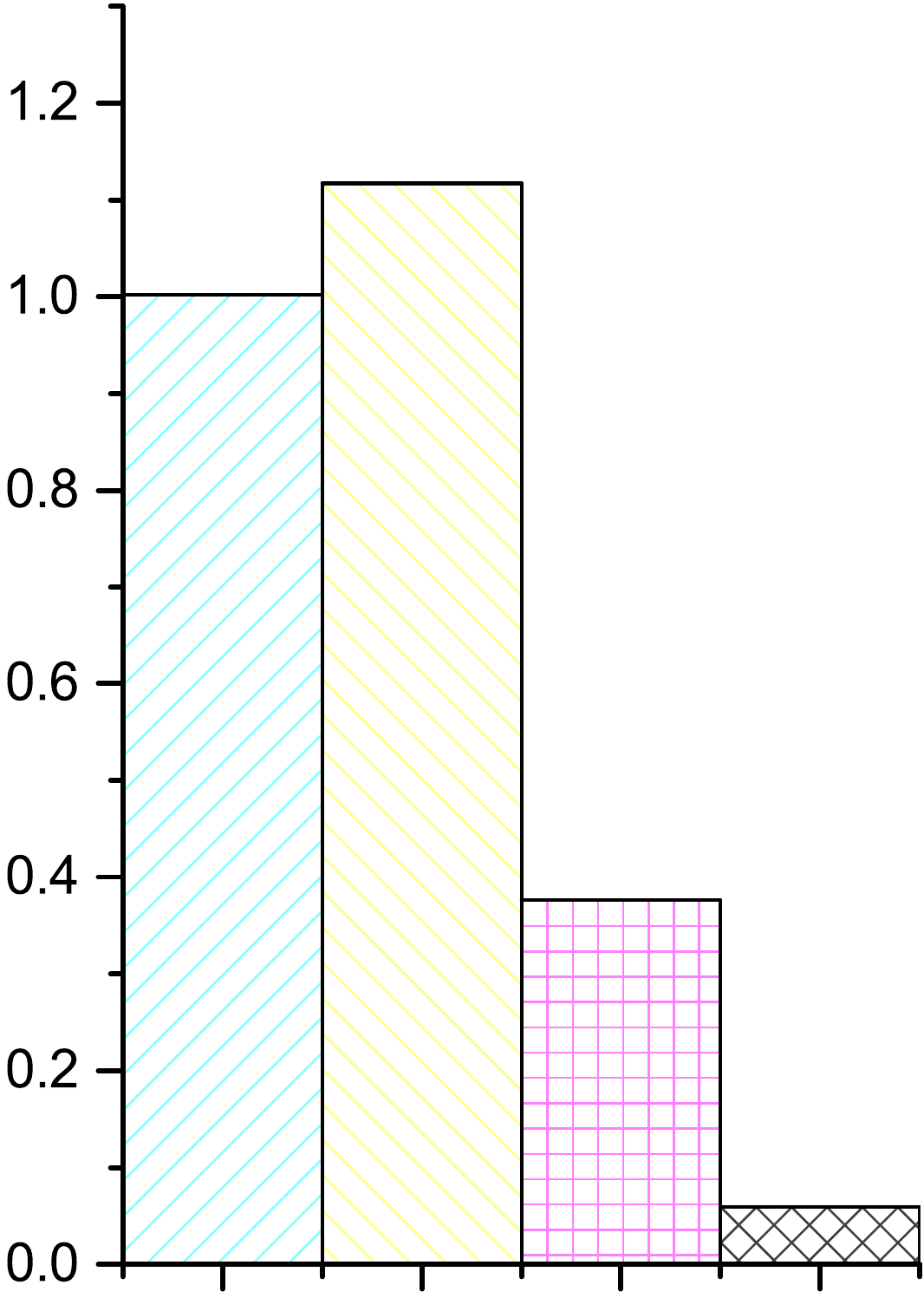}}\end{tabular}}\hspace{-12pt}
 \subfloat[]
    {\begin{tabular}{c}
    {\includegraphics[height= 0.09\textwidth]{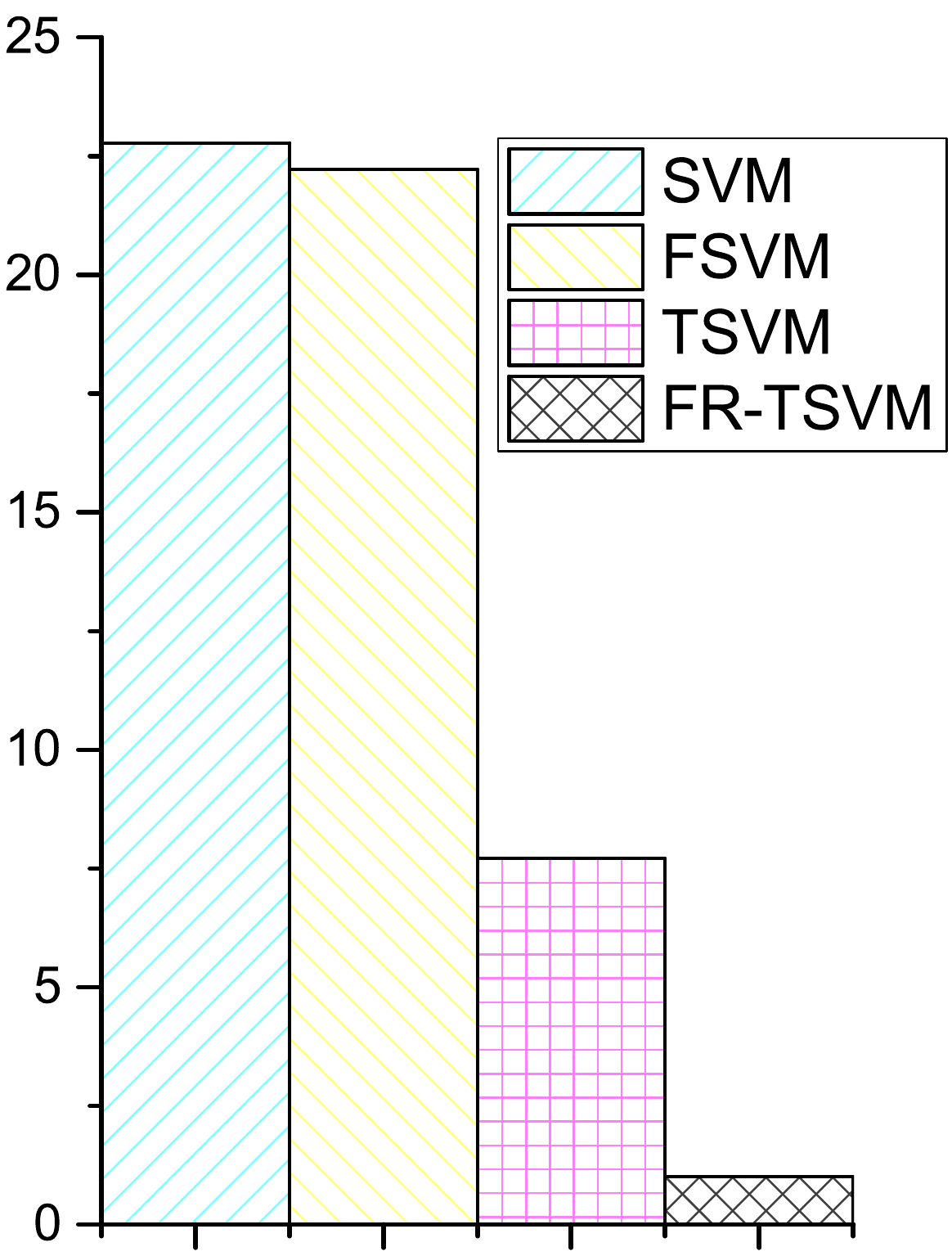}}\\
 {\includegraphics[height= 0.09\textwidth]{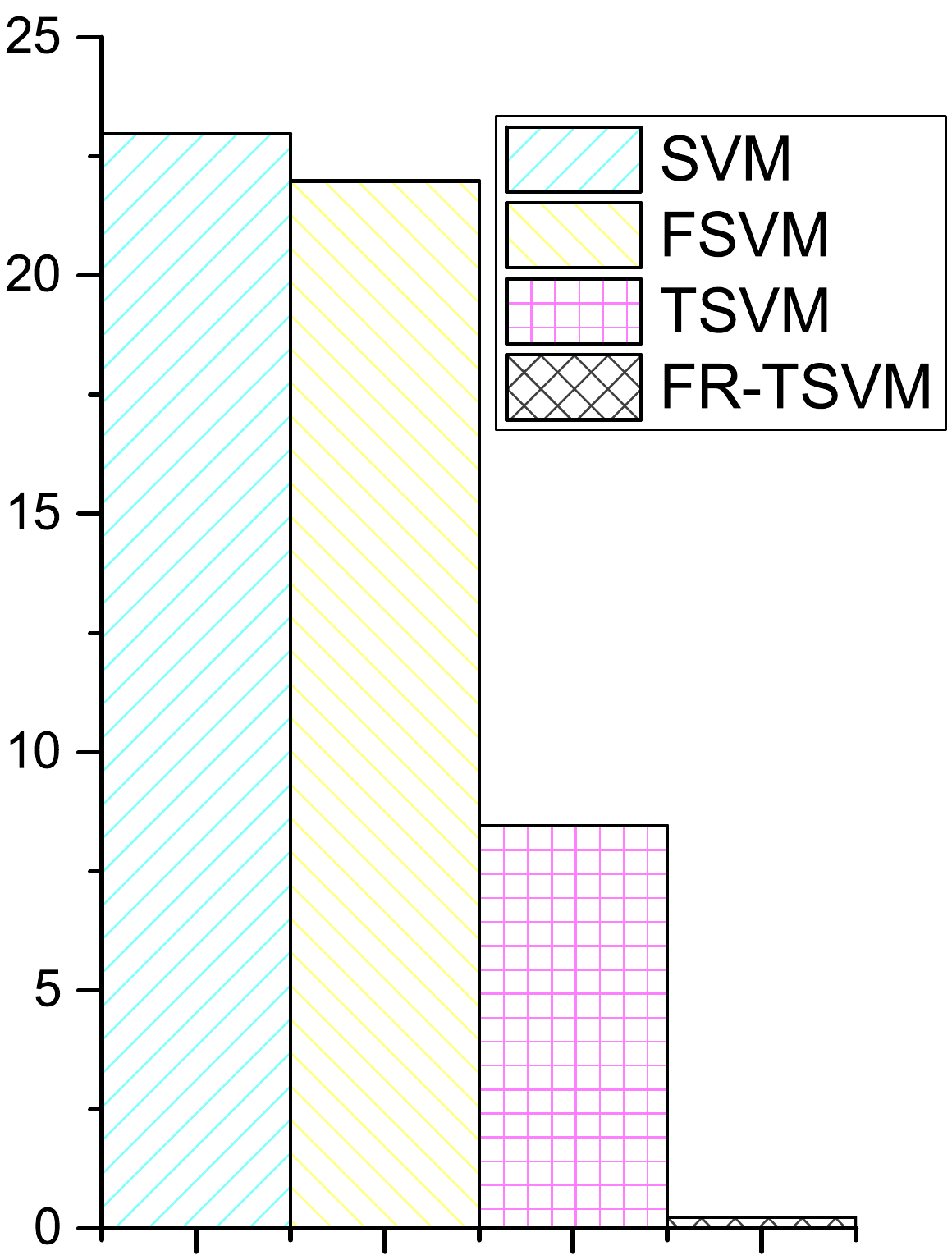}}\end{tabular}}
 \caption{Comparison of \textbf{training time} on thirteen benchmark datasets including 
  (a)Breast (b)Ionosphere (c)Iris (d)Australian (e)WDBC (f)Wine (g)Hepatitis (h)WPBC (i)Bupa (j)Sonar (k)Glass (l)Heart and (m)Pima.\label{fig:UCI_T}}
\end{figure*}

\subsection{Parameter analysis}

\begin{figure*}
    \centering
 \subfloat[$\mu=0.1$, $g=1$]
 {\includegraphics[width= 0.25\textwidth]{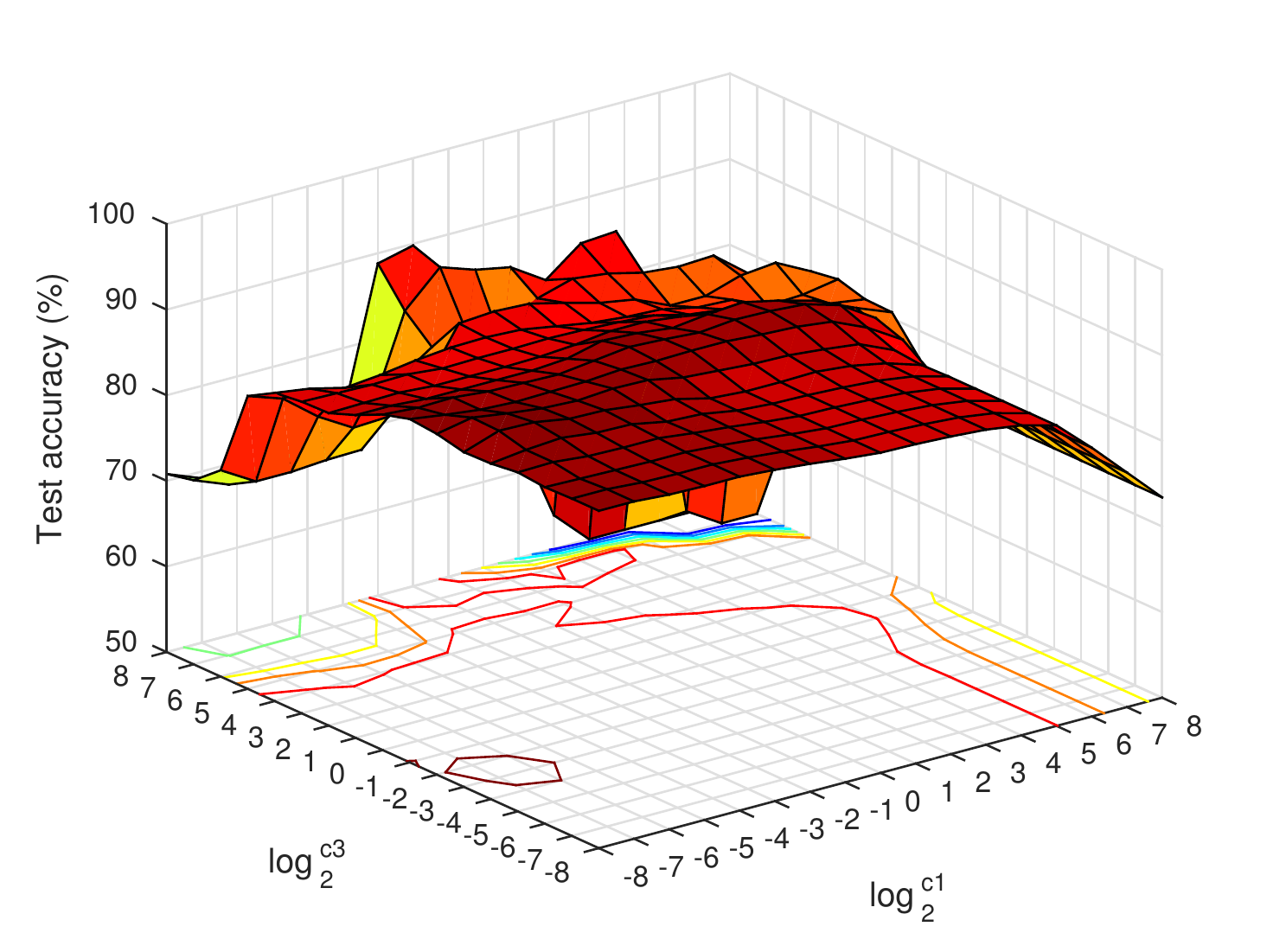}\label{fig:IP:a}}
    \subfloat[$\mu=0.1$, $c_1=32$]
    {\includegraphics[width= 0.25\textwidth]{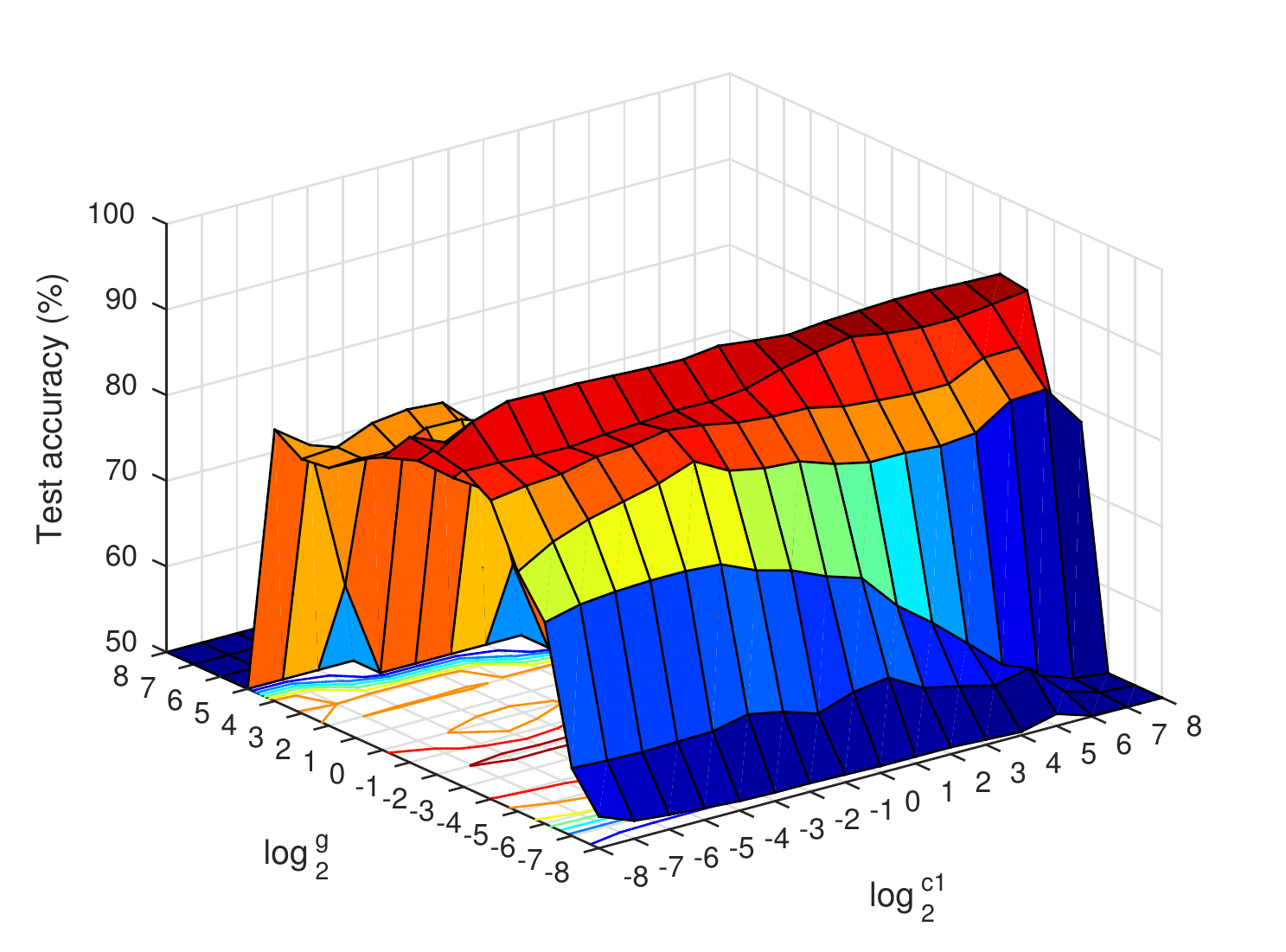}\label{fig:IP:b}}
 \subfloat[$\mu=0.1$, $c_3=0.5$]
 {\includegraphics[width= 0.25\textwidth]{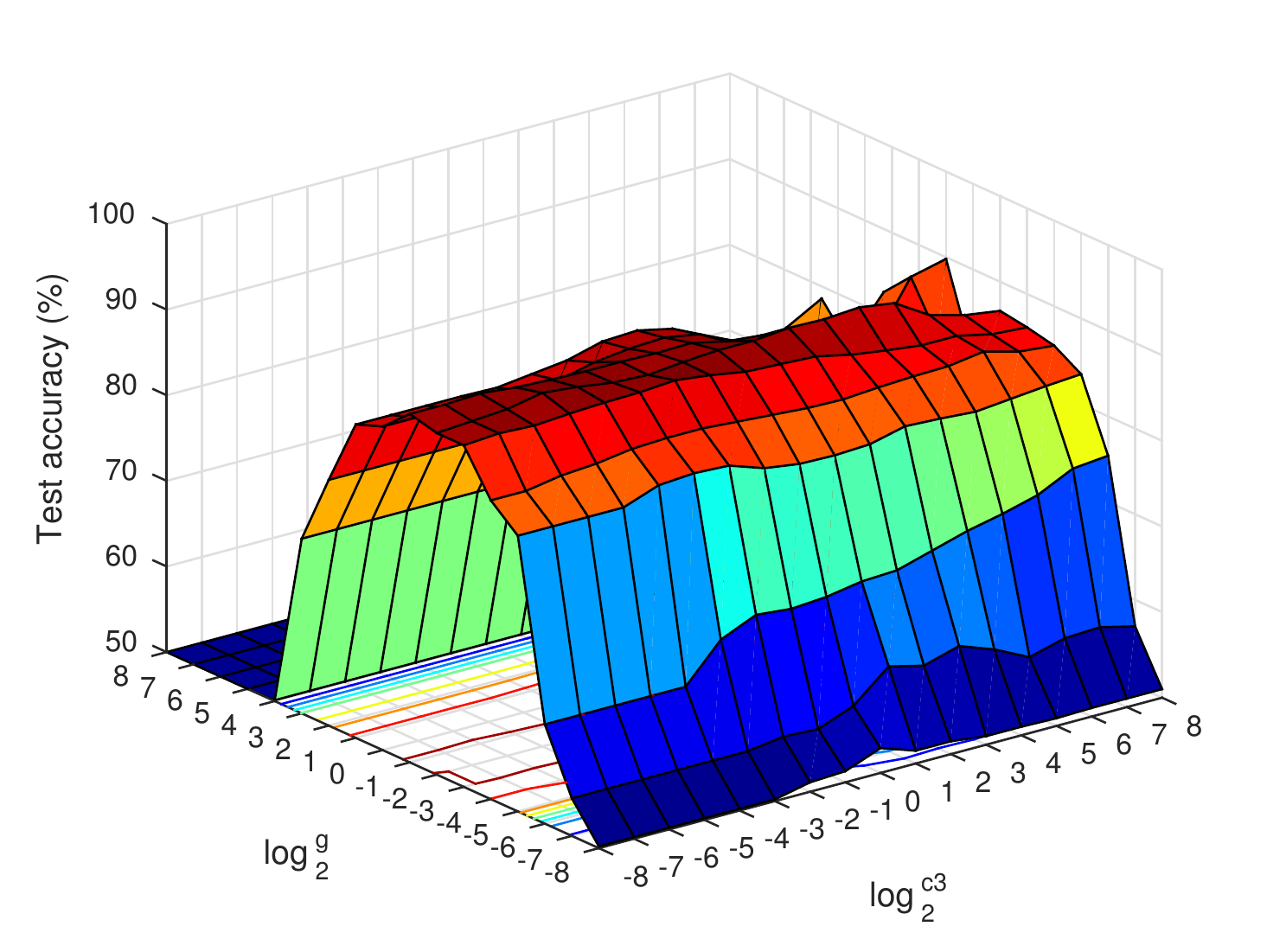}\label{fig:IP:c}}
 \subfloat[$c_1=32$, $c_3=0.5$]
 {\includegraphics[width= 0.25\textwidth]{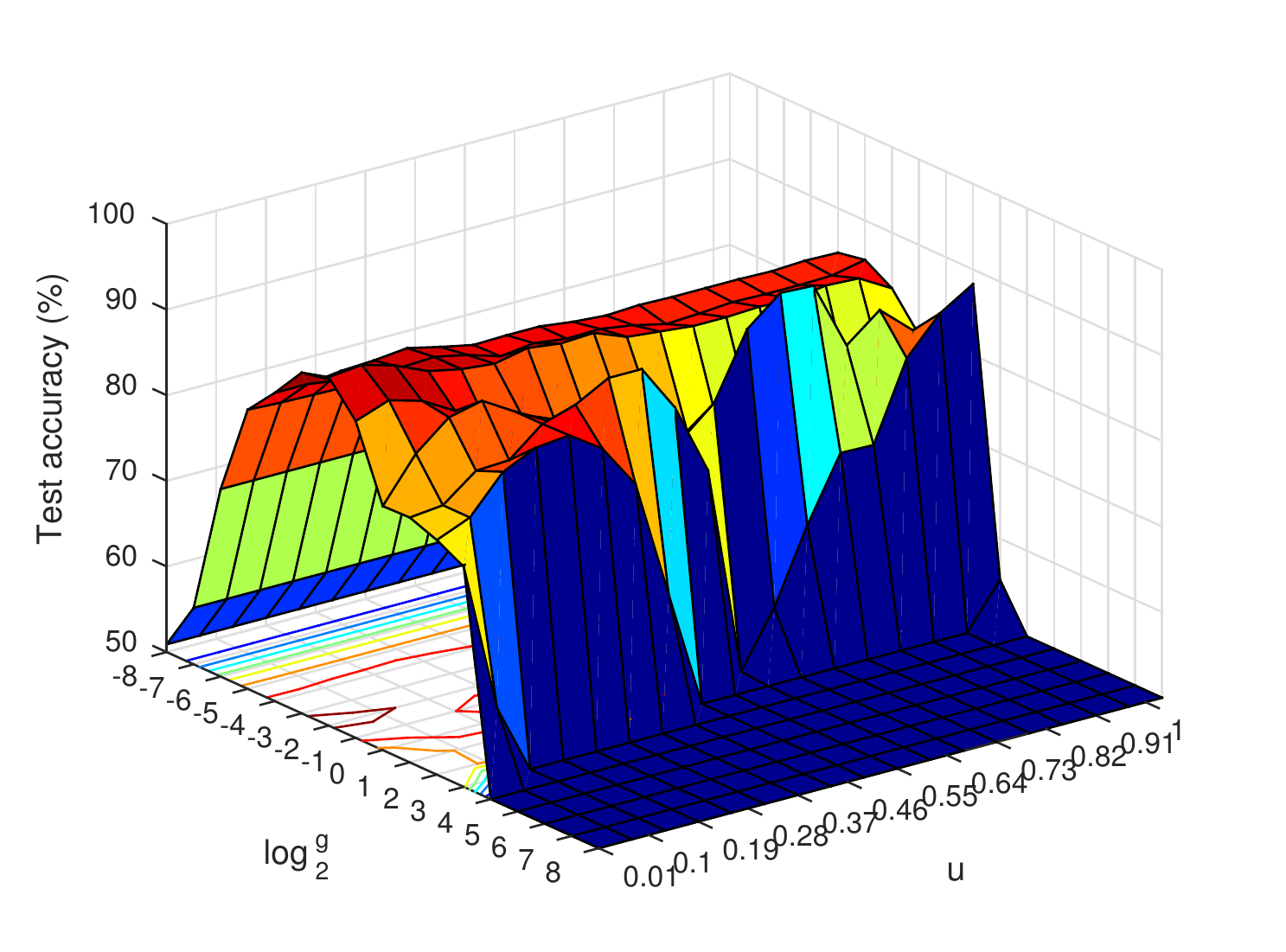}\label{fig:IP:d}}
 \caption{Accuracy varying adjustable parameters of nonlinear FR-TSVM on Ripleys dataset.\label{fig:IP}}
\end{figure*}

FR-TSVM's performance may be affected by hyperparameters. The hyperparameters corresponding to the investigation are those $c_1=c_2$, $c_3=c_4$, $u$, and $g$ in Gaussian kernel for nonlinear case. Here, we take \emph{Ripleys synathetic} dataset for an example.

Fig.~\ref{fig:IP} further shows the test accuracy with
varying the adjustable parameters on artificial dataset~\emph{Ripleys synathetic}.
Despite artificial dataset, two points are thus eventually asserted
from the investigation. First, each adjustable parameter affects
significantly the FR-TSVM's performance, especially Gaussian kernel parameter $g$. Second, the inconsistent trends of the accuracy with respect to the parameters reveal that the issue of parameter selection is still a challenge for us. We thus leave the issue as a future study.

\section{Conclusion}\label{sec:con}
In this paper, we propose a FR-TSVM algorithm for binary classification based on FSVM, TSVM and \emph{coordinate descent} methods. First, similar to FSVM classifier, the embedded fuzzy concept enhances noise-resistance capability and generalization ability. However, the fuzzy membership construction is different from that of the FSVM. Our method can effectively assign the fuzzy membership value to different instances by distinguishing the support vectors and the outliers. Second, as TSVM, the FR-TSVM finds a pair of nonparallel hyperplanes through two smaller sized QPPs rather than one large-sized QPP in the SVM or FSVM. As we all know, the dual problems of the TSVM may be ill-conditioned, but the proposed model's dual form has been brought to a pair of convex quadratic programming problems and confirmed it is capable of solution uniqueness and singularity avoidance. Third, a novel coordinate descent method with shrinking is developed to solve the dual problems. Compared to TSVM, our FR-TSVM is not only faster but also needs less memory storage. This indicates that our FR-TSVM is very suitable for large-scale data. Experiments with simulated and realistic datasets reveal that an exceedingly high classification accuracy with less computation time is achieved using both linear or nonlinear FR-TSVM. However, there are 6 parameters~(\emph{i.e.}, $c_1,c_2,c_3,c_4,\mu,g$) in our FR-TSVM, so the parameters selection is a practical problem and should be addressed in the future. Also, it is an interesting direction to extend FR-TSVM to more recognition problems such as multi-class or multi-label problem.

\section*{References}
\bibliographystyle{elsarticle-num}
\bibliography{bib}

\end{document}